\documentclass[11pt]{article}
\usepackage[margin=1in]{geometry}
\usepackage{times}

\usepackage{amsmath,amsfonts,bm}

\def\eqref#1{equation~\ref{#1}}

\def\1{\bm{1}}

\DeclareMathAlphabet{\mathsfit}{\encodingdefault}{\sfdefault}{m}{sl}
\SetMathAlphabet{\mathsfit}{bold}{\encodingdefault}{\sfdefault}{bx}{n}

\newcommand{\E}{\mathbb{E}}

\usepackage{amssymb,amsthm}
\newtheorem{theorem}{Theorem}[section]
\newtheorem{lemma}[theorem]{Lemma}
\newtheorem{proposition}[theorem]{Proposition}
\newtheorem{corollary}[theorem]{Corollary}
\theoremstyle{definition}
\newtheorem{definition}[theorem]{Definition}
\newtheorem{assumption}[theorem]{Assumption}
\theoremstyle{remark}

\renewcommand{\eqref}[1]{\textup{(\ref{#1})}}

\usepackage[round,authoryear]{natbib}
\usepackage{microtype}
\usepackage{xcolor}
\usepackage[colorlinks=true,citecolor=blue,linkcolor=blue,urlcolor=blue]{hyperref}
\usepackage{url}
\hypersetup{
  pdftitle={Flow Matching under Noisy Latent Structure: Beyond Exact Low-Dimensional Support},
  pdfauthor={Lifeng Hao and Shaolin Ji}
}

\title{Flow Matching under Noisy Latent Structure: Beyond Exact Low-Dimensional Support}

\author{
Lifeng Hao\thanks{
Zhongtai Securities Institute for Financial Studies,
Shandong University, Jinan 250100, Shandong, China.
\newline
Emails: \texttt{201911961@mail.sdu.edu.cn} (Lifeng Hao);
\texttt{jsl@sdu.edu.cn} (Shaolin Ji).
}
\qquad
Shaolin Ji\footnotemark[1]
}

\date{}

\begin{document}

\maketitle

\begin{abstract}
    Flow Matching (FM) learns a velocity field whose ODE transports a simple
    source distribution to a target law. Existing finite-sample theory largely
    treats ambient-space regularity or data supported exactly on low-dimensional
    sets. We study linear FM under a noisy latent-generator model, where a
    low-dimensional H\"older map is perturbed by nondegenerate ambient Gaussian
    noise, so the target law is full-dimensional despite its latent structure.
    We construct a spatially regular ReLU velocity class and establish
    non-asymptotic high-probability approximation and estimation bounds whose
    leading sample-size exponent is governed by the latent dimension rather than
    the ambient dimension, with ambient and noise dependence kept explicit.
    Fixed positive target noise keeps the interpolation nondegenerate over the
    full time interval. The same spatial regularity propagates the learned
    velocity error through the transport ODE, yielding a corresponding
    Wasserstein convergence guarantee. These results show that exact
    low-dimensional support is not necessary for Flow Matching to retain
    latent-dimensional statistical behavior.
\end{abstract}


\section{Introduction}
\label{sec:introduction}

Flow Matching (FM) has emerged as a simulation-free framework for continuous-time generative modeling, learning a time-dependent velocity field through regression and generating samples by integrating the resulting ordinary differential equation \citep{lipman2023flowmatching,albergo2023building,liu2023flow}. Its simple training objective and deterministic sampling dynamics have led to successful applications in high-dimensional domains including image, audio, video, and molecular generation \citep{esser2024scaling,le2023voicebox,bose2024se3}. A central theoretical question is therefore how accurately the population velocity can be learned from finitely many observations and how this statistical error propagates to the generated distribution.

Recent work has established increasingly complete statistical guarantees for FM. Stability analyses relate velocity-field error to distributional error of the induced flow \citep{benton2024error,gao2024gaussian}, while finite-sample analyses derive neural velocity-estimation and Wasserstein convergence rates under regularity assumptions on full-dimensional target distributions or ambient velocity fields \citep{zhou2025error,fukumizu2025almost,kunkel2026distribution}. In these formulations, the statistical complexity is primarily described through regularity in the observed state space.

A complementary line of work asks whether generative models can exploit lower-dimensional structure. For diffusion models, intrinsic-dimensional guarantees have been established under linear-subspace, singular-support, and smooth-manifold assumptions \citep{chen2023score,debortoli2022convergence,tang2024adaptivity,azangulov2024convergence}. Recent FM theory similarly establishes intrinsic-dimensional behavior for targets supported on low-dimensional subspaces or smooth manifolds \citep{kunkeltrabs2025minimax,kumar2026adaptive,pi2026learning,roy2026lowdimensional}. Thus, the question is no longer whether FM can exploit low-dimensional structure, but how far the exact-support assumptions underlying these guarantees can be relaxed.

Low-dimensional structure need not imply that the observed distribution is itself 
supported on a low-dimensional set. Motivated by noisy-manifold and 
latent-variable models \citep{genovese2012manifold,fefferman2018fitting,
tipping1999probabilistic,lawrence2005probabilistic}, we consider
$$
    X_1=g^*(U)+\sigma\xi,
    \qquad
    U\sim\operatorname{Unif}([0,1]^d),
    \qquad
    \xi\sim\mathcal N(0,I_D),
    \qquad d\ll D,
$$
where \(g^*\) is H\"older regular and \(0<\sigma<1\) is fixed. 
The target law is full-dimensional, while its systematic variation is 
parameterized by a \(d\)-dimensional latent variable; the image of 
\(g^*\) need not form a regular embedded manifold. 
Under the same noisy latent-generator model, \citet{yakovlev2025generalization} established latent-dimensional approximation and generalization guarantees for denoising score matching. We investigate whether the same statistical benefit persists for Flow Matching. For fixed \(\sigma>0\), the interpolation variance remains uniformly positive on \(t\in[0,1]\), so the present analysis naturally covers the full time interval.

Our results show that FM retains latent-dimensional statistical behavior 
in this full-dimensional noisy setting. Using the latent posterior 
structure, we construct a spatially regular ReLU velocity class and establish 
a non-asymptotic high-probability bound for its empirical risk minimizer, with 
leading sample-size factor \(n^{-2\beta/(2\beta+d)}\). The remaining dependence 
on the ambient dimension \(D\), the H\"older radius \(H\), the noise level 
\(\sigma\), the confidence parameter, and logarithmic factors is kept explicit in our main theorem. We then propagate this velocity-estimation error through the learned ODE and obtain a corresponding non-asymptotic \(W_2\) bound with leading sample-size factor \(n^{-\beta/(2\beta+d)}\). Thus, the algebraic dependence on the sample size is governed by the latent dimension \(d\), even though the observed target distribution has full ambient support.

\subsection{Main Contributions}
\label{subsec:contributions}

Our main contributions are summarized as follows.

\begin{itemize}

\item
\textbf{Latent-dimensional velocity approximation.}
Under the noisy H\"older latent-generator model, we construct a spatially regular ReLU approximation of the population FM velocity whose leading resolution complexity scales as \(\varepsilon^{-d}\), even though the target distribution has full ambient support.

\item
\textbf{Non-asymptotic velocity estimation.}
For empirical risk minimization over the corresponding spatially regular velocity class, we establish a high-probability excess-risk bound with leading sample-size factor \(n^{-2\beta/(2\beta+d)}\), while keeping the dependence on \(D\), \(H\), \(\sigma\), and the confidence level explicit.

\item
\textbf{Wasserstein recovery.}
The same spatial regularity propagates the velocity-estimation guarantee through the continuous learned flow, yielding a non-asymptotic \(W_2\) bound with leading sample-size factor \(n^{-\beta/(2\beta+d)}\). For fixed positive target noise, the continuous-flow analysis applies on the full interval \(t\in[0,1]\).

\end{itemize}

\subsection{Other Related Work}
\label{subsec:related-work}
\paragraph{Generative learning with noisy observations.}
A related line of work studies generative modeling when training observations are corrupted or noisy. Ambient Diffusion and subsequent developments learn an underlying clean distribution from corrupted observations using diffusion-based objectives \citep{daras2023ambient,daras2024consistent}; related approaches consider alternative estimators and training procedures under observation corruption \citep{kawar2024gsure,bai2024em}. These works treat corruption as an observation mechanism to be removed or inverted, whereas in our setting the nondegenerate ambient perturbation is part of the target distribution itself.
Noise smoothing also appears in continuous generative modeling through dequantization, where discrete observations are perturbed to obtain a continuous density suitable for flow-based models \citep{ho2019flowpp,theis2016note,nielsen2020closing}. This provides a separate algorithmic precedent for working with full-dimensional perturbed observations, rather than a statistical assumption underlying our analysis. Flow Matching has also been extended to prescribed non-Euclidean state spaces through Riemannian and geometric constructions \citep{chen2024geometry}; our setting remains Euclidean and instead concerns statistical complexity induced by a low-dimensional latent generating mechanism.
















\section{Problem Setup and Preliminaries}
\label{sec:preliminaries}

\subsection{Notation and conventions}
For vectors, $\|\cdot\|$ and $\|\cdot\|_\infty$ denote the Euclidean
and coordinatewise maximum norms, respectively. For a vector-valued
function $f$, write
$\|f\|_{L^\infty(\Omega)}:=\sup_{u\in\Omega}\|f(u)\|$.
We use $a\vee b:=\max\{a,b\}$ and $a\wedge b:=\min\{a,b\}$.
For $\beta,H>0$, $\mathcal H^\beta(\Omega,\mathbb R^m,H)$ denotes
the coordinatewise H\"older class used throughout the paper; its precise
multi-index definition is given in Appendix~\ref{app:setup-definitions}.
For nonnegative quantities, $A\lesssim_\Theta B$ means
$A\le C_\Theta B$ for a constant depending only on the parameters in
$\Theta$; when the subscript is omitted, the constant is universal.
We write $A\asymp_\Theta B$ when both
$A\lesssim_\Theta B$ and $B\lesssim_\Theta A$ hold.
Throughout, $\log$ denotes the natural logarithm and $e$ Euler's number.

\subsection{Noisy latent Flow Matching}
\label{subsec:noisy-latent-fm}

\paragraph{Target model and interpolation.}
Let $d$ and $D$ denote the latent and ambient dimensions, respectively.
We consider the target model
\begin{equation}
    X_1=g^*(U)+\sigma\xi,
    \qquad
    U\sim\operatorname{Unif}([0,1]^d),
    \qquad
    \xi\sim\mathcal N(0,I_D),
    \label{eq:target-model}
\end{equation}
where $0<\sigma<1$ is fixed and
\begin{equation}
    g^*\in\mathcal H^\beta([0,1]^d,\mathbb R^D,H),
    \qquad
    \|g^*\|_{L^\infty([0,1]^d)}\le1.
    \label{eq:generator-assumption}
\end{equation}
The variables $U$ and $\xi$ are independent. For every fixed
$\sigma>0$, the target law $\pi_1$ is full-dimensional; the
low-dimensional structure enters through the latent generator $g^*$,
rather than through exact low-dimensional support.
Let $X_0\sim\mathcal N(0,I_D)$ be independent of $(U,\xi)$ and consider
the linear interpolation
\begin{equation}
    X_t=(1-t)X_0+tX_1,
    \qquad t\in[0,1].
    \label{eq:linear-interpolation}
\end{equation}
We write $\pi_t=\operatorname{Law}(X_t)$ and
$\|h(\cdot,t)\|_{L^2(\pi_t)}^2
:=\mathbb E\|h(X_t,t)\|^2$.
Conditional on $U=u$,
$
    X_t\mid U=u
    \sim
    \mathcal N\!\left(tg^*(u),q_t I_D\right),
    q_t:=(1-t)^2+\sigma^2t^2.
$
Since
$q_t\ge \sigma^2/(1+\sigma^2)>0$,
the interpolation remains nondegenerate throughout $t\in[0,1]$.
\paragraph{Population objective.}
The population Flow Matching risk is
\[
    \mathcal L(v)
    :=
    \int_0^1
    \mathbb E_{X_0,X_1}\!\left[
        \|X_1-X_0-v(X_t,t)\|^2
    \right]dt,
\]
whose population velocity is
\[
    v^*(x,t)
    =
    \mathbb E[X_1-X_0\mid X_t=x].
\]
Our statistical experiment observes
$X_{1,1},\ldots,X_{1,n}\stackrel{\mathrm{iid}}{\sim}\pi_1$.
The Gaussian source and interpolation time are auxiliary variables and
are integrated exactly in the idealized empirical loss introduced in
Section~\ref{sec:generalization}.

\paragraph{Posterior representation.}
Define
\[
    a_t:=\frac{t\sigma^2-(1-t)}{q_t},
    \qquad
    b_t:=\frac{1-t}{q_t},
    \qquad
    m_t(x):=\mathbb E[g^*(U)\mid X_t=x].
\]
The following representation reduces the nonlinear part of the
population velocity to the latent posterior mean $m_t$.

\begin{lemma}[Posterior representation of the population velocity]
\label{lem:posterior-velocity}
Under the preceding model, the conditional means admit versions such
that, for every $(x,t)\in\mathbb R^D\times[0,1]$,
\[
    v^*(x,t)=a_tx+b_tm_t(x),
    \qquad
    \|m_t(x)\|\le1,
\]
where
\[
    m_t(x)
    =
    \frac{
        \displaystyle
        \int_{[0,1]^d}
        g^*(u)
        \exp\!\left(
            -\frac{\|x-tg^*(u)\|^2}{2q_t}
        \right)du
    }{
        \displaystyle
        \int_{[0,1]^d}
        \exp\!\left(
            -\frac{\|x-tg^*(u)\|^2}{2q_t}
        \right)du
    }.
\]
\end{lemma}

The proof is given in Appendix~\ref{app:population-identities}.
The decomposition separates the known affine component
$a_tx$ from the posterior component $b_tm_t(x)$ that is approximated
by the neural network in Section~\ref{sec:approximation}.

\subsection{Neural network class}
Throughout the approximation and statistical analysis, we use sparse
feed-forward ReLU networks. Let $\rho(z)=\max\{0,z\}$ act componentwise.
For $L,W,S\in\mathbb N$ and $B\ge1$, denoting depth, maximal width,
sparsity, and the parameter-magnitude bound, respectively, define
\[
\mathrm{NN}(L,W,S,B)
:=
\left\{
\begin{array}{l}
(A_L\rho(\cdot)+b_L)\circ\cdots\circ
(A_2\rho(\cdot)+b_2)\circ(A_1x+b_1):\\[1mm]
A_\ell\in\mathbb R^{W_\ell\times W_{\ell-1}},\
b_\ell\in\mathbb R^{W_\ell},\
\max_{0\le\ell\le L}W_\ell\le W,\\
\sum_{\ell=1}^L(\|A_\ell\|_0+\|b_\ell\|_0)\le S,\quad
\max_{1\le\ell\le L}
(\|A_\ell\|_\infty\vee\|b_\ell\|_\infty)\le B
\end{array}
\right\}.
\]
Here $\theta=\{(A_\ell,b_\ell)\}_{\ell=1}^L$ denotes the trainable
parameters, $\|\cdot\|_0$ counts nonzero entries, and
$\|\cdot\|_\infty$ is the largest absolute entry. For the velocity
network, $W_0=D+1$ and $W_L=D$; auxiliary networks use their stated
dimensions.

\section{Neural Approximation of the Population Velocity}
\label{sec:approximation}

This section establishes a neural approximation theorem for the population
FM velocity under the noisy latent-generator model. By the posterior
representation in Section~\ref{sec:preliminaries},
$
    v^*(x,t)=a_tx+b_tm_t(x),
$
the nonlinear approximation problem reduces to the posterior mean $m_t$.
The main difficulty is that $m_t$ is defined on the unbounded ambient state
space, while the exploitable structure is inherited from the
$d$-dimensional latent generator; the approximation must also retain the
spatial regularity required for the subsequent transport analysis.

Let $\operatorname{clip}_2$ be Euclidean projection onto $\mathcal B(0,2)$.
For a deterministic spatial envelope $\gamma_x:[0,1]\to[0,\infty)$, define
\begin{equation}
\begin{aligned}
\mathcal V(L,W,S,B,\gamma_x):=\bigl\{v_\theta:\;&
 v_\theta(x,t)=a_tx+b_t\operatorname{clip}_2(f_\theta(x,t)),
 \quad f_\theta\in\mathrm{NN}(L,W,S,B),\\
 &\operatorname{Lip}_x(v_\theta(\cdot,t))\le\gamma_x(t)
 \quad(t\in[0,1])\bigr\}.
\end{aligned}
\label{eq:structured-velocity-class}
\end{equation}
The network approximates the nonlinear posterior component, while $a_t$ and
$b_t$ remain the exact coefficients of the linear path. No separate temporal
regularity constraint is imposed on this class.
Let $\varepsilon\in(0,1/2]$ satisfy the scale conditions
\begin{gather}
 \varepsilon^{-2\beta}>D,\qquad
 H\sqrt D\,\varepsilon^{\beta\wedge1}\le c_{d,\beta}\sigma,
 \label{eq:approximation-scale-basic}\\
 \varepsilon\!\left[1+\frac{1+\sigma^2}{\sigma^2}D(H\vee1)
 \left\{2+\sqrt D+16\left(
 \sqrt{D\log(\varepsilon^{-2\beta}/D)}
 \vee\log(\varepsilon^{-2\beta}/D)\right)\right\}\right]
 \le c_{d,\beta},
 \label{eq:approximation-scale-localization}
\end{gather}
where $c_{d,\beta}>0$ is sufficiently small.

\begin{theorem}[Approximation of the population velocity]
\label{thm:velocity-approximation}
Assume \eqref{eq:target-model}--\eqref{eq:generator-assumption}, with
$0<\sigma<1$, and let $\varepsilon\in(0,1/2]$ satisfy
\eqref{eq:approximation-scale-basic}--
\eqref{eq:approximation-scale-localization}. Set
$ \Gamma_\varepsilon:=1+D+\log(1/\varepsilon)
 +\log(H\vee e)+\log(1/\sigma).
$
For a sufficiently large fixed constant $C$ depending only on $(d,\beta)$,
choose
$
 \gamma_x(t):=C\left[|a_t|+
 |b_t|\left(1+\sigma^{-1}+\frac{t}{q_t}\right)\right].
$
Fix one finite exponent
$c_{\rm arch}=c_{\rm arch}(d,\beta)$, large enough to dominate the
first-order realization and the covering-number calculation below.
Then network parameters can be chosen so that
\[
 L,\log B\lesssim_{d,\beta}
 \Gamma_\varepsilon^{c_{\rm arch}},
 \qquad W,S\lesssim_{d,\beta}D^2\varepsilon^{-d}
 \Gamma_\varepsilon^{c_{\rm arch}},
\]
and $\mathcal V(L,W,S,B,\gamma_x)$ contains a velocity field $v_\theta$
satisfying
\begin{equation}
 \int_0^1
 \|v_\theta(\cdot,t)-v^*(\cdot,t)\|_{L^2(\pi_t)}^2\,dt
 \lesssim_{d,\beta}
 \frac{D(1+H^2)}{\sigma}\,\varepsilon^{2\beta}.
 \label{eq:velocity-approximation-rate}
\end{equation}
Moreover,
$\|\gamma_x\|_{L^1(0,1)}\lesssim_{d,\beta}1+\sigma^{-2}$ and
$\|\gamma_x\|_{L^\infty(0,1)}\lesssim_{d,\beta}1+\sigma^{-3}$.
\end{theorem}
The proof is given in Appendix~\ref{app:approximation-main}.
The leading dependence on the approximation scale is
$\varepsilon^{-d}$, inherited from the $d$-dimensional latent
parameterization. The ambient dimension $D$ remains in the explicit
polynomial factors and in $\Gamma_\varepsilon$. The same comparator also
satisfies the spatial profile $\gamma_x$ used later in the transport
analysis.

\paragraph{Low-dimensional approximation.}
The approximation complexity of FM depends strongly on the structural
description of the target distribution. Existing results treat
bounded-support assumptions \citep{zhou2025error} and B-spline approximation
under Besov regularity of full-dimensional target densities
\citep{fukumizu2025almost}. Intrinsic-dimensional guarantees have also been
obtained from exact low-dimensional support, including smooth manifolds
\citep{kumar2026adaptive,kunkeltrabs2025minimax} and linear subspaces
\citep{pi2026learning}. Related diffusion analyses exploit manifold geometry
and regularity
\citep{tang2024adaptivity,azangulov2024convergence}, while
\citet{yakovlev2025generalization} study a noisy latent generator and
approximate its Gaussian-mixture posterior through local approximation on
the latent domain.
Under the present model, the FM posterior mean is defined on the ambient
state space but is generated by the $d$-dimensional map $g^*$.

\paragraph{Spatially regular approximation.}
Spatial regularity is a separate requirement because the learned velocity is
later used as the drift of a deterministic ODE. Neural approximation with
regularity control has been developed through norm constraints on network
parameters
\citep{jiao2023normconstrained} and through approximation results that control
derivatives together with function values
\citep{guehring2020sobolev,opschoor2022holomorphic}. In FM, such regularity
is directly tied to the stability of the learned ODE and is incorporated into
the velocity classes considered in
\citet{zhou2025error,kunkel2026distribution}.
In the present setting, the spatial control follows from the same posterior
representation used for value approximation. Gaussian smoothing gives an
explicit bound on the spatial Jacobian of the posterior mean, and the neural
realization is refined to control the corresponding first derivatives.
The resulting comparator satisfies both the integrated velocity
approximation bound and the profile $\gamma_x$.

\paragraph{Proof sketch.}
Let $N=\lceil\varepsilon^{-1}\rceil$ and partition $[0,1]^d$ into
$N^d$ latent cells. On each cell, approximate $g^*$ by its Taylor polynomial
of degree $r_\beta:=\lceil\beta\rceil-1$. The $\beta$-H\"older regularity of $g^*$
gives
\[
    \|g^*-g^\circ\|_{L^\infty([0,1]^d)}
    \lesssim_{d,\beta}
    H\sqrt D\,\varepsilon^\beta.
\]
Let $v^\circ$ denote the population FM velocity induced by the surrogate
target $g^\circ(U)+\sigma\xi$. The following proposition transfers this
latent approximation to the corresponding velocity error.

\begin{proposition}[Generator-to-velocity stability]
\label{prop:generator-velocity-stability}
Let $g^*,g^\circ:[0,1]^d\to\mathbb R^D$ be bounded and measurable, and
let $v^*,v^\circ$ be the population velocities generated by the targets
$g^*(U)+\sigma\xi$ and $g^\circ(U)+\sigma\xi$, respectively, under the
same linear interpolation with $\sigma>0$. Then
\[
    \int_0^1
    \|v^*(\cdot,t)-v^\circ(\cdot,t)\|_{L^2(\pi_t)}^2\,dt
    \le
    \frac{\pi}{4\sigma}
    \|g^*-g^\circ\|_{L^\infty([0,1]^d)}^2,
\]
where $\pi_t$ is the marginal law of the true interpolation.
\end{proposition}
The proof is given in
Appendix~\ref{app:generator-velocity-stability} and is based on a
Gaussian-mixture relative-entropy comparison.
It remains to approximate $v^\circ$. Since
$
    v^\circ(x,t)=a_tx+b_tm_t^\circ(x),
$
the nonlinear part reduces to the surrogate posterior mean $m_t^\circ$.
Gaussian concentration localizes the approximation to a high-probability
spatial region, on which the cellwise polynomial structure of $g^\circ$
reduces the posterior numerator and denominator to low-dimensional integral
terms. This concentration-based localization also avoids imposing a bounded
ambient support assumption.
Lemma~\ref{lem:app-surrogate-posterior-network} realizes these terms and
their ratio while controlling the spatial Jacobian. Combining this posterior
realization with
Proposition~\ref{prop:generator-velocity-stability} gives the approximation
rate, network complexity, and spatial profile stated in
Theorem~\ref{thm:velocity-approximation}.

\section{Generalization}
\label{sec:generalization}
Section~\ref{sec:approximation} shows that the structured velocity class
contains a comparator with small population approximation error. We now
study whether this accuracy can be recovered from finitely many target
observations. The starting point is the following population identity,
which identifies the statistical target with the integrated velocity error.
\begin{lemma}[Population risk identity]
\label{lem:fm-risk-identity}
For every measurable $v$ with finite population risk,
\begin{equation}
    \mathcal L(v)-\mathcal L(v^*)
    =
    \int_0^1
    \|v(\cdot,t)-v^*(\cdot,t)\|_{L^2(\pi_t)}^2\,dt.
    \label{eq:fm-risk-identity}
\end{equation}
\end{lemma}
Thus, controlling the excess Flow Matching risk is equivalent to controlling
the time-integrated $L^2(\pi_t)$ error of the learned velocity field. The
proof is given in Appendix~\ref{app:generalization-localization}.

\paragraph{Empirical evaluation.}
We observe
$X_{1,1},\ldots,X_{1,n}\stackrel{\mathrm{iid}}{\sim}\pi_1$.
For $x_1\in\mathbb R^D$ and a measurable velocity field $v$, define
\begin{equation}
    \ell(x_1,v)
    :=
    \int_0^1
    \mathbb E_{X_0}
    \bigl\|
        x_1-X_0-v((1-t)X_0+tx_1,t)
    \bigr\|^2\,dt,
    \qquad
    \widehat{\mathcal L}_n(v)
    :=
    \frac1n\sum_{i=1}^n\ell(X_{1,i},v).
    \label{eq:sample-fm-loss}
\end{equation}
Then $\mathcal L(v)=P\ell(\cdot,v)$, where $P$ denotes expectation with respect to $X_1\sim\pi_1$. As in \citet{fukumizu2025almost}, the auxiliary Gaussian source and the time variable are integrated exactly inside $\ell$, so the finite-sample randomness comes only from the target observations $X_{1,1:n}$. Finite auxiliary Monte Carlo sampling would introduce an additional error term; see, e.g., \citet{zhou2025error}.
For fixed $L,W,S,B$, define the empirical risk minimizer
\begin{equation}
    \widehat v
    \in
    \arg\min_{v\in\mathcal V(L,W,S,B,\gamma_x)}
    \widehat{\mathcal L}_n(v).
    \label{eq:fm-erm}
\end{equation}

\subsection{Error Decomposition}

The excess risk of $\widehat v$ decomposes as
\begin{equation}
\begin{aligned}
    \mathcal L(\widehat v)-\mathcal L(v^*)
    &=
    \underbrace{
        \mathcal L(\widehat v)
        -
        \inf_{v\in\mathcal V(L,W,S,B,\gamma_x)}
        \mathcal L(v)
    }_{\mathrm{estimation}}
    +
    \underbrace{
        \inf_{v\in\mathcal V(L,W,S,B,\gamma_x)}
        \bigl\{\mathcal L(v)-\mathcal L(v^*)\bigr\}
    }_{\mathrm{approximation}}.
\end{aligned}
\label{eq:estimation-approximation-decomposition}
\end{equation}
Section~\ref{sec:approximation} controls the approximation term. To analyze the estimation term, define the signed excess loss $\Delta_v(x_1):=\ell(x_1,v)-\ell(x_1,v^*)$, and write $Ph:=\mathbb E[h(X_1)]$ and $P_nh:=n^{-1}\sum_{i=1}^n h(X_{1,i})$. By Lemma~\ref{lem:fm-risk-identity}, $P\Delta_v=\mathcal L(v)-\mathcal L(v^*)\ge0$, although $\Delta_v(x_1)$ need not be pointwise nonnegative.
The squared Flow Matching loss is unbounded. For a real random variable $Z$, write $\|Z\|_{\psi_1}:=\inf\{c>0:\mathbb E\exp(|Z|/c)\le2\}$. The following lemma gives the localized moment and tail bounds needed for an unbounded Bernstein inequality.

\begin{lemma}[Localized Flow Matching excess loss]
\label{lem:fm-excess-localization}
For every $v\in\mathcal V(L,W,S,B,\gamma_x)$,
\begin{equation}
    P\Delta_v^2
    \lesssim
    \frac{1}{\sigma^2}P\Delta_v,
    \qquad
    \|\Delta_v-P\Delta_v\|_{\psi_1}
    \lesssim
    \frac1{\sigma}+\sqrt D,
    \label{eq:fm-excess-localization}
\end{equation}
where the hidden constants are universal.
\end{lemma}
Consequently, for every fixed deterministic $v\in\mathcal V(L,W,S,B,\gamma_x)$ and every $\delta\in(0,1)$, with probability at least $1-\delta$,
\begin{equation}
    |(P-P_n)\Delta_v|
    \lesssim
    \sqrt{
        \frac{P\Delta_v\,\log(4/\delta)}
             {\sigma^2 n}
    }
    +
    \frac{
        (\sigma^{-1}+\sqrt D)\log(en)\log(4/\delta)
    }{n}.
    \label{eq:main-fm-fixed-bernstein}
\end{equation}

\subsection{Covering and Generalization}
\begin{definition}[Covering number]
    \label{def:covering-number}
    Let $(\mathcal M,d)$ be a pseudo-metric space and let $\mathcal F\subseteq\mathcal M$. For $\eta>0$, a set $\mathcal A\subseteq\mathcal M$ is called an $\eta$-cover of $\mathcal F$ if for every $f\in\mathcal F$ there exists $\widetilde f\in\mathcal A$ such that $d(f,\widetilde f)\le\eta$. The covering number $\mathcal N(\eta,\mathcal F,d)$ is the minimum cardinality of an $\eta$-cover of $\mathcal F$.
\end{definition}
To make the fixed-field bound uniform over the excess-loss class, we first control the complexity of the underlying neural networks. By the standard sparse-ReLU covering estimate of \citet[Lemma~3]{suzuki2019adaptivity}, adapted to $D$-dimensional outputs, for every $R\ge1$, $\eta\in(0,1)$, and $W\ge D+1$,
\begin{equation}
\log\mathcal N\!\left(
\eta,\mathrm{NN}(L,W,S,B),
\|\cdot\|_{L^\infty(\mathcal B(0,R)\times[0,1];\ell_2)}
\right)
\lesssim
SL\log\!\left(
\frac{\sqrt D\,L(W+1)(B\vee1)(R+1)}{\eta}
\right).
\label{eq:main-relu-covering}
\end{equation}

The network covering estimate is local to a bounded state--time domain. Gaussian concentration and the local stability of the conditional Flow Matching loss allow the same network representatives to control both the population and empirical excess losses.

\begin{lemma}[Finite reduction of the Flow Matching excess-loss class]
\label{lem:main-fm-finite-reduction}
For every $\tau,\delta\in(0,1)$, there exists a deterministic family $v_1,\ldots,v_{M_{\tau,\delta}}\in\mathcal V(L,W,S,B,\gamma_x)$ such that, with probability at least $1-\delta$, every $v\in\mathcal V(L,W,S,B,\gamma_x)$ admits a representative $v_j$ satisfying
\begin{equation}
 |P(\Delta_v-\Delta_{v_j})|
 +|P_n(\Delta_v-\Delta_{v_j})|\le\tau.
 \label{eq:main-fm-finite-reduction}
\end{equation}
Moreover,
\begin{equation}
\log M_{\tau,\delta}
\lesssim
SL\log\!\left[
\frac{\sqrt D\,L(W+1)(B\vee1)}{\tau}
\left(
1+\frac1\sigma+\sqrt D
+\sqrt{\log\frac{en}{\delta}}
+\sqrt{\log\frac e\tau}
\right)^2
\right].
\label{eq:main-fm-finite-reduction-cardinality}
\end{equation}
\end{lemma}

The proof is given in Appendix~\ref{app:generalization-finite-reduction}. Applying \eqref{eq:main-fm-fixed-bernstein} to the finite representatives, taking a union bound, and using the ERM comparison yield, with probability at least $1-\delta$,
\begin{equation}
\begin{aligned}
P\Delta_{\widehat v}
\lesssim{}&
\inf_{v\in\mathcal V(L,W,S,B,\gamma_x)}
P\Delta_v
+\tau\\
&+
\left[
\frac1{\sigma^2}
+
\left(
\frac1\sigma+\sqrt D
\right)\log(en)
\right]
\frac{\log(8M_{\tau,\delta/2}/\delta)}{n}.
\end{aligned}
\label{eq:main-fm-oracle}
\end{equation}

Combining \eqref{eq:main-fm-oracle} with Theorem~\ref{thm:velocity-approximation} yields the finite-sample rate below.

\begin{theorem}[Finite-sample velocity learning]
    \label{thm:main-generalization}
    Assume \eqref{eq:target-model}, \eqref{eq:generator-assumption}, and \eqref{eq:linear-interpolation}, with $0<\sigma<1$. Given $X_{1,1},\ldots,X_{1,n}\stackrel{\mathrm{iid}}{\sim}\pi_1$, set $\varepsilon:=n^{-1/(2\beta+d)}$ and suppose that $\varepsilon$ satisfies \eqref{eq:approximation-scale-basic}--\eqref{eq:approximation-scale-localization}. Choose $\mathcal V(L,W,S,B,\gamma_x)$ according to Theorem~\ref{thm:velocity-approximation} at resolution $\varepsilon$, and let $\widehat v$ be the ERM in \eqref{eq:fm-erm}. Then, for every $\delta\in(0,1)$, with probability at least $1-\delta$ over the target observations,
    \begin{equation}
        \begin{aligned}
        \int_0^1\|\widehat v(\cdot,t)-v^*(\cdot,t)\|_{L^2(\pi_t)}^2\,dt
        \lesssim_{d,\beta}&
        n^{-\frac{2\beta}{2\beta+d}}
        \Bigg[
        \frac{D(1+H^2)}{\sigma}\\
        &+
        D^2\bigl[\sigma^{-2}
        +(\sigma^{-1}+\sqrt D)\log(en)\bigr]
        \bigl[\Gamma_{\varepsilon}+\log(e/\delta)\bigr]^{c_{\rm arch}}
        \Bigg].
        \end{aligned}
        \label{eq:fm-main-generalization-rate}
    \end{equation}
    
\end{theorem}
    
The algebraic sample-size exponent $2\beta/(2\beta+d)$ is determined by the latent dimension $d$ and the Hölder smoothness $\beta$ of the generator. The ambient dimension $D$, noise level $\sigma$, Hölder radius $H$, and confidence level $\delta$ remain explicit in the prefactor and logarithmic complexity term.
\paragraph{Proof sketch.}
For a fixed velocity field, the conditional Flow Matching excess loss is unbounded, but its squared-loss structure yields a localized moment bound. The conditional covariance of the Flow Matching label, together with the discrepancy bound obtained from affine cancellation and radius-two clipping, gives $P\Delta_v^2\lesssim\sigma^{-2}P\Delta_v$, while the Gaussian tail of the target observation gives a uniform sub-exponential envelope for $\Delta_v-P\Delta_v$. These estimates yield the fixed-field Bernstein inequality \eqref{eq:main-fm-fixed-bernstein}.

To pass from a fixed field to the data-dependent ERM, we reduce the velocity class to finitely many deterministic representatives. The sparse-ReLU covering estimate is available on bounded state--time domains; Gaussian localization and the local stability of the conditional loss allow the same representatives to control both $P$ and $P_n$. A union bound and the ERM comparison then give \eqref{eq:main-fm-oracle}. Finally, the Section~\ref{sec:approximation} comparator has approximation scale $\varepsilon^{2\beta}$, while the corresponding entropy contribution has scale $\varepsilon^{-d}/n$, which yields the stated choice of $\varepsilon$ and the final rate. The complete proof is given in Appendix~\ref{app:generalization-main}.

\section{Distributional Convergence of the Learned Flow}
\label{sec:sampling}

We now translate this statistical error into
distributional error of the induced flow. For a measurable map $T$, let
$T_\#\mu$ denote the pushforward of $\mu$, and let $W_2$ denote the
$2$-Wasserstein distance.
\subsection{Continuous-time transport}
\label{subsec:continuous-transport}

Let $\Phi_t$ denote the flow generated by the population velocity,
$\partial_t\Phi_t(x)=v^*(\Phi_t(x),t)$ with $\Phi_0(x)=x$.

\begin{lemma}[Population transport]
\label{lem:population-flow-regularity}
The population flow is globally well posed and satisfies
$(\Phi_t)_\#\pi_0=\pi_t$ for every $t\in[0,1]$.
\end{lemma}
\begin{proof}
Let $G=g^*(U)$. Differentiating the posterior mean gives
\[
    \nabla_xm_t(x)
    =
    \frac{t}{q_t}
    \operatorname{Cov}(G\mid X_t=x),
\]
and hence
\[
    \nabla_xv^*(x,t)
    =
    a_tI_D+
    \frac{t(1-t)}{q_t^2}
    \operatorname{Cov}(G\mid X_t=x).
\]
Since $\|G\|\le1$, the conditional covariance has operator norm at most
one. Therefore
$\operatorname{Lip}_x(v^*(\cdot,t))
\le |a_t|+t(1-t)/q_t^2$, while the posterior representation also gives
$\|v^*(x,t)\|\le |a_t|\|x\|+|b_t|$.
Because $q_t$ is uniformly bounded away from zero on $[0,1]$, these
bounds yield a finite spatial Lipschitz profile and uniform linear growth.
The characteristic ODE therefore admits a unique global flow.

For every $\varphi\in C_c^\infty(\mathbb R^D)$,
\[
    \frac{d}{dt}\mathbb E\varphi(X_t)
    =
    \mathbb E\!\left[
        \nabla\varphi(X_t)^\top(X_1-X_0)
    \right]
    =
    \mathbb E\!\left[
        \nabla\varphi(X_t)^\top v^*(X_t,t)
    \right].
\]
Thus $(\pi_t)_{t\in[0,1]}$ solves the continuity equation associated with
$v^*$. Since $v^*$ is globally Lipschitz in space with an integrable
Lipschitz profile, the characteristic solution is unique, and hence
$(\Phi_t)_\#\pi_0=\pi_t$ for all $t\in[0,1]$.
\end{proof}

Couple the population and learned flows through the common initial
condition $Z_0=\widehat Z_0\sim\pi_0$, with
$\dot Z_t=v^*(Z_t,t)$ and
$\dot{\widehat Z}_t=\widehat v(\widehat Z_t,t)$, and write
$\widehat\pi_t:=\operatorname{Law}(\widehat Z_t)$.
By Lemma~\ref{lem:population-flow-regularity}, $Z_t\sim\pi_t$.

\begin{theorem}[Continuous-flow Wasserstein convergence]
\label{thm:continuous-sampling}
Under the conditions of Theorem~\ref{thm:main-generalization}, the learned
ODE is globally well posed and
\begin{equation}
    W_2^2(\widehat\pi_1,\pi_1)
    \le
    \exp\!\left\{
        1+2\|\gamma_x\|_{L^1(0,1)}
    \right\}
    \int_0^1
    \|\widehat v(\cdot,t)-v^*(\cdot,t)\|_{L^2(\pi_t)}^2\,dt.
    \label{eq:continuous-w2-stability}
\end{equation}
Consequently, with probability at least $1-\delta$ over the target
observations, the resulting nonasymptotic $W_2$ bound has leading
sample-size factor $n^{-\beta/(2\beta+d)}$.
\end{theorem}
\begin{proof}
Condition on the target observations. Since
$\widehat v\in\mathcal V(L,W,S,B,\gamma_x)$,
the spatial class constraint and
Lemma~\ref{lem:sampling-linear-growth} give global spatial Lipschitz
continuity and uniform linear growth. Hence the learned ODE admits a
unique global solution.

Set $R_t:=\mathbb E\|Z_t-\widehat Z_t\|^2$, where the expectation is
with respect to the common initial condition. Since
$(Z_t,\widehat Z_t)$ is a coupling of $(\pi_t,\widehat\pi_t)$,
$W_2^2(\widehat\pi_t,\pi_t)\le R_t$.
For almost every $t$,
\[
\begin{aligned}
    R_t'
    &=
    2\mathbb E
    \left\langle
        Z_t-\widehat Z_t,\,
        v^*(Z_t,t)-\widehat v(\widehat Z_t,t)
    \right\rangle  \\
    &=
    2\mathbb E
    \left\langle
        Z_t-\widehat Z_t,\,
        v^*(Z_t,t)-\widehat v(Z_t,t)
    \right\rangle \\
    &\quad+
    2\mathbb E
    \left\langle
        Z_t-\widehat Z_t,\,
        \widehat v(Z_t,t)-\widehat v(\widehat Z_t,t)
    \right\rangle .
\end{aligned}
\]
Using $2\langle a,b\rangle\le\|a\|^2+\|b\|^2$ for the first term,
$Z_t\sim\pi_t$, and
$\operatorname{Lip}_x(\widehat v(\cdot,t))\le\gamma_x(t)$ for the second
gives
\[
    R_t'
    \le
    \{1+2\gamma_x(t)\}R_t
    +
    \|\widehat v(\cdot,t)-v^*(\cdot,t)\|_{L^2(\pi_t)}^2,
    \qquad R_0=0.
\]
The time-dependent Gr\"onwall inequality therefore yields, for every
$t\in[0,1]$,
\[
    R_t
    \le
    \int_0^t
    \exp\!\left\{
        (t-s)+2\int_s^t\gamma_x(u)\,du
    \right\}
    \|\widehat v(\cdot,s)-v^*(\cdot,s)\|_{L^2(\pi_s)}^2\,ds.
\]
Setting $t=1$ and using
$W_2^2(\widehat\pi_1,\pi_1)\le R_1$
proves \eqref{eq:continuous-w2-stability}.

Finally, Theorem~\ref{thm:velocity-approximation} gives
$\|\gamma_x\|_{L^1(0,1)}
\lesssim_{d,\beta}1+\sigma^{-2}$.
Combining \eqref{eq:continuous-w2-stability} with
Theorem~\ref{thm:main-generalization} and taking square roots yields the
stated nonasymptotic Wasserstein consequence.
\end{proof}

Because the fixed positive target noise keeps $q_t$ uniformly bounded
away from zero, the continuous-flow bound applies directly at the target
endpoint $t=1$.

\subsection{Euler discretization}
\label{subsec:euler-discretization}

The continuous-flow result above uses the spatial regularity built into the
learning class. To control the additional numerical error introduced by a
time discretization, we impose the following temporal regularity condition
only in this subsection.

\begin{assumption}[Temporal regularity of the learned velocity]
\label{ass:learned-temporal-regularity}
There exists a deterministic constant $\gamma_t<\infty$, independent of $n$,
such that, almost surely with respect to the target observations,
\[
    \|\widehat v(x,t)-\widehat v(x,s)\|
    \le
    \gamma_t(1+\|x\|)|t-s|,
    \qquad
    x\in\mathbb R^D,\quad s,t\in[0,1].
\]
\end{assumption}

Let
$
    0=t_0<t_1<\cdots<t_M=1,
    h:=\max_{0\le k<M}(t_{k+1}-t_k),
$
be a deterministic time grid. Starting from
$\widetilde Z_{t_0}\sim\pi_0$, define the explicit Euler scheme
\begin{equation}
    \widetilde Z_{t_{k+1}}
    =
    \widetilde Z_{t_k}
    +(t_{k+1}-t_k)
    \widehat v(\widetilde Z_{t_k},t_k),
    0\le k<M,
    \label{eq:euler-sampling-scheme}
\end{equation}
and let
$\widetilde\pi_1:=\operatorname{Law}(\widetilde Z_{t_M})$.
Set
$
    \bar\gamma_x:=\|\gamma_x\|_{L^\infty(0,1)},
    K_\sigma^{\mathrm{sam}}
    :=
    \max\left\{
        \sup_{t\in[0,1]}|a_t|,
        2\sup_{t\in[0,1]}|b_t|
    \right\}.
$

\begin{proposition}[Euler discretization error]
\label{prop:euler-sampling-error}
Under the conditions of Theorem~\ref{thm:main-generalization} and
Assumption~\ref{ass:learned-temporal-regularity}, with probability at least
$1-\delta$ over the target observations,
\begin{equation}
\begin{aligned}
    W_2(\widetilde\pi_1,\pi_1)
    \le{}&
    \frac12(1+\sqrt D)
    e^{K_\sigma^{\mathrm{sam}}+\bar\gamma_x}
    \bigl(
        \bar\gamma_xK_\sigma^{\mathrm{sam}}+\gamma_t
    \bigr)h
    \\
    &+
    \exp\!\left\{
        \frac12+\|\gamma_x\|_{L^1(0,1)}
    \right\}
    \left(
        \int_0^1
        \|\widehat v(\cdot,t)-v^*(\cdot,t)\|_{L^2(\pi_t)}^2
        \,dt
    \right)^{1/2}.
\end{aligned}
\label{eq:euler-sampling-bound}
\end{equation}
\end{proposition}

The first term in \eqref{eq:euler-sampling-bound} is the numerical
discretization error, whereas the second is the continuous-flow statistical
error controlled by Theorem~\ref{thm:main-generalization}. Consequently, the
leading algebraic contributions in the time step and sample size are
$h$ and $n^{-\beta/(2\beta+d)}$, respectively. For a uniform grid,
$h=M^{-1}$, so choosing $M$ at least of order
$n^{\beta/(2\beta+d)}$ balances these two leading scales.
The proof of Proposition~\ref{prop:euler-sampling-error} is given in
Appendix~\ref{app:sampling-euler}.

\section{Discussion and Conclusion}
\label{sec:discussion}

We studied linear Flow Matching when high-dimensional observations are generated by a low-dimensional H\"older mechanism with nondegenerate ambient noise. Although the resulting target distribution has full-dimensional support, its population velocity retains a latent posterior structure that can be approximated and learned with a leading complexity governed by the latent dimension. This yields a nonasymptotic high-probability velocity bound with leading sample-size factor $n^{-2\beta/(2\beta+d)}$, while the dependence on the ambient dimension and noise level remains explicit. By carrying the required spatial regularity through the same learning class, the velocity guarantee further propagates to the learned continuous flow, giving a $W_2$ bound with leading sample-size factor $n^{-\beta/(2\beta+d)}$. For fixed $\sigma>0$, the interpolation remains nondegenerate on the full interval $[0,1]$. Overall, the results show that exact low-dimensional support is not necessary for Flow Matching to retain latent-dimensional statistical behavior.

The analysis also identifies several boundaries of this conclusion. The ambient dimension $D$ remains in the prefactors and network complexity, and the constants deteriorate as $\sigma$ decreases, so the present theory does not provide a uniform passage to the exact-support limit. Moreover, the established rates are upper bounds rather than a minimax characterization: proving minimax optimality would require matching lower bounds for the same noisy latent-generator class and integrated velocity risk, and a separate lower-bound analysis is needed for the induced Wasserstein error. Extending the theory to vanishing noise, unknown latent dimension or smoothness, and finite auxiliary sampling remains an important direction for future work. These questions would further clarify when the statistical complexity of Flow Matching is determined by the latent generating mechanism rather than by the ambient observation space.

\subsection*{Acknowledgments}

This work was supported by the National Key R\&D Program of China
(No.~2023YFA1008701) and the Key Project of the National Natural Science
Foundation of China (No.~12431017).

\bibliography{references}
\bibliographystyle{plainnat}

\appendix
\section{Population Identities and Structural Properties}
\label{app:population-identities}

\begin{proof}[Proof of Lemma~\ref{lem:posterior-velocity}]
Write $G=g^*(U)$, $W=\sigma\xi-X_0$ and
$\eta_t=(1-t)X_0+t\sigma\xi$, so $X_t=tG+\eta_t$ and
$X_1-X_0=G+W$. The pair $(W,\eta_t)$ is centered Gaussian,
independent of $U$, with
$\operatorname{Var}(\eta_t)=q_tI_D$ and
$\operatorname{Cov}(W,\eta_t)=c_tI_D$, where
$c_t=t\sigma^2-(1-t)$. Since $q_t>0$, Gaussian conditioning gives
\begin{align*}
 \mathbb E[X_1-X_0\mid X_t,U]
 &=G+\frac{c_t}{q_t}(X_t-tG)\\
 &=a_tX_t+b_tG,
\end{align*}
using $q_t-tc_t=1-t$. Taking conditional expectation given $X_t$
proves the velocity representation almost surely.

Conditional on $U=u$, $X_t$ has Gaussian density
$(2\pi q_t)^{-D/2}\exp(-\|x-tg^*(u)\|^2/(2q_t))$.
Bayes' formula with the uniform latent prior gives the displayed
posterior representation of $m_t$. Its denominator is positive for every
$x$ and $t$, so the ratio defines a version everywhere; we take $v^*$ to be the version given in
Lemma~\ref{lem:posterior-velocity}.
The posterior is a probability measure and $\|g^*(u)\|\le1$, whence
$\|m_t(x)\|\le1$ by the triangle inequality.
At $t=0$, the posterior equals the prior and
$v^*(x,0)=\mathbb E G-x$; at $t=1$, $q_1=\sigma^2$, $a_1=1$,
$b_1=0$, and $v^*(x,1)=x$. Both agree with the defining conditional
expectations. No division by $t$ or $1-t$ is used.
\end{proof}

The same Gaussian calculation applies to any bounded surrogate generator
$g^\circ$: the posterior bound then becomes
$\|m_t^\circ(x)\|\le\|g^\circ\|_{L^\infty}$.

\subsection{Smoothness and network conventions}
\label{app:setup-definitions}

We use the global notation introduced in
Section~\ref{sec:preliminaries}. This subsection records the
multi-index, smoothness, and network-architecture conventions used in
the proofs.

Let $\mathbb Z_+:=\{0,1,2,\ldots\}$. For
$\mathbf k=(k_1,\ldots,k_p)\in\mathbb Z_+^p$, define
\[
|\mathbf k|:=\sum_{j=1}^p k_j,
\qquad
\mathbf k!:=\prod_{j=1}^p k_j!,
\qquad
z^{\mathbf k}:=\prod_{j=1}^p z_j^{k_j},
\]
and
\[
\partial^{\mathbf k}
:=
\partial_1^{k_1}\cdots\partial_p^{k_p},
\qquad
\partial^{\mathbf 0}f:=f.
\]

Let $\Omega\subseteq\mathbb R^p$. For $\beta,H>0$, the scalar
H\"older class $\mathcal H^\beta(\Omega,\mathbb R,H)$ consists of
functions $f:\Omega\to\mathbb R$ satisfying
\[
\max_{1\le|\mathbf k|\le\lfloor\beta\rfloor}
\|\partial^{\mathbf k}f\|_{L^\infty(\Omega)}
\le H,
\qquad
\max_{|\mathbf k|=\lfloor\beta\rfloor}
\sup_{\substack{x,y\in\Omega\\x\ne y}}
\frac{
|\partial^{\mathbf k}f(x)-\partial^{\mathbf k}f(y)|
}{
\min\{1,\|x-y\|_\infty\}^{\beta-\lfloor\beta\rfloor}
}
\le H.
\]
The first condition is understood as vacuous when
$\lfloor\beta\rfloor=0$. A vector-valued function
$f=(f_1,\ldots,f_m)^\top$ belongs to
$\mathcal H^\beta(\Omega,\mathbb R^m,H)$ if each coordinate
$f_j$ belongs to $\mathcal H^\beta(\Omega,\mathbb R,H)$.
The Euclidean supremum bound of a vector-valued function is imposed
separately when required, as in
\eqref{eq:generator-assumption}.

For an architecture vector
$\mathbf W=(W_0,\ldots,W_L)\in\mathbb N^{L+1}$, write
\[
\|\mathbf W\|_\infty
:=
\max_{0\le\ell\le L}W_\ell.
\]
For fixed input and output dimensions, let
$\mathrm{NN}(L,\mathbf W,S,B)$ denote the subclass of
$\mathrm{NN}(L,W,S,B)$ whose layer widths are exactly
$\mathbf W$. Equivalently,
\[
\mathrm{NN}(L,W,S,B)
=
\bigcup_{\substack{
\mathbf W=(W_0,\ldots,W_L)\\
\|\mathbf W\|_\infty\le W
}}
\mathrm{NN}(L,\mathbf W,S,B),
\]
where $W_0$ and $W_L$ are held fixed by context. For velocity networks,
$W_0=D+1$ and $W_L=D$; auxiliary networks use the input and output
dimensions stated locally.

Throughout, $L$ counts affine layers, and no activation is applied
after the final affine layer. Width padding is performed by adding zero
neurons, which preserves the realized function and does not increase
the sparsity or parameter-magnitude bounds. When networks of different
depths must be combined, depth alignment is carried out using exact
identity padding as in Lemma~\ref{lem:nn-identity-padding}, with the
corresponding complexity cost accounted for there.

\section{Proofs for Neural Approximation}
\label{app:approximation}

\subsection{Consequences of the scale conditions}
The conditions
\eqref{eq:approximation-scale-basic}--
\eqref{eq:approximation-scale-localization} are the scale assumptions
used throughout this proof. For every fixed
$(d,\beta,D,H,\sigma)$ with $\sigma>0$, the corresponding range is nonempty:
the first two left-hand sides have the required limits as
$\varepsilon\downarrow0$, while the last follows from
$\varepsilon\log(1/\varepsilon)\to0$ and
$\varepsilon\sqrt{\log(1/\varepsilon)}\to0$.
The same conditions imply
$H\sqrt D\,N^{-(\beta\wedge1)}\lesssim_{d,\beta}\sigma$ for
$N=\lceil\varepsilon^{-1}\rceil$ and provide the small-cell bound used
in the posterior construction. All conclusions remain pointwise in the
fixed noise level $\sigma>0$.

\subsection{Proof of the main approximation theorem}
\label{app:approximation-main}

\begin{proof}[Proof of Theorem~\ref{thm:velocity-approximation}]
\medskip\noindent\textbf{Step 1: Local polynomial surrogate.}
The structural smoothness lies in the latent generator, rather than in
a generic function on $\mathbb R^D$. We therefore begin in latent space:
partition $[0,1]^d$ into $N^d$ cells and replace $g^*$ on each cell by
its local Taylor polynomial. Set $r_\beta:=\lceil\beta\rceil-1$ and
$N=\lceil\varepsilon^{-1}\rceil$. For $j=1,\ldots,N-1$, let
$I_j=[(j-1)/N,j/N)$ and let $I_N=[(N-1)/N,1]$. For
$\mathbf j=(j_1,\ldots,j_d)\in\{1,\ldots,N\}^d$, define
\begin{equation}
 \mathcal U_{\mathbf j}:=\prod_{r=1}^d I_{j_r},
 \qquad u_{\mathbf j}:=\frac{\mathbf j}{N}.
 \label{eq:latent-partition}
\end{equation}
These $N^d$ half-open cells form a disjoint partition of $[0,1]^d$. On
$\overline{\mathcal U_{\mathbf j}}$, let
\begin{equation}
 g_{\mathbf j}^\circ(u):=
 \sum_{|\mathbf k|\le r_\beta}
 \frac{\partial^{\mathbf k}g^*(u_{\mathbf j})}{\mathbf k!}
 (u-u_{\mathbf j})^{\mathbf k},
 \qquad
 g^\circ(u):=\sum_{\mathbf j}g_{\mathbf j}^\circ(u)
 \mathbf 1_{\{u\in\mathcal U_{\mathbf j}\}}.
 \label{eq:g-circ-definition}
\end{equation}
The partition is the source of the factor
$N^d\lesssim_d\varepsilon^{-d}$ in the network size.
Lemma~\ref{lem:app-generator-surrogate} gives
$\|g^*-g^\circ\|_{L^\infty}
\lesssim_{d,\beta}H\sqrt D\,\varepsilon^\beta$ and
$\|g^\circ\|_{L^\infty}\le2$. The next step transfers this latent
error to the FM velocity.

\medskip\noindent\textbf{Step 2: Transfer generator error to velocity error.}
Proposition~\ref{prop:generator-velocity-stability} controls the induced
velocities directly in the integrated norm under the
true path law.
Let $X_1^\circ=g^\circ(U)+\sigma\xi$, and denote its posterior mean and
population velocity by $m_t^\circ$ and
$v^\circ(x,t)=a_tx+b_tm_t^\circ(x)$. For any candidate $v_\theta$,
\[
 v^*-v_\theta=(v^*-v^\circ)+(v^\circ-v_\theta).
\]
The first difference is the change in the conditional velocity caused by
replacing $g^*$ with $g^\circ$. Proposition~
\ref{prop:generator-velocity-stability} and
Lemma~\ref{lem:app-generator-surrogate} give
\begin{equation}
 \int_0^1
 \|v^*(\cdot,t)-v^\circ(\cdot,t)\|_{L^2(\pi_t)}^2dt
 \lesssim_{d,\beta}
 \frac{DH^2}{\sigma}\,\varepsilon^{2\beta}.
\label{eq:app-main-surrogate-velocity}
\end{equation}
It remains to realize $v^\circ$, equivalently its nonlinear component
$m_t^\circ$.

\medskip\noindent\textbf{Step 3: Localize the surrogate posterior.}
The state variable $x$ is unbounded, whereas the posterior circuit is
constructed uniformly on a controlled state region. We therefore split
the velocity error into uniform posterior approximation on a
time-dependent region $\mathcal K_t$ and a Gaussian-tail contribution
outside that region.
The surrogate posterior has the explicit representation
\begin{equation}
 m_t^\circ(x)=
 \frac{\int_{[0,1]^d}g^\circ(u)
 \exp\{-\|x-tg^\circ(u)\|^2/(2q_t)\}\,du}
 {\int_{[0,1]^d}
 \exp\{-\|x-tg^\circ(u)\|^2/(2q_t)\}\,du}.
 \label{eq:surrogate-posterior-explicit}
\end{equation}
For a measurable $f$, put
$v_f(x,t)=a_tx+b_t\operatorname{clip}_2(f(x,t))$. Since
$\|m_t^\circ(x)\|\le2$,
\[
 \|v^\circ(x,t)-v_f(x,t)\|^2
 \le b_t^2\bigl(16\wedge\|m_t^\circ(x)-f(x,t)\|^2\bigr).
\]
Let
$L_\varepsilon=\log(\varepsilon^{-2\beta}/D)$,
\[
 R_t:=\sqrt{q_tD}+16\sqrt{q_t}
 \bigl(\sqrt{DL_\varepsilon}\vee L_\varepsilon\bigr),\qquad
 \mathcal K_t:=\left\{x:
 \min_{u\in[0,1]^d}\|x-tg^*(u)\|\le R_t\right\},
\]
and
$\mathcal C_{[0,1]}^*:=\{(x,t):x\in\mathcal K_t,\ t\in[0,1]\}$.
Splitting the integral over $\mathcal K_t$ and its complement isolates the
compact-domain approximation from the Gaussian tail.
Lemma~\ref{lem:app-compact-localization} makes this decomposition
quantitative: a uniform value approximation on
$\mathcal C_{[0,1]}^*$ is sufficient, while the complement contributes
at most order $D\varepsilon^{2\beta}/\sigma$ to the velocity risk.

\medskip\noindent\textbf{Step 4: Neural realization of the surrogate posterior.}
On the localized region, the posterior is a ratio of sums over the same
latent cells used in Step 1. We first describe the value realization;
the spatial constraint is then obtained from first-order realizations of
its scalar modules.
The remaining problem is uniform approximation of the ratio in
\eqref{eq:surrogate-posterior-explicit} on
$\mathcal C_{[0,1]}^*$. Using the same cells and Taylor polynomials as in
\eqref{eq:latent-partition}--\eqref{eq:g-circ-definition}, write
\begin{equation}
 m_t^\circ(x)=\frac{P^\circ(x,t)}{Q^\circ(x,t)},\quad
 P^\circ:=\sum_{\mathbf j}P_{\mathbf j}^\circ,\quad
 Q^\circ:=\sum_{\mathbf j}Q_{\mathbf j}^\circ,
 \label{eq:posterior-cell-decomposition}
\end{equation}
where
\[
 P_{\mathbf j}^\circ
 :=\int_{\mathcal U_{\mathbf j}}g_{\mathbf j}^\circ(u)
 e^{-\|x-tg_{\mathbf j}^\circ(u)\|^2/(2q_t)}du,\qquad
 Q_{\mathbf j}^\circ
 :=\int_{\mathcal U_{\mathbf j}}
 e^{-\|x-tg_{\mathbf j}^\circ(u)\|^2/(2q_t)}du.
\]
For a fixed cell, set
$h_{\mathbf j}(u)=g_{\mathbf j}^\circ(u)-g^*(u_{\mathbf j})$.
Expanding the Gaussian exponent gives
\[
 \frac{\|x-tg_{\mathbf j}^\circ(u)\|^2}{2q_t}
 =\frac{\|x-tg^*(u_{\mathbf j})\|^2}{2q_t}
 +\frac{t^2}{2q_t}\|h_{\mathbf j}(u)\|^2
 -\frac{t}{q_t}(x-tg^*(u_{\mathbf j}))^\top h_{\mathbf j}(u).
\]
After the change of variables from $\mathcal U_{\mathbf j}$ to the unit
cube, $h_{\mathbf j}$ is a polynomial in $d$ variables with coefficients
given by the derivatives of $g^*$ at $u_{\mathbf j}$. Each cell integral
therefore depends on $(x,t)$ through finitely many scalar coefficients:
the local Taylor coefficients paired with
$x-tg^*(u_{\mathbf j})$, their quadratic products, and
$1/q_t,t/q_t,t^2/q_t$. This is the low-dimensional
exponential-integral map approximated by the network: its nonlinear
input dimension depends on $(d,\beta)$, even though $x\in\mathbb R^D$.

The network is assembled in the same order as this formula. Subnetworks first
approximate the time coefficients and scalar coordinates, followed by the
cellwise exponential integrals. Their outputs are summed over the $N^d$
cells. A uniform lower bound for $Q^\circ$ permits reciprocal and
multiplication networks to form $P^\circ/Q^\circ$, and the $D$ scalar
coordinates are parallelized.

\emph{Compatibility with the velocity class.}
The class in Theorem~\ref{thm:velocity-approximation} also requires a
spatial Lipschitz envelope. Gaussian smoothing makes $m_t^\circ$ smooth
in $x$, and first-order realizations of the scalar cell modules preserve
the value approximation. Lemma~\ref{lem:app-surrogate-posterior-network}
provides the resulting network, its global spatial bound, and its
architecture.

\medskip\noindent\textbf{Step 5: Assemble the comparator.}
Steps 1--4 provide the latent surrogate error, its velocity transfer,
the localization reduction, and a spatially regular posterior network.
Insert that network into the structured velocity formula and decompose
the final error as
\[
 v^*-v_\theta=(v^*-v^\circ)+(v^\circ-v_\theta).
\]
Set $v_\theta(x,t)=a_tx+b_t\operatorname{clip}_2(f_\theta(x,t))$.
The factor $\sqrt D$ in \eqref{eq:app-main-posterior-error} converts the
coordinatewise approximation into Euclidean output error; the coefficients
$a_t,b_t$ remain exact. Since Euclidean projection is nonexpansive,
Lemmas~\ref{lem:app-compact-localization} and
\ref{lem:app-surrogate-posterior-network} yield
\begin{equation}
 \int_0^1
 \|v^\circ(\cdot,t)-v_\theta(\cdot,t)\|_{L^2(\pi_t)}^2dt
 \lesssim_{d,\beta}\frac{D}{\sigma}\varepsilon^{2\beta}.
 \label{eq:app-main-network-velocity}
\end{equation}
Combining this with \eqref{eq:app-main-surrogate-velocity} gives
\[
 \int_0^1
 \|v^*(\cdot,t)-v_\theta(\cdot,t)\|_{L^2(\pi_t)}^2dt
 \lesssim_{d,\beta}
 \frac{D(1+H^2)}{\sigma}\,\varepsilon^{2\beta},
\]
which is \eqref{eq:velocity-approximation-rate}. The construction above has
$N^d\asymp_d\varepsilon^{-d}$ parallel cell modules, which determine the
algebraic dependence on the approximation scale. Time-coefficient networks, scalar
exponential-integral primitives, stable division, the first-order
refinement, and the global gate satisfy the explicit
$\Gamma_\varepsilon$ bounds in
Lemma~\ref{lem:app-surrogate-posterior-network}. Euclidean clipping is
nonexpansive, so the same field also obeys
\[
 \operatorname{Lip}_x(v_\theta(\cdot,t))
 \le |a_t|+|b_t|\operatorname{Lip}_x(f_\theta(\cdot,t))
 \le\gamma_x(t).
\]
Taking the fixed constant in the definition of $\gamma_x$ sufficiently
large, depending only on $(d,\beta)$, makes the right-hand side no larger
than $\gamma_x(t)$.
The elementary bounds $|a_t|,|b_t|\le2/((1-t)+\sigma t)$ and
$q_t\ge((1-t)+\sigma t)^2/2$ give the supremum and integral estimates
stated in Theorem~\ref{thm:velocity-approximation}; hence the risk, profile,
and architecture bounds all hold for this one
$v_\theta\in\mathcal V(L,W,S,B,\gamma_x)$. Thus the spatially
regular comparator retains the same latent-dimensional leading
approximation complexity.
\end{proof}

\subsection{Latent surrogate and localization}

The results below supply the latent approximation, its FM stability
transfer, and the reduction from the unbounded state space to
$\mathcal C_{[0,1]}^*$.

\subsubsection{Latent Taylor approximation}

\begin{lemma}[Latent Taylor surrogate]
\label{lem:app-generator-surrogate}
The surrogate defined above satisfies
\begin{equation}
 \|g^*-g^\circ\|_{L^\infty([0,1]^d)}
 \lesssim_{d,\beta}H\sqrt D\,\varepsilon^\beta.
 \label{eq:app-main-generator-surrogate}
\end{equation}
Under the scale assumptions of Theorem~\ref{thm:velocity-approximation},
$\|g^\circ\|_{L^\infty([0,1]^d)}\le2$.
\end{lemma}

\begin{proof}
Fix a cell and write $h=u-u_{\mathbf j}$, so
$\|h\|_\infty\le N^{-1}$. We first bound one coordinate $g_l^*$.
If $\beta\notin\mathbb N$, write $\beta=r_\beta+\alpha$ with
$\alpha\in(0,1)$. For $r_\beta=0$, the defining H\"older condition gives
$|g_l^*(u)-g_l^*(u_{\mathbf j})|\le H\|h\|_\infty^\beta$.
For $r_\beta\ge1$, apply the one-dimensional integral Taylor formula to
$\varphi(s)=g_l^*(u_{\mathbf j}+sh)$. After subtracting the order
$r_\beta$ Taylor polynomial, the remainder is
\[
 \frac{1}{(r_\beta-1)!}\int_0^1(1-s)^{r_\beta-1}
 \{\varphi^{(r_\beta)}(s)-\varphi^{(r_\beta)}(0)\}\,ds.
\]
The coordinatewise H\"older condition and the multinomial formula imply
that its absolute value is at most
$C_{d,\beta}H\|h\|_\infty^\beta$.

If $\beta=m\in\mathbb N$, then $r_\beta=m-1$. The integral remainder
after the degree-$(m-1)$ Taylor polynomial is
\[
 \frac{1}{(m-1)!}\int_0^1(1-s)^{m-1}\varphi^{(m)}(s)\,ds.
\]
All order-$m$ partial derivatives are bounded by $H$ under the stated
H\"older-class convention, so the multinomial formula bounds this remainder
by $C_{d,m}H\|h\|_\infty^m$. Thus, in both cases,
\[
 \sup_{u\in\overline{\mathcal U_{\mathbf j}}}
 |g_l^*(u)-g_{\mathbf j,l}^\circ(u)|
 \le C_{d,\beta}HN^{-\beta},
 \qquad 1\le l\le D.
\]
Converting these coordinatewise bounds once to the Euclidean norm yields
\begin{equation}
 \sup_{u\in\overline{\mathcal U_{\mathbf j}}}
 \|g^*(u)-g_{\mathbf j}^\circ(u)\|
 \le C_{d,\beta}H\sqrt D\,N^{-\beta}.
 \label{eq:g-local-cell-error}
\end{equation}
The half-open cells are disjoint and cover $[0,1]^d$, while all anchors
$u_{\mathbf j}$ belong to this cube. Hence
\begin{equation}
 \|g^*-g^\circ\|_{L^\infty([0,1]^d)}
 \le C_{d,\beta}H\sqrt D\,N^{-\beta}
 \le C_{d,\beta}H\sqrt D\,\varepsilon^\beta.
 \label{eq:g-local-poly-error}
\end{equation}
The second condition in \eqref{eq:approximation-scale-basic}, decreased
through $c_{d,\beta}$ if necessary, makes this error at most one because
$\varepsilon^\beta\le\varepsilon^{\beta\wedge1}$. Together with
$\|g^*\|_{L^\infty}\le1$, this gives
$\|g^\circ\|_{L^\infty}\le2$.
\end{proof}

\subsubsection{Generator-to-velocity stability}
\label{app:generator-velocity-stability}
\begin{proof}[Proof of Proposition~\ref{prop:generator-velocity-stability}]
This proposition connects latent-generator approximation to the integrated
FM velocity risk. The score representation used below has an apparent
singularity at $t=0$, while relative entropy supplies the compensating
quadratic behavior in time.

The true and surrogate interpolants have the Gaussian-mixture
representations
\[
    X_t\overset{d}{=}tg^*(U)+\sqrt{q_t}\,Z,
    \qquad
    X_t^\circ\overset{d}{=}tg^\circ(U)+\sqrt{q_t}\,Z,
    \qquad Z\sim\mathcal N(0,I_D).
\]
The lower bound on $q_t$ from Section~\ref{sec:preliminaries} implies that
their densities $p_t^*$ and $p_t^\circ$ are smooth and positive throughout
$[0,1]$. Let $K(t):=D_{\mathrm{KL}}(p_t^*\,\|\,p_t^\circ)$. Introduce the
joint laws $\mathsf R_t^*$ and $\mathsf R_t^\circ$ of $(U,X_t)$ under
$g^*$ and $g^\circ$, respectively. Data processing under the projection
$(U,X_t)\mapsto X_t$, followed by conditioning on the common uniform
variable $U$, gives
\begin{equation}
\begin{aligned}
    K(t)
    &\le D_{\mathrm{KL}}(\mathsf R_t^*\,\|\,\mathsf R_t^\circ) \\
    &=
    \mathbb E_U D_{\mathrm{KL}}\!\left(
      \mathcal N(tg^*(U),q_tI_D)\,\middle\|\,
      \mathcal N(tg^\circ(U),q_tI_D)
    \right) \\
    &\le \frac{t^2}{2q_t}
    \|g^*-g^\circ\|_{L^\infty([0,1]^d)}^2.
\end{aligned}
\label{eq:fixed-time-kl-bound}
\end{equation}
In particular, $K(0)=0$ and $K(t)=O(t^2)$ as $t\downarrow0$.

For $t\in(0,1)$, write
$s_t^*=\nabla\log p_t^*$ and $s_t^\circ=\nabla\log p_t^\circ$.
Gaussian conditioning and Lemma~\ref{lem:posterior-velocity} give
\[
    v^*(x,t)=\frac{x}{t}+\frac{1-t}{t}s_t^*(x),
    \qquad
    v^\circ(x,t)=\frac{x}{t}+\frac{1-t}{t}s_t^\circ(x),
\]
and hence
\begin{equation}
    v^*(x,t)-v^\circ(x,t)
    =\frac{1-t}{t}\bigl(s_t^*(x)-s_t^\circ(x)\bigr).
\label{eq:velocity-score-difference}
\end{equation}
The two marginal paths satisfy their continuity equations with velocities
$v^*$ and $v^\circ$. Bounded generators and the uniform positive lower bound
on $q_t$ give smooth positive Gaussian-mixture densities with Gaussian
envelopes sufficient for differentiation under the integral and spatial
integration by parts. Differentiating $K(t)$ along the two continuity
equations first gives
\[
    K'(t)=\int_{\mathbb R^D}p_t^*(x)
    \bigl(v^*(x,t)-v^\circ(x,t)\bigr)\cdot
    \bigl(s_t^*(x)-s_t^\circ(x)\bigr)\,dx.
\]
Substituting \eqref{eq:velocity-score-difference} yields
\[
    K'(t)=\frac{1-t}{t}
    \int_{\mathbb R^D}\|s_t^*(x)-s_t^\circ(x)\|^2p_t^*(x)\,dx
\]
and therefore
\begin{equation*}
    \int_{\mathbb R^D}\|v^*(x,t)-v^\circ(x,t)\|^2p_t^*(x)\,dx
    =\frac{1-t}{t}K'(t).
\end{equation*}
For $\delta\in(0,1)$, integration by parts gives
\[
 \int_\delta^1\frac{1-t}{t}K'(t)\,dt
 =-\frac{1-\delta}{\delta}K(\delta)
   +\int_\delta^1\frac{K(t)}{t^2}\,dt.
\]
The boundary term vanishes as $\delta\downarrow0$ because
$K(\delta)=O(\delta^2)$. Combining the resulting identity with
that bound gives
\[
\begin{aligned}
    \int_0^1\!\int_{\mathbb R^D}
    \|v^*(x,t)-v^\circ(x,t)\|^2p_t^*(x)\,dx\,dt
    &\le \frac12\|g^*-g^\circ\|_{L^\infty([0,1]^d)}^2
      \int_0^1\frac{dt}{q_t} \\
    &=\frac{\pi}{4\sigma}
      \|g^*-g^\circ\|_{L^\infty([0,1]^d)}^2,
\end{aligned}
\]
where
$\int_0^1((1-t)^2+\sigma^2t^2)^{-1}dt=\pi/(2\sigma)$.
Thus the $O(t^2)$ relative-entropy behavior exactly compensates for the
source-endpoint weight. The target endpoint is harmless because the weight
$(1-t)/t$ vanishes there.
\end{proof}

\subsubsection{Localization of the surrogate posterior}

\begin{lemma}[Compact localization of the surrogate posterior]
\label{lem:app-compact-localization}
The set $\mathcal C_{[0,1]}^*$ defined above is compact and, for every
measurable $f:\mathbb R^D\times[0,1]\to\mathbb R^D$,
\begin{equation}
\begin{aligned}
 \int_0^1\|v^\circ-v_f\|_{L^2(\pi_t)}^2dt
 &\lesssim \frac{D\varepsilon^{2\beta}}{\sigma}\\
 &\quad+\int_0^1\int_{\mathcal K_t}
 b_t^2\|m_t^\circ(x)-f(x,t)\|^2\,d\pi_t(x)\,dt.
\end{aligned}
\label{eq:app-main-compact-reduction}
\end{equation}
\end{lemma}

\begin{proof}
The representation \eqref{eq:surrogate-posterior-explicit} and
\eqref{eq:g-local-poly-error} imply
\begin{equation}
    \|g^\circ\|_{L^\infty([0,1]^d)}\le2.
    \label{eq:g-circ-bound}
\end{equation}
Thus $\|m_t^\circ(x)\|\le2$ and
$m_t^\circ(x)=\operatorname{clip}_2(m_t^\circ(x))$ on
$\mathbb R^D\times[0,1]$.

Let $f_\theta:\mathbb R^D\times[0,1]\to\mathbb R^D$ be measurable and
let $v_\theta$ be its structured velocity.

Since the Euclidean projection onto the ball
$\mathcal B(0,2)$ is nonexpansive and both clipped vectors have norm at
most two,
\begin{equation*}
\begin{aligned}
    \|v^\circ(x,t)-v_\theta(x,t)\|^2
    &=
    b_t^2
    \left\|
        \operatorname{clip}_2(m_t^\circ(x))
        -
        \operatorname{clip}_2(f_\theta(x,t))
    \right\|^2
    \\
    &\le
    b_t^2
    \left(
        16
        \wedge
        \|m_t^\circ(x)-f_\theta(x,t)\|^2
    \right).
\end{aligned}
\end{equation*}

Splitting the integral over the set $\mathcal K_t$ defined in the main
proof and its complement gives
\begin{equation}
\begin{aligned}
    &
    \int_{\mathbb R^D}
    \|v^\circ(x,t)-v_\theta(x,t)\|^2
    p_t^*(x)\,dx
    \\
    &\qquad\le
    b_t^2
    \int_{\mathcal K_t}
    \|m_t^\circ(x)-f_\theta(x,t)\|^2
    p_t^*(x)\,dx
    \\
    &\qquad\quad+
    16b_t^2
    \int_{\mathbb R^D\setminus\mathcal K_t}
    p_t^*(x)\,dx.
\end{aligned}
\label{eq:compact-tail-decomposition}
\end{equation}
The map $\operatorname{clip}_2$ is deterministic post-processing; the
ReLU-network complexity is that of the underlying map $f_\theta$.

The second term is controlled by Gaussian concentration. Indeed, the
true interpolation satisfies
\[
    X_t
    \overset{d}{=}
    tg^*(U)+\sqrt{q_t}\,Z.
\]
Therefore, by the same chi-squared concentration argument as in
\cite[Lemma A.2]{yakovlev2025generalization}, with
$\widetilde\sigma_t^2$ replaced by $q_t$ and $m_t$ replaced by $t$,
one obtains, whenever $R_t^2\ge Dq_t$,
\begin{equation}
\begin{aligned}
    \int_{\mathbb R^D\setminus\mathcal K_t}
    p_t^*(x)\,dx
    \le
    \exp\left\{
        -\frac{1}{16}
        \left(
            \frac{
                R_t^2-Dq_t
            }{
                Dq_t
            }
            \wedge
            \frac{
                \sqrt{R_t^2-Dq_t}
            }{
                \sqrt{q_t}
            }
        \right)
    \right\}.
\end{aligned}
\label{eq:fm-gaussian-tail}
\end{equation}

The choice of $R_t$ in the main proof and the assumption
$\varepsilon^{-2\beta}>D$ imply from \eqref{eq:fm-gaussian-tail} that
\begin{equation*}
    \int_{\mathbb R^D\setminus\mathcal K_t}
    p_t^*(x)\,dx
    \lesssim
    D\varepsilon^{2\beta},
    \qquad
    t\in[0,1].
\end{equation*}

Integrating \eqref{eq:compact-tail-decomposition} over time therefore
yields
\begin{equation*}
\begin{aligned}
    &
    \int_0^1
    \int_{\mathbb R^D}
    \|v^\circ(x,t)-v_\theta(x,t)\|^2
    p_t^*(x)\,dx\,dt
    \\
    &\qquad\lesssim
    D\varepsilon^{2\beta}
    \int_0^1 b_t^2\,dt
    \\
    &\qquad\quad+
    \int_0^1
    \int_{\mathcal K_t}
    b_t^2
    \|m_t^\circ(x)-f_\theta(x,t)\|^2
    p_t^*(x)\,dx\,dt.
\end{aligned}
\end{equation*}
Since
\begin{equation*}
    \int_0^1 b_t^2\,dt
    =
    \int_0^1
    \frac{(1-t)^2}{
        \bigl((1-t)^2+t^2\sigma^2\bigr)^2
    }
    \,dt
    =
    \frac{\pi}{4\sigma},
\end{equation*}
we further obtain
\begin{equation}
\begin{aligned}
    &
    \int_0^1
    \int_{\mathbb R^D}
    \|v^\circ(x,t)-v_\theta(x,t)\|^2
    p_t^*(x)\,dx\,dt
    \\
    &\qquad\lesssim
    \frac{
        D\varepsilon^{2\beta}
    }{
        \sigma
    }
    +
    \int_0^1
    \int_{\mathcal K_t}
    b_t^2
    \|m_t^\circ(x)-f_\theta(x,t)\|^2
    p_t^*(x)\,dx\,dt.
\end{aligned}
\label{eq:fm-compact-reduction-final}
\end{equation}

Since $g^*$ is continuous on the compact set $[0,1]^d$ and
$t\mapsto R_t$ is continuous, the map
\[
    (x,t)
    \longmapsto
    \min_{u\in[0,1]^d}
    \|x-tg^*(u)\|-R_t
\]
is continuous. Hence $\mathcal C_{[0,1]}^*$ is closed.
Moreover, the definition of $\mathcal K_t$ and
\eqref{eq:generator-assumption} imply a uniform bound on $\|x\|$
over $\mathcal C_{[0,1]}^*$.
Therefore $\mathcal C_{[0,1]}^*$ is compact.

Notice that, since $\sigma>0$,
\[
    q_t
    \ge
    \frac{\sigma^2}{1+\sigma^2}
    >0,
    \qquad t\in[0,1],
\]
so the above construction is uniformly nondegenerate on the entire
time interval, including $t=0$. This allows us to carry out the
subsequent neural network approximation directly on
$\mathcal C_{[0,1]}^*$ without introducing a positive stopping time.
This proves the compactness and reduction asserted in
Lemma~\ref{lem:app-compact-localization}.
\end{proof}

\subsection{Neural realization of the surrogate posterior}
\label{app:surrogate-posterior-network}

\begin{lemma}[Neural approximation of the surrogate posterior]
\label{lem:app-surrogate-posterior-network}
Under the assumptions of Theorem~\ref{thm:velocity-approximation}, there is a
$D$-output ReLU network $f_\theta\in\mathrm{NN}(L,W,S,B)$ such that
\begin{equation}
 \sup_{(x,t)\in\mathcal C_{[0,1]}^*}
 \|f_\theta(x,t)-m_t^\circ(x)\|_2
 \lesssim_{d,\beta}\sqrt D\,\varepsilon^\beta.
 \label{eq:app-main-posterior-error}
\end{equation}
For every $t\in[0,1]$, the same network satisfies
\begin{equation}
 \operatorname{Lip}_x(f_\theta(\cdot,t))
 \le C_{d,\beta}(1+\sigma^{-1}+t/q_t).
 \label{eq:app-posterior-spatial-bound}
\end{equation}
Its architecture satisfies
\begin{equation}
 L,\log B\lesssim_{d,\beta}
 \Gamma_\varepsilon^{c_{\rm arch}},
 \qquad W,S\lesssim_{d,\beta}D^2\varepsilon^{-d}
 \Gamma_\varepsilon^{c_{\rm arch}}.
 \label{eq:app-regular-posterior-architecture}
\end{equation}
\end{lemma}

The construction proceeds from the latent-cell value approximation to a
first-order realization on an enlarged tube, and then to a global network.

\subsubsection{Value realization of the latent-cell posterior}

The decomposition \eqref{eq:posterior-cell-decomposition} separates the
dependence on $(x,t)$ from the polynomial variable on each latent cell. We
now construct uniform approximations of its numerator and denominator on
$\mathcal C_{[0,1]}^*$.

Fix $\mathbf j\in\{1,\ldots,N\}^d$. The exponent expansion derived in
the main proof will now be encoded by explicit scalar coordinates.

Retain the Taylor degree $r:=r_\beta=\lceil\beta\rceil-1$ from the
main proof and set
\[
 m_\beta^{\rm Tay}:=\binom{d+r_\beta}{d},
 \qquad m_\beta:=\binom{d+\lfloor\beta\rfloor}{d}.
\]
Thus $m_\beta^{\rm Tay}\le m_\beta$. The first quantity is the actual
number of Taylor monomials, whereas the second is a convenient
$(d,\beta)$-dependent upper bound used in the fixed-dimensional cell
construction.
From the definition of the local Taylor polynomial,
\begin{equation}
    g_{\mathbf j}^\circ(u)-g^*(u_{\mathbf j})
    =
    \sum_{\substack{
        \mathbf k\in\mathbb Z_+^d\\
        1\le|\mathbf k|\le r
    }}
    \frac{
        \partial^{\mathbf k}g^*(u_{\mathbf j})
    }{
        \mathbf k!
    }
    (u-u_{\mathbf j})^{\mathbf k},
    \qquad
    u\in\mathcal U_{\mathbf j}.
    \label{eq:local-poly-without-constant}
\end{equation}

Define
\begin{equation}
    V(t)
    :=
    \frac{t^2}{2q_t},
    \qquad
    V_{\mathbf j,0}(x,t)
    :=
    \frac{
        \|x-tg^*(u_{\mathbf j})\|^2
    }{
        2q_t
    },
    \label{eq:fm-V-and-Vj0}
\end{equation}
and, for
$\mathbf k\in\mathbb Z_+^d$ with
$1\le|\mathbf k|\le r$,
\begin{equation}
\begin{aligned}
    V_{\mathbf j,\mathbf k}(x,t)
    &:=
    -
    \frac{t}{q_t}
    \left.
    \partial_u^{\mathbf k}
    \left[
        \bigl(
            x-tg^*(u_{\mathbf j})
        \bigr)^\top
        g^*(u)
    \right]
    \right|_{u=u_{\mathbf j}}
    \\
    &=
    -
    \frac{t}{q_t}
    \bigl(
        x-tg^*(u_{\mathbf j})
    \bigr)^\top
    \partial^{\mathbf k}g^*(u_{\mathbf j}).
\end{aligned}
\label{eq:fm-Vjk}
\end{equation}

With this notation,
\begin{equation}
\begin{aligned}
    \frac{
        \|x-tg_{\mathbf j}^\circ(u)\|^2
    }{
        2q_t
    }
    &=
    V_{\mathbf j,0}(x,t)
    +
    V(t)
    \left\|
        g_{\mathbf j}^\circ(u)-g^*(u_{\mathbf j})
    \right\|^2
    \\
    &\quad+
    \sum_{\substack{
        \mathbf k\in\mathbb Z_+^d\\
        1\le|\mathbf k|\le r
    }}
    V_{\mathbf j,\mathbf k}(x,t)
    \frac{
        (u-u_{\mathbf j})^{\mathbf k}
    }{
        \mathbf k!
    }.
\end{aligned}
\label{eq:fm-exponent-composition}
\end{equation}

For each fixed $\mathbf j$, collect the functions
$V_{\mathbf j,\mathbf k}$ together with $V$ into
\begin{equation*}
    \mathcal V_{\mathbf j}(x,t)
    :=
    \left(
        \left(
            V_{\mathbf j,\mathbf k}(x,t)
        \right)_{
            \substack{
                \mathbf k\in\mathbb Z_+^d\\
                1\le|\mathbf k|\le r
            }
        },
        V(t)
    \right)^\top.
\end{equation*}
Since
\[
    \#\left\{
        \mathbf k\in\mathbb Z_+^d:
        |\mathbf k|\le r
    \right\}
    =
    \binom{d+r}{d}=m_\beta^{\rm Tay},
\]
the vector $\mathcal V_{\mathbf j}$ has dimension
\begin{equation*}
    \dim(\mathcal V_{\mathbf j})
    =
    \binom{d+r}{d}=m_\beta^{\rm Tay}.
\end{equation*}

For each fixed $\mathbf j$, the representation
\eqref{eq:fm-exponent-composition} shows that the dependence of the
$\mathbf j$-th cell integrals on $(x,t)$ is entirely captured by
$V_{\mathbf j,0}(x,t)$ and $\mathcal V_{\mathbf j}(x,t)$.
Indeed, after factoring out
$\exp\{-V_{\mathbf j,0}(x,t)\}$, the remaining integral is a function
of $\mathcal V_{\mathbf j}(x,t)$ only.

Since
\[
    \#\left\{
        \mathbf k\in\mathbb Z_+^d:
        1\le |\mathbf k|\le r
    \right\}
    =
    \binom{d+r}{d}-1=m_\beta^{\rm Tay}-1,
\]
the vector $\mathcal V_{\mathbf j}$ has dimension
\[
    \binom{d+r}{d}=m_\beta^{\rm Tay},
\]
and hence each cell integral is a composition involving only
$m_\beta^{\rm Tay}+1\le m_\beta+1$ scalar intermediate variables,
including $V_{\mathbf j,0}$. This dimension depends only on $(d,\beta)$
and not on the ambient input dimension $D$.

The required neural components are $V_{\mathbf j,0}$,
$V_{\mathbf j,\mathbf k}$, and $V$.

\paragraph{Scalar coordinate networks.}

We now approximate the functions
$V_{\mathbf j,0}(x,t)$,
$V_{\mathbf j,\mathbf k}(x,t)$,
$1\le|\mathbf k|\le r$, and $V(t)$ introduced in
\eqref{eq:fm-V-and-Vj0}--\eqref{eq:fm-Vjk}.
Throughout this step, let
\[
M
:=
1\vee
\sup_{(x,t)\in\mathcal C_{[0,1]}^*}
\|x\|_\infty.
\]
By the construction of $\mathcal C_{[0,1]}^*$ above,
\[
M
\le
1\vee
\sup_{t\in[0,1]}(R_t+1),
\]
and hence $M<\infty$.
We first consider $V_{\mathbf j,0}$. Since
$g_{\mathbf j}^\circ(u_{\mathbf j})=g^*(u_{\mathbf j})$, we have
\begin{equation*}
\begin{aligned}
V_{\mathbf j,0}(x,t)
&=
\frac{\|x-tg^*(u_{\mathbf j})\|^2}{2q_t}\\
&=
\frac{\|x\|^2}{2q_t}
-
\frac{t\,x^\top g^*(u_{\mathbf j})}{q_t}
+
\frac{\|g^*(u_{\mathbf j})\|^2}{2}
\frac{t^2}{q_t}.
\end{aligned}
\end{equation*}
The three terms on the right-hand side are precisely of the forms
treated in Lemmas~\ref{lem:fm-quadratic-term},
\ref{lem:fm-linear-term}, and
\ref{lem:fm-time-coefficients}, respectively.
Since
$\|g^*\|_{L^\infty([0,1]^d)}\le1$, applying these lemmas with
accuracy $\varepsilon'/3$ yields the ReLU network
\begin{equation*}
\widetilde V_{\mathbf j,0}(x,t)
:=
\rho_{\varepsilon'/3}(x,t)
-
\omega_{\varepsilon'/3}(x,t)
+
\frac{\|g^*(u_{\mathbf j})\|^2}{2}
\chi_{2,\varepsilon'/3}(t),
\end{equation*}
where Lemma~\ref{lem:fm-linear-term} is used with
$a=g^*(u_{\mathbf j})$. Consequently,
\begin{equation*}
\left\|
\widetilde V_{\mathbf j,0}
-
V_{\mathbf j,0}
\right\|_{L^\infty(\mathcal C_{[0,1]}^*)}
\le
\varepsilon'.
\end{equation*}

We next consider
$V_{\mathbf j,\mathbf k}$,
$1\le|\mathbf k|\le r$. By \eqref{eq:fm-Vjk},
\begin{equation*}
\begin{aligned}
V_{\mathbf j,\mathbf k}(x,t)
&=
-
\frac{t}{q_t}
\bigl(
x-tg^*(u_{\mathbf j})
\bigr)^\top
\partial^{\mathbf k}g^*(u_{\mathbf j})\\
&=
-
\frac{
t\,x^\top\partial^{\mathbf k}g^*(u_{\mathbf j})
}{
q_t
}
+
g^*(u_{\mathbf j})^\top
\partial^{\mathbf k}g^*(u_{\mathbf j})
\frac{t^2}{q_t}.
\end{aligned}
\end{equation*}
The first term is approximated by
Lemma~\ref{lem:fm-linear-term} with
$a=\partial^{\mathbf k}g^*(u_{\mathbf j})$.
For the second term, set
\[
\delta_{\mathbf j,\mathbf k}
:=
\frac{\varepsilon'}
{2\bigl(
|g^*(u_{\mathbf j})^\top
\partial^{\mathbf k}g^*(u_{\mathbf j})|
\vee1
\bigr)}.
\]
Using Lemma~\ref{lem:fm-linear-term} with accuracy
$\varepsilon'/2$ and
Lemma~\ref{lem:fm-time-coefficients} with
$\gamma=2$ and accuracy
$\delta_{\mathbf j,\mathbf k}$, define
\begin{equation*}
\begin{aligned}
\widetilde V_{\mathbf j,\mathbf k}(x,t)
:={}&
-
\omega_{\varepsilon'/2}(x,t)\\
&+
g^*(u_{\mathbf j})^\top
\partial^{\mathbf k}g^*(u_{\mathbf j})
\,
\chi_{2,\delta_{\mathbf j,\mathbf k}}(t),
\end{aligned}
\end{equation*}
where $\omega_{\varepsilon'/2}$ is constructed with
$a=\partial^{\mathbf k}g^*(u_{\mathbf j})$.
Then
\begin{equation*}
\left\|
\widetilde V_{\mathbf j,\mathbf k}
-
V_{\mathbf j,\mathbf k}
\right\|_{L^\infty(\mathcal C_{[0,1]}^*)}
\le
\varepsilon',
\qquad
1\le|\mathbf k|\le r.
\end{equation*}

Finally, $V(t)=t^2/(2q_t)$ is treated directly by
Lemma~\ref{lem:fm-time-coefficients}. Namely, setting
\begin{equation*}
\widetilde V(t)
:=
\frac12\chi_{2,\varepsilon'}(t),
\end{equation*}
we have
\begin{equation*}
\left\|
\widetilde V-V
\right\|_{L^\infty([0,1])}
\le
\varepsilon'.
\end{equation*}

We now collect the corresponding network configurations.
Since $g^*\in\mathcal H^\beta([0,1]^d,\mathbb R^D,H)$,
\[
\|\partial^{\mathbf k}g^*(u_{\mathbf j})\|_\infty
\le H,
\qquad
1\le|\mathbf k|\le r,
\]
and
\[
\left|
g^*(u_{\mathbf j})^\top
\partial^{\mathbf k}g^*(u_{\mathbf j})
\right|
\lesssim
\sqrt D\,H.
\]
Thus, uniformly over
$\mathbf j\in\{1,\ldots,N\}^d$ and
$1\le|\mathbf k|\le r$, all the networks constructed above may be
embedded into a common class
\[
\mathrm{NN}
\bigl(
\widetilde L,
\widetilde{\mathbf W},
\widetilde S,
\widetilde B
\bigr),
\]
with
\begin{equation}
\begin{aligned}
\widetilde L\vee\log\widetilde B
&\lesssim_{d,\beta}
\log^2(1/\varepsilon')
+\log^2\!\bigl(MD(H\vee1)\bigr)
+\log^2\!\left(\frac{1+\sigma^2}{\sigma^2}\right),\\
\|\widetilde{\mathbf W}\|_\infty
&\lesssim_{d,\beta}
D\left[
\log^3(1/\varepsilon')
+\log^3\!\bigl(MD(H\vee1)\bigr)
+\log^3\!\left(\frac{1+\sigma^2}{\sigma^2}\right)
\right],\\
\widetilde S
&\lesssim_{d,\beta}
D\left[
\log^4(1/\varepsilon')
+\log^4\!\bigl(MD(H\vee1)\bigr)
+\log^4\!\left(\frac{1+\sigma^2}{\sigma^2}\right)
\right].
\end{aligned}
\label{eq:fm-step4-common-configuration}
\end{equation}
The dependence on $H$ remains explicit in this logarithmic factor and
is retained in the final envelope $\Gamma_\varepsilon$.

Then, uniformly over
$\mathbf j\in\{1,\ldots,N\}^d$,
\begin{equation*}
\max\left\{
    \|\widetilde V-V\|_{L^\infty([0,1])},
    \,
    \max_{\substack{
        \mathbf k\in\mathbb Z_+^d\\
        1\le|\mathbf k|\le r
    }}
    \|\widetilde V_{\mathbf j,\mathbf k}
      -V_{\mathbf j,\mathbf k}\|_
      {L^\infty(\mathcal C_{[0,1]}^*)}
\right\}
\le
\varepsilon',
\end{equation*}
and
\begin{equation*}
    \|\widetilde V_{\mathbf j,0}
      -V_{\mathbf j,0}\|_
      {L^\infty(\mathcal C_{[0,1]}^*)}
    \le
    \varepsilon'.
\end{equation*}

\paragraph{Low-dimensional cell integrals.}

We use the scalar coordinate networks above to
approximate the local numerator and denominator integrals in
\eqref{eq:posterior-cell-decomposition}. Fix
$\mathbf j\in\{1,\ldots,N\}^d$ and consider the change of variables
\[
    u
    =
    u_{\mathbf j}-\frac{w}{N},
    \qquad
    w\in[0,1]^d.
\]
By the definition of the partition in
\eqref{eq:latent-partition}, this map sends $[0,1]^d$ onto
$\overline{\mathcal U_{\mathbf j}}$ and satisfies
$du=N^{-d}dw$. Since
$\overline{\mathcal U_{\mathbf j}}\setminus\mathcal U_{\mathbf j}$
has Lebesgue measure zero, for every integrable function $F$,
\begin{equation}
    N^d
    \int_{\mathcal U_{\mathbf j}}
    F(u)\,du
    =
    \int_{[0,1]^d}
    F\left(
        u_{\mathbf j}-\frac{w}{N}
    \right)dw .
    \label{eq:fm-cell-change-of-variables}
\end{equation}

By \eqref{eq:fm-exponent-composition},
\begin{equation*}
\begin{aligned}
&\frac{
    \|x-tg_{\mathbf j}^\circ
    (u_{\mathbf j}-N^{-1}w)\|^2
}{
    2q_t
}
\\
&\qquad=
V_{\mathbf j,0}(x,t)
+
V(t)
\left\|
    g_{\mathbf j}^\circ
    (u_{\mathbf j}-N^{-1}w)
    -g^*(u_{\mathbf j})
\right\|^2
\\
&\qquad\quad+
\sum_{\substack{
    \mathbf k\in\mathbb Z_+^d\\
    1\le|\mathbf k|\le r
}}
V_{\mathbf j,\mathbf k}(x,t)
\frac{
    (-N^{-1}w)^{\mathbf k}
}{
    \mathbf k!
}.
\end{aligned}
\end{equation*}

When $r=0$, all families and sums indexed by
$\mathbf k\in\mathbb Z_+^d$ with $1\le|\mathbf k|\le r$
are understood to be absent. In this case, the constructions below
contain only the coordinate corresponding to $V(t)$.

For every
$\mathbf k\in\mathbb Z_+^d$ with $1\le|\mathbf k|\le r$, set
\[
    C_{\mathbf j,\mathbf k}
    :=
    1\vee
    \|V_{\mathbf j,\mathbf k}\|_
    {L^\infty(\mathcal C_{[0,1]}^*)},
\]
and define
\begin{equation}
R_{\mathbf j}(x,t)
:=
\left(
\left(
    \frac{
        V_{\mathbf j,\mathbf k}(x,t)
    }{
        2C_{\mathbf j,\mathbf k}
    }
    +\frac12
\right)_{
    \substack{
        \mathbf k\in\mathbb Z_+^d\\
        1\le|\mathbf k|\le r
    }
},
\,
\frac{
    V(t)
}{
    \|V\|_{L^\infty([0,1])}
}
\right)^\top .
\label{eq:fm-Rj}
\end{equation}
By construction,
\[
    R_{\mathbf j}(x,t)
    \in
    [0,1]^{\binom{d+r}{d}}.
\]

For $w\in[0,1]^d$, define
\begin{equation}
a_{\mathbf j}(w)
:=
\left(
\left(
    \frac{
        2C_{\mathbf j,\mathbf k}
        (-N^{-1}w)^{\mathbf k}
    }{
        \mathbf k!
    }
\right)_{
    \substack{
        \mathbf k\in\mathbb Z_+^d\\
        1\le|\mathbf k|\le r
    }
},
\,
\|V\|_{L^\infty([0,1])}
\left\|
    g_{\mathbf j}^\circ
    (u_{\mathbf j}-N^{-1}w)
    -g^*(u_{\mathbf j})
\right\|^2
\right)^\top
\label{eq:fm-aj}
\end{equation}
and
\begin{equation}
b_{\mathbf j}(w)
:=
-
\sum_{\substack{
    \mathbf k\in\mathbb Z_+^d\\
    1\le|\mathbf k|\le r
}}
C_{\mathbf j,\mathbf k}
\frac{
    (-N^{-1}w)^{\mathbf k}
}{
    \mathbf k!
}.
\label{eq:fm-bj}
\end{equation}
Then the decomposition is algebraically exact:
\begin{equation*}
\frac{
    \|x-tg_{\mathbf j}^\circ
    (u_{\mathbf j}-N^{-1}w)\|^2
}{
    2q_t
}
=
V_{\mathbf j,0}(x,t)
+
R_{\mathbf j}(x,t)^\top a_{\mathbf j}(w)
+
b_{\mathbf j}(w).
\end{equation*}

\paragraph{Auxiliary estimates for the local-integral construction.}
\label{subsec:fm-step5-auxiliary}

The following results are used in the proof of
Lemma~\ref{lem:fm-local-integral}. The first result controls the
normalized intermediate coordinates, the second is the
low-dimensional integral approximation result of
\cite{yakovlev2025generalization}, and the third provides the
FM-specific bounds required to verify its assumptions.

\paragraph{Normalization bounds.}
Let $F,\widetilde F:\Omega\to\mathbb R$ satisfy
\[
    \|F\|_{L^\infty(\Omega)}\le C,\qquad
    C\ge1,\qquad
    \|\widetilde F-F\|_{L^\infty(\Omega)}\le\eta .
\]
For $z\in\mathbb R$, let
\[
    \operatorname{clip}_{[0,1]}(z)
    :=\min\{1,\max\{0,z\}\}.
\]
Define
\[
    R:=\frac{F}{2C}+\frac12,\qquad
    \widetilde R
    :=
    \operatorname{clip}_{[0,1]}
    \left(
        \frac{\widetilde F}{2C}+\frac12
    \right).
\]
Then $R,\widetilde R\in[0,1]$ and
\[
    \|\widetilde R-R\|_{L^\infty(\Omega)}
    \le
    \frac{\eta}{2C}.
\]
If, in addition, $F\ge0$ and
$\|F\|_{L^\infty(\Omega)}>0$, then
\[
    R_+:=
    \frac{F}{\|F\|_{L^\infty(\Omega)}},
    \qquad
    \widetilde R_+
    :=
    \operatorname{clip}_{[0,1]}
    \left(
        \frac{\widetilde F}
        {\|F\|_{L^\infty(\Omega)}}
    \right)
\]
satisfy
\[
    \|\widetilde R_+-R_+\|_{L^\infty(\Omega)}
    \le
    \frac{
        \|\widetilde F-F\|_{L^\infty(\Omega)}
    }{
        \|F\|_{L^\infty(\Omega)}
    }.
\]

\begin{proof}
The map $\operatorname{clip}_{[0,1]}$ is one-Lipschitz and equals the
identity on $[0,1]$. Since
\[
    \frac{F}{2C}+\frac12\in[0,1],
\]
we obtain
\[
\begin{aligned}
|\widetilde R-R|
&=
\left|
\operatorname{clip}_{[0,1]}
\left(
    \frac{\widetilde F}{2C}+\frac12
\right)
-
\operatorname{clip}_{[0,1]}
\left(
    \frac{F}{2C}+\frac12
\right)
\right|
\\
&\le
\frac{|\widetilde F-F|}{2C}.
\end{aligned}
\]
Taking the supremum over $\Omega$ proves the first claim.

If $F\ge0$, then
\[
    \frac{F}{\|F\|_{L^\infty(\Omega)}}\in[0,1].
\]
The same one-Lipschitz argument gives
\[
\left|
\widetilde R_+-R_+
\right|
\le
\frac{
    |\widetilde F-F|
}{
    \|F\|_{L^\infty(\Omega)}
},
\]
which proves the second claim.
\end{proof}

\begin{lemma}[Low-dimensional exponential-integral approximation;
Lemma A.6 of~\cite{yakovlev2025generalization}]
\label{lem:low-dimensional-integral-map}
Let $m\in\mathbb N$ and let
\[
    \varphi:[0,1]^d\to\mathbb R,\qquad
    a:[0,1]^d\to\mathbb R^m,\qquad
    b:[0,1]^d\to\mathbb R .
\]
Consider
\[
    \Phi(v)
    :=
    \int_{[0,1]^d}
    \varphi(w)
    \exp\{-v^\top a(w)-b(w)\}\,dw,
    \qquad
    v\in[0,1]^m.
\]
Suppose that, for some
$\varphi_{\max}\ge1$, $a_{\max}\ge1$, and $A\ge0$,
\[
    |\varphi(w)|\le\varphi_{\max},\qquad
    \|a(w)\|_\infty\le a_{\max},
\]
and
\[
    -v^\top a(w)-b(w)\le A
\]
for every $v\in[0,1]^m$ and $w\in[0,1]^d$.
Then, for every $\delta\in(0,1)$, there exists a ReLU network
\[
    \widetilde\Phi
    \in
    \mathrm{NN}(L_\Phi,\mathbf W_\Phi,S_\Phi,1)
\]
such that
\[
    \|\widetilde\Phi-\Phi\|_{L^\infty([0,1]^m)}
    \le\delta .
\]
Moreover,
\[
L_\Phi
\lesssim_{\varphi_{\max},a_{\max},A}
\left(
    m+1+\log\frac1\delta
\right)
\log\left(
    e+m+\log\frac1\delta
\right)
\log\left(
    e+\log\frac1\delta
\right),
\]
\[
\|\mathbf W_\Phi\|_\infty
\lesssim_{\varphi_{\max},a_{\max},A}
\left(
    m+1+\log\frac1\delta
\right)^{m+1},
\qquad
S_\Phi
\lesssim_{\varphi_{\max},a_{\max},A}
\left(
    m+1+\log\frac1\delta
\right)^{2m+5}.
\]
The hidden constants depend only on
$\varphi_{\max}$, $a_{\max}$, and $A$.
\end{lemma}

\paragraph{Bounds for the FM intermediate coefficients.}
With the notation introduced in Steps~3--5, uniformly over
$\mathbf j\in\{1,\ldots,N\}^d$,
\[
\|V_{\mathbf j,0}\|_
{L^\infty(\mathcal C_{[0,1]}^*)}
\lesssim
\frac{1+\sigma^2}{\sigma^2}
D(M+1)^2,
\]
\[
\|V\|_{L^\infty([0,1])}
=
\frac1{2\sigma^2},
\]
and, for every
$\mathbf k\in\mathbb Z_+^d$ with $1\le|\mathbf k|\le r$,
\[
\|V_{\mathbf j,\mathbf k}\|_
{L^\infty(\mathcal C_{[0,1]}^*)}
\lesssim
\frac{1+\sigma^2}{\sigma^2}
D H(M+1).
\]
Furthermore,
\begin{equation}
\sup_{w\in[0,1]^d}
\left\|
g_{\mathbf j}^\circ
\left(
    u_{\mathbf j}-\frac{w}{N}
\right)
-g^*(u_{\mathbf j})
\right\|
\lesssim
\sqrt D\,H
\left(
    e^{d/N}-1
\right).
\label{eq:fm-local-poly-increment-bound}
\end{equation}
Consequently, if \eqref{eq:fm-step5-small-cell} holds, then
\begin{equation}
\sup_{w\in[0,1]^d}
\|a_{\mathbf j}(w)\|_\infty
\le1
\label{eq:fm-aj-uniform-bound}
\end{equation}
and
\begin{equation}
\sup_{\substack{
v\in[0,1]^{\binom{d+r}{d}}\\
w\in[0,1]^d
}}
\left\{
    -v^\top a_{\mathbf j}(w)-b_{\mathbf j}(w)
\right\}
\le1,
\label{eq:fm-exponent-uniform-bound}
\end{equation}
provided the constant $c_0$ in
\eqref{eq:fm-step5-small-cell} is chosen sufficiently small.

\begin{proof}
We first bound $V_{\mathbf j,0}$. By
\eqref{eq:fm-V-and-Vj0},
\[
V_{\mathbf j,0}(x,t)
=
\frac{
    \|x-tg^*(u_{\mathbf j})\|^2
}{
    2q_t
}.
\]
On $\mathcal C_{[0,1]}^*$,
$\|x\|_2\le\sqrt D\,M$, while
$\|g^*(u_{\mathbf j})\|_2\le1$.
Since
\[
q_t
\ge
\frac{\sigma^2}{1+\sigma^2},
\]
we have
\[
\begin{aligned}
V_{\mathbf j,0}(x,t)
&\le
\frac{1+\sigma^2}{2\sigma^2}
\left(
    \sqrt D\,M+1
\right)^2
\\
&\lesssim
\frac{1+\sigma^2}{\sigma^2}
D(M+1)^2 .
\end{aligned}
\]

Next,
\[
V(t)=\frac{t^2}{2q_t}.
\]
For $t>0$,
\[
\frac{t^2}{q_t}
=
\frac{
    1
}{
    ((1-t)/t)^2+\sigma^2
}
\le
\frac1{\sigma^2},
\]
and equality is attained at $t=1$. Since $V(0)=0$,
\[
\|V\|_{L^\infty([0,1])}
=
\frac1{2\sigma^2}.
\]

For
$\mathbf k\in\mathbb Z_+^d$ with $1\le|\mathbf k|\le r$,
\eqref{eq:fm-Vjk} gives
\[
\begin{aligned}
|V_{\mathbf j,\mathbf k}(x,t)|
&\le
\frac{t}{q_t}
\|x-tg^*(u_{\mathbf j})\|_2
\|\partial^{\mathbf k}g^*(u_{\mathbf j})\|_2
\\
&\le
\frac{1+\sigma^2}{\sigma^2}
(\sqrt D\,M+1)\sqrt D\,H
\\
&\lesssim
\frac{1+\sigma^2}{\sigma^2}
D H(M+1).
\end{aligned}
\]

We next control the local Taylor increment. By
\eqref{eq:local-poly-without-constant},
\[
g_{\mathbf j}^\circ
\left(
    u_{\mathbf j}-\frac{w}{N}
\right)
-g^*(u_{\mathbf j})
=
\sum_{\substack{
\mathbf k\in\mathbb Z_+^d\\
1\le|\mathbf k|\le r
}}
\frac{
    \partial^{\mathbf k}g^*(u_{\mathbf j})
}{
    \mathbf k!
}
\left(
    -\frac{w}{N}
\right)^{\mathbf k}.
\]
Since
\[
\|\partial^{\mathbf k}g^*(u_{\mathbf j})\|_2
\le
\sqrt D\,H,
\]
the multinomial theorem yields
\[
\begin{aligned}
&
\left\|
g_{\mathbf j}^\circ
\left(
    u_{\mathbf j}-\frac{w}{N}
\right)
-g^*(u_{\mathbf j})
\right\|
\\
&\qquad\le
\sqrt D\,H
\sum_{m=1}^r
\sum_{\substack{
\mathbf k\in\mathbb Z_+^d\\
|\mathbf k|=m
}}
\frac{
    N^{-m}w^{\mathbf k}
}{
    \mathbf k!
}
\\
&\qquad=
\sqrt D\,H
\sum_{m=1}^r
\frac{
    (N^{-1}\|w\|_1)^m
}{
    m!
}
\\
&\qquad\le
\sqrt D\,H
\left(
    e^{d/N}-1
\right),
\end{aligned}
\]
which proves
\eqref{eq:fm-local-poly-increment-bound}.

We now verify the two bounds required by
Lemma~\ref{lem:low-dimensional-integral-map}.
For every Taylor coordinate of $a_{\mathbf j}$,
\[
\left|
\frac{
    2C_{\mathbf j,\mathbf k}
    (-N^{-1}w)^{\mathbf k}
}{
    \mathbf k!
}
\right|
\le
\frac{
    2C_{\mathbf j,\mathbf k}
}{
    N^{|\mathbf k|}
}
\le
\frac{
    2C_{\mathbf j,\mathbf k}
}{
    N
}.
\]
By the bound on $V_{\mathbf j,\mathbf k}$ above,
\[
\frac{
    C_{\mathbf j,\mathbf k}
}{
    N
}
\lesssim
\frac1N
+
\frac1N
\frac{1+\sigma^2}{\sigma^2}
D(H\vee1)(M+1).
\]
Hence \eqref{eq:fm-step5-small-cell} implies, after reducing $c_0$
if necessary,
\[
\sup_{\substack{
\mathbf k\in\mathbb Z_+^d\\
1\le|\mathbf k|\le r
}}
\sup_{w\in[0,1]^d}
\left|
\frac{
    2C_{\mathbf j,\mathbf k}
    (-N^{-1}w)^{\mathbf k}
}{
    \mathbf k!
}
\right|
\le1 .
\]

For the last coordinate of $a_{\mathbf j}$,
\[
\begin{aligned}
&
\|V\|_{L^\infty([0,1])}
\left\|
g_{\mathbf j}^\circ
\left(
    u_{\mathbf j}-\frac{w}{N}
\right)
-g^*(u_{\mathbf j})
\right\|^2
\\
&\qquad\le
\frac{
    D H^2
}{
    2\sigma^2
}
\left(
    e^{d/N}-1
\right)^2
\le1 ,
\end{aligned}
\]
again under \eqref{eq:fm-step5-small-cell}, after decreasing $c_0$
if necessary. This proves
\eqref{eq:fm-aj-uniform-bound}.

Finally, write
\[
v
=
\left(
(v_{\mathbf k})_
{\substack{
\mathbf k\in\mathbb Z_+^d\\
1\le|\mathbf k|\le r
}},
v_V
\right)
\in
[0,1]^{\binom{d+r}{d}}.
\]
Using \eqref{eq:fm-aj} and \eqref{eq:fm-bj},
\[
\begin{aligned}
&
-v^\top a_{\mathbf j}(w)
-b_{\mathbf j}(w)
\\
&=
\sum_{\substack{
\mathbf k\in\mathbb Z_+^d\\
1\le|\mathbf k|\le r
}}
(1-2v_{\mathbf k})
C_{\mathbf j,\mathbf k}
\frac{
    (-N^{-1}w)^{\mathbf k}
}{
    \mathbf k!
}
\\
&\qquad
-
v_V
\|V\|_{L^\infty([0,1])}
\left\|
g_{\mathbf j}^\circ
\left(
    u_{\mathbf j}-\frac{w}{N}
\right)
-g^*(u_{\mathbf j})
\right\|^2 .
\end{aligned}
\]
The last term is non-positive. Therefore
\[
\begin{aligned}
-v^\top a_{\mathbf j}(w)-b_{\mathbf j}(w)
&\le
\max_{\substack{
\mathbf k\in\mathbb Z_+^d\\
1\le|\mathbf k|\le r
}}
C_{\mathbf j,\mathbf k}
\sum_{m=1}^r
\frac{(d/N)^m}{m!}
\\
&\le
\max_{\substack{
\mathbf k\in\mathbb Z_+^d\\
1\le|\mathbf k|\le r
}}
C_{\mathbf j,\mathbf k}
\left(
    e^{d/N}-1
\right).
\end{aligned}
\]
Since
\[
    e^{d/N}-1
    \lesssim_d
    \frac{1}{N},
\]
and
\[
    C_{\mathbf j,\mathbf k}
    =
    1\vee
    \|V_{\mathbf j,\mathbf k}\|_
    {L^\infty(\mathcal C_{[0,1]}^*)}
    \lesssim
    1+
    \frac{1+\sigma^2}{\sigma^2}
    D H(M+1),
\]
the small-cell condition
\eqref{eq:fm-step5-small-cell} implies, after reducing $c_0$ if
necessary, that
\[
    \max_{\substack{
        \mathbf k\in\mathbb Z_+^d\\
        1\le|\mathbf k|\le r
    }}
    C_{\mathbf j,\mathbf k}
    \left(
        e^{d/N}-1
    \right)
    \le 1.
\]
Therefore,
\[
\sup_{\substack{
    v\in[0,1]^{\binom{d+r}{d}}\\
    w\in[0,1]^d
}}
\left\{
    -v^\top a_{\mathbf j}(w)
    -b_{\mathbf j}(w)
\right\}
\le 1,
\]
which proves
\eqref{eq:fm-exponent-uniform-bound}.
\end{proof}
The following lemma combines these estimates into the neural approximation
of a normalized local integral.

\begin{lemma}[Approximation of a normalized local FM integral]
\label{lem:fm-local-integral}
Let $\varepsilon'\in(0,1)$ and fix
$\mathbf j\in\{1,\ldots,N\}^d$.
Let
$\psi_{\mathbf j}:[0,1]^d\to\mathbb R$ satisfy
\[
    \|\psi_{\mathbf j}\|_{L^\infty([0,1]^d)}
    \le2,
\]
and define
\begin{equation}
\Upsilon_{\mathbf j}(x,t)
:=
e^{-V_{\mathbf j,0}(x,t)}
\int_{[0,1]^d}
\psi_{\mathbf j}(w)
\exp\left\{
    -R_{\mathbf j}(x,t)^\top a_{\mathbf j}(w)
    -b_{\mathbf j}(w)
\right\}dw .
\label{eq:fm-Upsilon-j}
\end{equation}
Assume that
\begin{equation}
\frac{1}{N}
\left[
    1+
    \frac{1+\sigma^2}{\sigma^2}
    D(H\vee1)(M+1)
\right]
\le c_0,
\label{eq:fm-step5-small-cell}
\end{equation}
where $c_0>0$ is sufficiently small and depends only on $d$ and
$\beta$. Then there exists a ReLU network
\[
    \widetilde\Upsilon_{\mathbf j}
    \in
    \mathrm{NN}
    (L_\Upsilon,\mathbf W_\Upsilon,S_\Upsilon,B_\Upsilon)
\]
such that
\begin{equation}
\left\|
    \widetilde\Upsilon_{\mathbf j}
    -\Upsilon_{\mathbf j}
\right\|_
{L^\infty(\mathcal C_{[0,1]}^*)}
\lesssim_{d,\beta}
\frac{\varepsilon'}{N}.
\label{eq:fm-Upsilon-error}
\end{equation}
Moreover,
\begin{equation}
\begin{aligned}
L_\Upsilon
&\lesssim_{d,\beta}
\log^2\!\left(\frac{N}{\varepsilon'}\right)
+\log^2\!\bigl(MD(H\vee1)\bigr)
+\log^2\!\left(\frac{1+\sigma^2}{\sigma^2}\right)
\\
&\qquad+
\left(
    \log N+\log\frac1{\varepsilon'}
\right)
\log^2\!\left(
    e+\log N+\log\frac1{\varepsilon'}
\right),
\\[1mm]
\|\mathbf W_\Upsilon\|_\infty
&\lesssim_{d,\beta}
D\left[
\log^3\!\left(\frac{N}{\varepsilon'}\right)
+\log^3\!\bigl(MD(H\vee1)\bigr)
+\log^3\!\left(\frac{1+\sigma^2}{\sigma^2}\right)
\right]
\\
&\qquad\vee
\left(
    1+\log N+\log\frac1{\varepsilon'}
\right)^{
    \binom{d+r}{d}+1
},
\\[1mm]
S_\Upsilon
&\lesssim_{d,\beta}
D\left[
\log^4\!\left(\frac{N}{\varepsilon'}\right)
+\log^4\!\bigl(MD(H\vee1)\bigr)
+\log^4\!\left(\frac{1+\sigma^2}{\sigma^2}\right)
\right]
\\
&\qquad+
\left(
    1+\log N+\log\frac1{\varepsilon'}
\right)^{
    2\binom{d+r}{d}+5
},
\\[1mm]
\log B_\Upsilon
&\lesssim_{d,\beta}
\log^2\!\left(\frac{N}{\varepsilon'}\right)
+\log^2\!\bigl(MD(H\vee1)\bigr)
+\log^2\!\left(\frac{1+\sigma^2}{\sigma^2}\right).
\end{aligned}
\label{eq:fm-Upsilon-configuration}
\end{equation}
The hidden constants depend only on $d$ and $\beta$.
\end{lemma}

\begin{proof}
Choose the scalar-coordinate networks with the refined
accuracy allocation
\begin{equation}
\begin{aligned}
    \|\widetilde V_{\mathbf j,0}
      -V_{\mathbf j,0}\|_
      {L^\infty(\mathcal C_{[0,1]}^*)}
    &\le
    \frac{\varepsilon'}{N},
\\
    \|\widetilde V-V\|_{L^\infty([0,1])}
    &\le
    \frac{\varepsilon'}{N},
\end{aligned}
\label{eq:fm-step5-refined-accuracy}
\end{equation}
while, for every
$\mathbf k\in\mathbb Z_+^d$ with $1\le|\mathbf k|\le r$,
\begin{equation}
    \|\widetilde V_{\mathbf j,\mathbf k}
      -V_{\mathbf j,\mathbf k}\|_
      {L^\infty(\mathcal C_{[0,1]}^*)}
    \le
    \varepsilon'.
    \label{eq:fm-step5-Vjk-accuracy}
\end{equation}
Since these constructions admit arbitrary prescribed
accuracies, the common configuration
\eqref{eq:fm-step4-common-configuration} remains valid after replacing
$\log(1/\varepsilon')$ by $\log(N/\varepsilon')$.

\smallskip\noindent\textbf{Normalized-coordinate approximation.}

For
$\mathbf k\in\mathbb Z_+^d$ with $1\le|\mathbf k|\le r$, define
\[
\widetilde R_{\mathbf j,\mathbf k}(x,t)
:=
\operatorname{clip}_{[0,1]}
\left(
    \frac{
        \widetilde V_{\mathbf j,\mathbf k}(x,t)
    }{
        2C_{\mathbf j,\mathbf k}
    }
    +\frac12
\right),
\]
and
\[
\widetilde R_{\mathbf j,V}(t)
:=
\operatorname{clip}_{[0,1]}
\left(
    \frac{
        \widetilde V(t)
    }{
        \|V\|_{L^\infty([0,1])}
    }
\right).
\]
Let $\widetilde R_{\mathbf j}$ denote the vector obtained by
concatenating these coordinates in the same order as in
\eqref{eq:fm-Rj}. By construction,
\[
\widetilde R_{\mathbf j}(x,t)
\in
[0,1]^{\binom{d+r}{d}}.
\]

The normalization bounds above together with
\eqref{eq:fm-step5-refined-accuracy} and
\eqref{eq:fm-step5-Vjk-accuracy} yields
\[
\left|
\widetilde R_{\mathbf j,\mathbf k}
-
R_{\mathbf j,\mathbf k}
\right|
\le
\frac{\varepsilon'}{2C_{\mathbf j,\mathbf k}},
\qquad
1\le|\mathbf k|\le r,
\]
and
\[
\left|
\widetilde R_{\mathbf j,V}
-
R_{\mathbf j,V}
\right|
\le
\frac{
    \varepsilon'/N
}{
    \|V\|_{L^\infty([0,1])}
}.
\]
Hence, by the definition of $a_{\mathbf j}$,
\begin{align*}
&
\left|
a_{\mathbf j}(w)^\top
\left(
    \widetilde R_{\mathbf j}(x,t)
    -
    R_{\mathbf j}(x,t)
\right)
\right|
\nonumber\\
&\qquad\le
\varepsilon'
\sum_{\substack{
\mathbf k\in\mathbb Z_+^d\\
1\le|\mathbf k|\le r
}}
\frac{
    |(-N^{-1}w)^{\mathbf k}|
}{
    \mathbf k!
}
\nonumber\\
&\qquad\quad+
\frac{\varepsilon'}{N}
\left\|
g_{\mathbf j}^\circ
\left(
    u_{\mathbf j}-\frac{w}{N}
\right)
-g^*(u_{\mathbf j})
\right\|^2 .
\end{align*}

By the multinomial theorem,
\begin{align*}
\sum_{\substack{
\mathbf k\in\mathbb Z_+^d\\
1\le|\mathbf k|\le r
}}
\frac{
    |(-N^{-1}w)^{\mathbf k}|
}{
    \mathbf k!
}
&\le
e^{d/N}-1
\nonumber\\
&\lesssim_d
\frac1N .
\end{align*}
Moreover, by \eqref{eq:g-circ-bound} and
\eqref{eq:generator-assumption},
\[
\left\|
g_{\mathbf j}^\circ
\left(
    u_{\mathbf j}-\frac{w}{N}
\right)
-g^*(u_{\mathbf j})
\right\|
\le
\|g^\circ\|_{L^\infty([0,1]^d)}
+
\|g^*\|_{L^\infty([0,1]^d)}
\le3.
\]
Consequently,
\begin{equation}
\sup_{w\in[0,1]^d}
\left|
a_{\mathbf j}(w)^\top
\left(
    \widetilde R_{\mathbf j}
    -
    R_{\mathbf j}
\right)
\right|
\lesssim_{d,\beta}
\frac{\varepsilon'}{N}.
\label{eq:fm-weighted-R-error-final}
\end{equation}

\smallskip\noindent\textbf{Low-dimensional integral map.}

Define
\begin{equation*}
\Phi_{\mathbf j}(v)
:=
\int_{[0,1]^d}
\psi_{\mathbf j}(w)
\exp\left\{
    -v^\top a_{\mathbf j}(w)
    -b_{\mathbf j}(w)
\right\}\,dw,
\qquad
v\in
[0,1]^{\binom{d+r}{d}}.
\end{equation*}
Then
\[
\Psi_{\mathbf j}(x,t)
:=
\int_{[0,1]^d}
\psi_{\mathbf j}(w)
\exp\left\{
    -R_{\mathbf j}(x,t)^\top a_{\mathbf j}(w)
    -b_{\mathbf j}(w)
\right\}\,dw
=
\Phi_{\mathbf j}(R_{\mathbf j}(x,t)).
\]

By the intermediate-coefficient bounds above,
the assumptions of
Lemma~\ref{lem:low-dimensional-integral-map} hold with
$\varphi_{\max}=2$, $a_{\max}=1$, and $A=1$.
Applying that lemma with
\[
    \delta=\frac{\varepsilon'}{N}
\]
gives a ReLU network
$\widetilde\Phi_{\mathbf j}$ satisfying
\begin{equation}
\left\|
\widetilde\Phi_{\mathbf j}
-
\Phi_{\mathbf j}
\right\|_
{L^\infty([0,1]^{\binom{d+r}{d}})}
\le
\frac{\varepsilon'}{N}.
\label{eq:fm-Phi-j-error}
\end{equation}

Since the input dimension of $\Phi_{\mathbf j}$ is
$\binom{d+r}{d}$, its configuration satisfies
\[
L_\Phi
\lesssim_{d,\beta}
\left(
    1+\log N+\log\frac1{\varepsilon'}
\right)
\log^2\left(
    e+\log N+\log\frac1{\varepsilon'}
\right),
\]
\[
\|\mathbf W_\Phi\|_\infty
\lesssim_{d,\beta}
\left(
    1+\log N+\log\frac1{\varepsilon'}
\right)^{
    \binom{d+r}{d}+1
},
\]
and
\[
S_\Phi
\lesssim_{d,\beta}
\left(
    1+\log N+\log\frac1{\varepsilon'}
\right)^{
    2\binom{d+r}{d}+5
},
\]
where the dependence on $d$ and $\beta$ is absorbed into the hidden
constants.

\smallskip\noindent\textbf{Composition with the normalized coordinates.}

Set
\[
\widetilde\Psi_{\mathbf j}(x,t)
:=
\widetilde\Phi_{\mathbf j}
\left(
    \widetilde R_{\mathbf j}(x,t)
\right).
\]
Since both
$R_{\mathbf j}(x,t)$ and
$\widetilde R_{\mathbf j}(x,t)$ belong to
$[0,1]^{\binom{d+r}{d}}$,
\begin{align*}
&
\left|
\widetilde\Psi_{\mathbf j}(x,t)
-
\Psi_{\mathbf j}(x,t)
\right|
\nonumber\\
&\qquad\le
\left|
\widetilde\Phi_{\mathbf j}
(\widetilde R_{\mathbf j})
-
\Phi_{\mathbf j}
(\widetilde R_{\mathbf j})
\right|
+
\left|
\Phi_{\mathbf j}
(\widetilde R_{\mathbf j})
-
\Phi_{\mathbf j}
(R_{\mathbf j})
\right|.
\end{align*}
The first term is bounded by $\varepsilon'/N$ by
\eqref{eq:fm-Phi-j-error}.

For the second term, differentiating under the integral gives
\[
\nabla\Phi_{\mathbf j}(v)
=
-
\int_{[0,1]^d}
\psi_{\mathbf j}(w)
a_{\mathbf j}(w)
\exp\left\{
    -v^\top a_{\mathbf j}(w)
    -b_{\mathbf j}(w)
\right\}\,dw .
\]
By the Newton--Leibniz formula,
\begin{align}
&
\left|
\Phi_{\mathbf j}
(\widetilde R_{\mathbf j})
-
\Phi_{\mathbf j}
(R_{\mathbf j})
\right|
\nonumber\\
&\qquad\le
2
\sup_{\substack{
v\in[0,1]^{\binom{d+r}{d}}\\
w\in[0,1]^d
}}
\exp\left\{
    -v^\top a_{\mathbf j}(w)
    -b_{\mathbf j}(w)
\right\}
\nonumber\\
&\qquad\quad\times
\sup_{w\in[0,1]^d}
\left|
a_{\mathbf j}(w)^\top
\left(
    \widetilde R_{\mathbf j}
    -
    R_{\mathbf j}
\right)
\right|.
\end{align}
Using
\eqref{eq:fm-exponent-uniform-bound} and
\eqref{eq:fm-weighted-R-error-final},
\[
\left|
\Phi_{\mathbf j}
(\widetilde R_{\mathbf j})
-
\Phi_{\mathbf j}
(R_{\mathbf j})
\right|
\lesssim_{d,\beta}
\frac{\varepsilon'}{N}.
\]
Consequently,
\begin{equation}
\left\|
\widetilde\Psi_{\mathbf j}
-
\Psi_{\mathbf j}
\right\|_
{L^\infty(\mathcal C_{[0,1]}^*)}
\lesssim_{d,\beta}
\frac{\varepsilon'}{N}.
\label{eq:fm-Psi-final-error}
\end{equation}

The network
$\widetilde R_{\mathbf j}$ is obtained by parallelizing a fixed
number $\binom{d+r}{d}$ of the coordinate subnetworks and appending
exact ReLU clipping maps. Hence its depth is of the same order as
their depth, while its width and sparsity increase only by a
factor depending on $d$ and $\beta$.
Composing it with $\widetilde\Phi_{\mathbf j}$ therefore gives the
first two complexity blocks in
\eqref{eq:fm-Upsilon-configuration}.

\smallskip\noindent\textbf{Outer exponential and final product.}

By construction,
\[
    V_{\mathbf j,0}(x,t)\ge0 .
\]
For the local-integral construction, the coordinate networks are
instantiated at accuracy
$\varepsilon'/N$, and hence
\eqref{eq:fm-step5-refined-accuracy} gives
\[
\left\|
\widetilde V_{\mathbf j,0}
-
V_{\mathbf j,0}
\right\|_
{L^\infty(\mathcal C_{[0,1]}^*)}
\le
\frac{\varepsilon'}{N}.
\]
Therefore
\[
    \widetilde V_{\mathbf j,0}(x,t)
    \ge
    -\frac{\varepsilon'}{N}.
\]
Applying Corollary~\ref{cor:relu-exp-shift} with
\[
    a=\frac{\varepsilon'}{N}
\]
and exponential approximation accuracy $\varepsilon'/N$, we obtain
a ReLU network $\widetilde E_{\mathbf j,0}$ satisfying
\begin{equation*}
\left\|
\widetilde E_{\mathbf j,0}
-
e^{-V_{\mathbf j,0}}
\right\|_
{L^\infty(\mathcal C_{[0,1]}^*)}
\lesssim
\frac{\varepsilon'}{N}.
\end{equation*}
Moreover,
\[
L_{\exp}
\lesssim
\log^2\left(\frac{N}{\varepsilon'}\right),
\qquad
\|\mathbf W_{\exp}\|_\infty
\lesssim
\log\left(\frac{N}{\varepsilon'}\right),
\]
\[
S_{\exp}
\lesssim
\log^2\left(\frac{N}{\varepsilon'}\right),
\qquad
\log B_{\exp}
\lesssim
\log^2\left(\frac{N}{\varepsilon'}\right).
\]

By \eqref{eq:fm-exponent-uniform-bound} and
$\|\psi_{\mathbf j}\|_\infty\le2$,
\[
    |\Psi_{\mathbf j}(x,t)|
    \le 2e,
\]
and \eqref{eq:fm-Psi-final-error} therefore implies
$|\widetilde\Psi_{\mathbf j}(x,t)|\lesssim1$.
Likewise,
\[
    e^{-V_{\mathbf j,0}(x,t)}\le1,
    \qquad
    |\widetilde E_{\mathbf j,0}(x,t)|\lesssim1.
\]

Apply the multiplication-network lemma from the neural-network
toolkit on a fixed bounded interval with base accuracy
$\varepsilon'/N$. For a sufficiently large universal constant $C$,
its perturbation estimate gives
\begin{align}
&
\left|
\phi_{\mathrm{prod}}
\left(
    \widetilde E_{\mathbf j,0},
    \widetilde\Psi_{\mathbf j}
\right)
-
e^{-V_{\mathbf j,0}}
\Psi_{\mathbf j}
\right|
\nonumber\\
&\qquad\le
\frac{\varepsilon'}{N}
+
2C
\left(
\left|
\widetilde E_{\mathbf j,0}
-
e^{-V_{\mathbf j,0}}
\right|
\vee
\left|
\widetilde\Psi_{\mathbf j}
-
\Psi_{\mathbf j}
\right|
\right)
\nonumber\\
&\qquad\lesssim_{d,\beta}
\frac{\varepsilon'}{N}.
\end{align}
Thus
\[
\widetilde\Upsilon_{\mathbf j}
:=
\phi_{\mathrm{prod}}
\left(
    \widetilde E_{\mathbf j,0},
    \widetilde\Psi_{\mathbf j}
\right)
\]
satisfies
\[
\left\|
\widetilde\Upsilon_{\mathbf j}
-
\Upsilon_{\mathbf j}
\right\|_
{L^\infty(\mathcal C_{[0,1]}^*)}
\lesssim_{d,\beta}
\frac{\varepsilon'}{N}.
\]

The networks used here belong to the common class
\eqref{eq:fm-step4-common-configuration} with
$\log(1/\varepsilon')$ replaced by
$\log(N/\varepsilon')$.
The parallelization producing
$\widetilde R_{\mathbf j}$ changes width and sparsity only by a
factor depending on $d$ and $\beta$.
The low-dimensional network
$\widetilde\Phi_{\mathbf j}$ contributes
\[
\left(
    1+\log N+\log\frac1{\varepsilon'}
\right)^{
    \binom{d+r}{d}+1
}
\]
to the width and
\[
\left(
    1+\log N+\log\frac1{\varepsilon'}
\right)^{
    2\binom{d+r}{d}+5
}
\]
to the sparsity.
The exponential and multiplication networks contribute only lower
polylogarithmic terms. Combining these bounds with the standard
parallelization and composition rules yields precisely
\eqref{eq:fm-Upsilon-configuration}.
\end{proof}

We first apply Lemma~\ref{lem:fm-local-integral} with
$\psi_{\mathbf j}\equiv1$. By
\eqref{eq:fm-cell-change-of-variables}, there exists a scalar ReLU
network $Q_{\mathbf j}$ such that
\begin{equation}
\left|
N^d
\int_{\mathcal U_{\mathbf j}}
\exp\left\{
    -\frac{
        \|x-tg_{\mathbf j}^\circ(u)\|^2
    }{
        2q_t
    }
\right\}du
-
Q_{\mathbf j}(x,t)
\right|
\lesssim_{d,\beta}
\frac{\varepsilon'}{N}.
\label{eq:fm-Qj-local-error}
\end{equation}

Next, for each $1\le l\le D$, take
\[
    \psi_{\mathbf j}(w)
    :=
    g_{\mathbf j,l}^\circ
    \left(
        u_{\mathbf j}-\frac{w}{N}
    \right).
\]
By \eqref{eq:g-local-cell-error},
\eqref{eq:generator-assumption}, and the standing choice of
$\varepsilon$ for which the right-hand side of
\eqref{eq:g-local-poly-error} is at most one,
\[
\begin{aligned}
\sup_{w\in[0,1]^d}
\left\|
g_{\mathbf j}^\circ
\left(
    u_{\mathbf j}-\frac{w}{N}
\right)
\right\|
&\le
\sup_{u\in\overline{\mathcal U_{\mathbf j}}}
\|g^*(u)\|
\\
&\quad+
\sup_{u\in\overline{\mathcal U_{\mathbf j}}}
\|g_{\mathbf j}^\circ(u)-g^*(u)\|
\\
&\le 2.
\end{aligned}
\]
Hence
\[
    \|\psi_{\mathbf j}\|_{L^\infty([0,1]^d)}
    \le2,
\]
and Lemma~\ref{lem:fm-local-integral} yields scalar ReLU networks
$P_{\mathbf j,l}$ satisfying
\begin{equation}
\max_{1\le l\le D}
\left|
N^d
\int_{\mathcal U_{\mathbf j}}
g_{\mathbf j,l}^\circ(u)
\exp\left\{
    -\frac{
        \|x-tg_{\mathbf j}^\circ(u)\|^2
    }{
        2q_t
    }
\right\}du
-
P_{\mathbf j,l}(x,t)
\right|
\lesssim_{d,\beta}
\frac{\varepsilon'}{N}.
\label{eq:fm-Pjl-local-error}
\end{equation}

Recall the exact numerator and denominator
$P^\circ=(P_1^\circ,\ldots,P_D^\circ)^\top$ and $Q^\circ$ from
\eqref{eq:posterior-cell-decomposition}. We restore the cell
volumes and assemble the local integral networks. Define
\begin{equation*}
Q(x,t):=\sum_{\mathbf j\in\{1,\ldots,N\}^d}N^{-d}Q_{\mathbf j}(x,t),\qquad
P_l(x,t):=\sum_{\mathbf j\in\{1,\ldots,N\}^d}N^{-d}P_{\mathbf j,l}(x,t),
\quad 1\le l\le D,
\end{equation*}
and set
\begin{equation*}
P(x,t):=(P_1(x,t),\ldots,P_D(x,t))^\top .
\end{equation*}

Since there are exactly $N^d$ cells, \eqref{eq:fm-Qj-local-error} and
\eqref{eq:fm-Pjl-local-error} imply
\begin{equation}
\|Q-Q^\circ\|_{L^\infty(\mathcal C_{[0,1]}^*)}
\lesssim_{d,\beta}\frac{\varepsilon'}{N},
\qquad
\max_{1\le l\le D}
\|P_l-P_l^\circ\|_{L^\infty(\mathcal C_{[0,1]}^*)}
\lesssim_{d,\beta}\frac{\varepsilon'}{N}.
\label{eq:fm-step5-PQ-handoff}
\end{equation}
We retain the numerator estimate in this coordinatewise form because the subsequent division is performed separately for each scalar ratio $P_l^\circ/Q^\circ$. Thus no Euclidean norm conversion, and hence no factor $\sqrt D$, is needed at this stage.

By
Lemma~\ref{lem:parallelization-network}, the scalar denominator network
$Q$ and each scalar numerator-coordinate network $P_l$ can be realized
with
\begin{equation*}
\begin{aligned}
    L_Q,\;L_{P_l}
    &\le L_\Upsilon+1, \\
    \|\mathbf W_Q\|_\infty,\;\|\mathbf W_{P_l}\|_\infty
    &\lesssim N^d\|\mathbf W_\Upsilon\|_\infty, \\
    S_Q,\;S_{P_l}
    &\lesssim N^d(S_\Upsilon+1), \\
    B_Q,\;B_{P_l}
    &\lesssim B_\Upsilon\vee1 .
\end{aligned}
\end{equation*}
The coefficients $N^{-d}$ appearing in the final summation are
implemented exactly by the output affine layer and have magnitude at
most one.

Parallelizing the $D$ numerator-coordinate networks gives the
$D$-output ReLU network $P$ satisfying
\begin{equation*}
\begin{aligned}
    L_P
    &\le L_\Upsilon+1, \\
    \|\mathbf W_P\|_\infty
    &\lesssim DN^d\|\mathbf W_\Upsilon\|_\infty, \\
    S_P
    &\lesssim DN^d(S_\Upsilon+1), \\
    B_P
    &\lesssim B_\Upsilon\vee1 .
\end{aligned}
\end{equation*}
Thus the ambient dimension enters the network size through the
parallelization of the $D$ output coordinates, whereas the
approximation tolerance for each numerator coordinate remains of order
$\varepsilon'/N$.

Finally, parallelizing $P$ and $Q$ gives a $(D+1)$-output ReLU network
$(P,Q)$ with
\begin{equation*}
\begin{aligned}
    L_{P,Q}
    &\le L_\Upsilon+1, \\
    \|\mathbf W_{P,Q}\|_\infty
    &\lesssim (D+1)N^d\|\mathbf W_\Upsilon\|_\infty, \\
    S_{P,Q}
    &\lesssim (D+1)N^d(S_\Upsilon+1), \\
    B_{P,Q}
    &\lesssim B_\Upsilon\vee1 .
\end{aligned}
\end{equation*}
Hence the local-integral construction reduces the remaining approximation
problem to obtaining
a uniform positive lower bound for $Q^\circ$ and approximating the
scalar ratios $P_l^\circ/Q^\circ$, $1\le l\le D$.

\paragraph{Denominator control and posterior assembly.}

We approximate the coordinatewise ratios $P_l^\circ/Q^\circ$,
$1\le l\le D$, by first establishing that the exact denominator is
uniformly bounded away from zero on $\mathcal C_{[0,1]}^*$ and then
applying the robust division network from
\cite[Lemma A.4]{yakovlev2025generalization}.

Fix $(x,t)\in\mathcal C_{[0,1]}^*$. By the definition of
$\mathcal K_t$, there exists $u_*\in[0,1]^d$ such that
$\|x-tg^*(u_*)\|\le R_t$. Let $\mathcal U_{\mathbf j_*}$ be the unique
cell containing $u_*$. By the H\"older regularity of $g^*$ and the local
Taylor estimate in \eqref{eq:g-local-poly-error}, for every
$u\in\mathcal U_{\mathbf j_*}$,
\[
    \|g_{\mathbf j_*}^\circ(u)-g^*(u_*)\|
    \lesssim_{d,\beta} H\sqrt D\,N^{-(\beta\wedge1)} .
\]
Consequently,
\begin{equation*}
\begin{aligned}
    \frac{\|x-tg_{\mathbf j_*}^\circ(u)\|^2}{2q_t}
    &\le \frac{\|x-tg^*(u_*)\|^2}{q_t}
    +\frac{t^2\|g_{\mathbf j_*}^\circ(u)-g^*(u_*)\|^2}{q_t} \\
    &\le \frac{R_t^2}{q_t}
    +C_{d,\beta}\frac{DH^2}{\sigma^2N^{2(\beta\wedge1)}} .
\end{aligned}
\end{equation*}
Here we used $t^2/q_t\le\sigma^{-2}$ for $t\in[0,1]$. Therefore,
using only the cell $\mathcal U_{\mathbf j_*}$ in
\eqref{eq:posterior-cell-decomposition},
\begin{equation}
    Q^\circ(x,t)
    \ge N^{-d}
    \exp\left\{
        -\frac{R_t^2}{q_t}
        -C_{d,\beta}\frac{DH^2}
        {\sigma^2N^{2(\beta\wedge1)}}
    \right\}.
    \label{eq:fm-step6-Q-lower-bound}
\end{equation}
Since the integrand defining $Q^\circ$ is bounded by one,
$Q^\circ(x,t)\le1$. Moreover, by the choice of $R_t$ above,
$R_t^2/q_t$ is independent of $t$, so the lower bound
\eqref{eq:fm-step6-Q-lower-bound} is uniform over
$\mathcal C_{[0,1]}^*$.

Choose an integer $K\ge4$ such that
\begin{equation}
    2^{-K}
    \le N^{-d}
    \exp\left\{
        -\frac{R_t^2}{q_t}
        -C_{d,\beta}\frac{DH^2}
        {\sigma^2N^{2(\beta\wedge1)}}
    \right\}.
    \label{eq:fm-step6-K-choice}
\end{equation}
Then $2^{-K}\le Q^\circ(x,t)\le1$ uniformly on
$\mathcal C_{[0,1]}^*$. In particular, one may take
\[
    K\lesssim_{d,\beta}
    1+d\log N+\frac{R_t^2}{q_t}
    +\frac{DH^2}{\sigma^2N^{2(\beta\wedge1)}} .
\]

By $\|g^\circ\|_{L^\infty}\le2$, each exact numerator coordinate
satisfies
\[
    |P_l^\circ(x,t)|\le2Q^\circ(x,t).
\]
Hence
\[
    \left|\frac12P_l^\circ(x,t)\right|
    \le Q^\circ(x,t),
\]
which is precisely the numerator range required by
Lemma~\ref{lem:robust-division}. Define its output-accuracy parameter by
\[
 \eta_{\mathrm{div}}
 :=c_1\frac{\varepsilon^\beta}
 {K^2+\log^2(1/\varepsilon)},
\]
where $c_1=c_1(d,\beta)>0$ is sufficiently small, and choose
\begin{equation}
    \varepsilon'
    :=
    c_2N2^{-2K}\eta_{\mathrm{div}},
    \label{eq:fm-step6-epsilon-prime}
\end{equation}
with $c_2=c_2(d,\beta)>0$ small enough to absorb the constants in
\eqref{eq:fm-step5-PQ-handoff}. The approximations
$(P_l/2,Q)$ then differ from $(P_l^\circ/2,Q^\circ)$ by at most
$2^{-2K}\eta_{\mathrm{div}}$. Lemma~\ref{lem:robust-division}
therefore provides a scalar ReLU network $\mathcal R$
such that, for
\[
    \widetilde m_l(x,t)
    :=
    2\,\mathcal R\left(\frac12P_l(x,t),Q(x,t)\right),
    \qquad 1\le l\le D,
\]
we have
\begin{equation*}
    \max_{1\le l\le D}
    \|\widetilde m_l-m_l^\circ\|_
    {L^\infty(\mathcal C_{[0,1]}^*)}
    \lesssim_{d,\beta} \varepsilon^\beta .
\end{equation*}

Setting
$\widetilde m=(\widetilde m_1,\ldots,\widetilde m_D)^\top$ and
performing the scalar-to-vector norm conversion gives
\begin{equation}
    \|\widetilde m-m_t^\circ\|_
    {L^\infty(\mathcal C_{[0,1]}^*)}
    \lesssim_{d,\beta} \sqrt D\,\varepsilon^\beta .
    \label{eq:fm-step6-vector-posterior-error}
\end{equation}

Since $K\ge4$,
$\log(1/\eta_{\mathrm{div}})\lesssim_{d,\beta}
K+\log(1/\varepsilon)$. Hence the complexity statement in
Lemma~\ref{lem:robust-division} gives
\begin{equation*}
\begin{aligned}
    L_R
    &\lesssim_{d,\beta} K^2+\log^2\frac1\varepsilon, \\
    \|\mathbf W_R\|_\infty
    &\lesssim_{d,\beta} K^3+K\log^2\frac1\varepsilon, \\
    S_R
    &\lesssim_{d,\beta} K^4+K\log^3\frac1\varepsilon, \\
    B_R
    &\lesssim_{d,\beta} 16^K
    \left(K^2+\log^2\frac1\varepsilon\right).
\end{aligned}
\end{equation*}
Starting from the $(D+1)$-output network $(P,Q)$ constructed in
the local-integral construction, one exact affine layer forms the $D$ pairs
$(P_l/2,Q)$, $1\le l\le D$. Parallelizing $D$ copies of
$\mathcal R$ and applying the final exact scaling by $2$ yields the
$D$-output network $\widetilde m$ with
\begin{equation}
\begin{aligned}
    L_m
    &\lesssim_{d,\beta} L_{P,Q}+L_R+2, \\
    \|\mathbf W_m\|_\infty
    &\lesssim_{d,\beta} \|\mathbf W_{P,Q}\|_\infty
    +D\|\mathbf W_R\|_\infty+D, \\
    S_m
    &\lesssim_{d,\beta} S_{P,Q}+DS_R+D, \\
    B_m
    &\lesssim_{d,\beta} B_{P,Q}\vee B_R\vee2 .
\end{aligned}
\label{eq:fm-step6-posterior-configuration}
\end{equation}
All hidden constants depend only on $d$ and $\beta$.
The value-level output is $f^{\rm val}:=\widetilde m$. To record the
complexity of this value-only circuit, use
$N^d\lesssim_d\varepsilon^{-d}$,
$\log N\lesssim1+\log(1/\varepsilon)$, and
\[
 \frac{R_t^2}{q_t}\lesssim_\beta\Gamma_\varepsilon^2,
 \qquad K\lesssim_{d,\beta}\Gamma_\varepsilon^2,
 \qquad \log\frac{N}{\varepsilon'}
 \lesssim_{d,\beta}\Gamma_\varepsilon^2.
\]
Substitution into the component configurations and parallelization of the
$D$ outputs give
$L\lesssim_{d,\beta}\Gamma_\varepsilon^4$,
$W\lesssim_{d,\beta}D^2\varepsilon^{-d}
\Gamma_\varepsilon^{2m_\beta+6}$,
$S\lesssim_{d,\beta}D^2\varepsilon^{-d}
\Gamma_\varepsilon^{4m_\beta+10}$, and
$\log B\lesssim_{d,\beta}\Gamma_\varepsilon^4$.
Together with \eqref{eq:fm-step6-vector-posterior-error}, this
establishes the value-level latent-cell construction. These are baseline
bounds for the value modules, rather than the final architecture in
Theorem~\ref{thm:velocity-approximation}. The spatial refinement below
rebuilds the same latent-cell circuit with first-order modules and derives
the final common envelope in
\eqref{eq:app-regular-posterior-architecture}.

\subsubsection{First-order spatial refinement}

We refine the cell factorization from
\eqref{eq:fm-exponent-composition}--\eqref{eq:fm-cell-change-of-variables}
by replacing its scalar modules with first-order stable realizations.
Two distinct variables play different
roles. H\"older regularity of $g^*$ controls approximation in the latent
variable $u$, producing the $\varepsilon^\beta$ error and the
$\varepsilon^{-d}$ cell count. Spatial differentiability in the ambient
state $x$ instead comes from the Gaussian kernel: differentiating the
posterior weights requires no derivative of the bounded generator.
Consequently, Gaussian smoothing supplies the covariance identity below.
The resulting spatial control is the regularity used in deterministic
FM perturbation analyses such as \citet{zhou2025error} and
\citet{kunkel2026distribution}.

The proof has four interfaces, recorded below in the order in which they
are used. The covariance identity controls the exact posterior; the
moving tube supplies both a denominator lower bound and an exact ReLU
gate; fixed-box first-order modules realize the cell integrals and their
quotient; and the final local lemma collects the error and architecture
bookkeeping. Lemma~\ref{lem:app-surrogate-posterior-network} then performs
the global assembly. This separation keeps the main approximation proof
focused on the statistical construction while making every first-order
ingredient explicit here.

For fixed $t$, the Gaussian numerator and denominator are analytic in
$x$; since $Q^\circ(x,t)>0$ on the real domain, their ratio is real
analytic there. After the latent-cell factorization, the nonlinear
dependence is carried by fixed-dimensional analytic maps of coefficient
vectors whose dimensions depend only on $(d,\beta)$. General first-order
ReLU approximation is studied, for example, by
\citet{guehring2020sobolev}. Our fixed-box module follows the constructive
polynomial architecture underlying \citet{opschoor2022holomorphic}; because
the entrywise parameter bound needed here is not quoted as a separate
conclusion, it is derived below from the same polynomial realization.
The citation therefore supplies the proof architecture, while the stated
$W^{1,\infty}$ and parameter-magnitude bounds are verified in the lemma
itself. This preserves the latent cell count.

\paragraph{Exact posterior geometry.}

\begin{lemma}[Posterior covariance identity]
\label{lem:app-surrogate-spatial-covariance}
Let $g:[0,1]^d\to\mathbb R^D$ be measurable with
$\sup_u\|g(u)\|\le R$, and let $m_{g,t}$ be the posterior mean for
$X_t=tg(U)+\sqrt{q_t}Z$. For every $x$ and $t\in[0,1]$,
\begin{equation}
 \nabla_xm_{g,t}(x)=\frac{t}{q_t}
 \operatorname{Cov}(g(U)\mid X_t=x),\qquad
 \|\nabla_xm_{g,t}(x)\|_{\rm op}\le R^2\frac{t}{q_t}.
 \label{eq:app-surrogate-covariance}
\end{equation}
In particular, $\operatorname{Lip}_x(m_t^\circ)\le4t/q_t$.
\end{lemma}
\begin{proof}
Let $w(u)=\exp\{-\|x-tg(u)\|^2/(2q_t)\}$ and normalize $w(u)du$
to a probability measure $\mu_{x,t}$. Dominated differentiation of the
posterior quotient gives
\[
 \partial_{x_k}m_{g,t,l}
 =\frac{t}{q_t}\{\mathbb E_\mu[g_lg_k]
 -\mathbb E_\mu g_l\,\mathbb E_\mu g_k\}.
\]
For any unit vector $z$, the quadratic form of this covariance is
$\operatorname{Var}_\mu(z^\top g)\le\mathbb E_\mu(z^\top g)^2\le R^2$.
The surrogate bound \eqref{eq:g-circ-bound} gives the last assertion.
\end{proof}

Put $\Lambda_\varepsilon:=R_t/\sqrt{q_t}
=\sqrt D+16(\sqrt{DL_\varepsilon}\vee L_\varepsilon)$; the right side
is independent of $t$. The original localization set $\mathcal K_t$ is
unchanged. The following larger tube is only the domain on which the
network ratio and its spatial derivative are controlled.

\begin{lemma}[Moving tube and denominator]
\label{lem:app-moving-tube}
Write $c_{\mathbf j}=g^*(u_{\mathbf j})$,
$s_\sigma(t)=(1-t)+\sigma t$, and
$d_N(x,t)=\min_{\mathbf j}\|x-tc_{\mathbf j}\|_\infty$.
Choose $E_N=C_{d,\beta}H\sqrt D\,N^{-(\beta\wedge1)}$ to dominate
the cellwise Taylor variation, and a fixed $c_0>0$. Set
\[
 r_{\rm in}=\Lambda_\varepsilon s_\sigma+E_N,\quad
 r_{\rm out}=r_{\rm in}+c_0\sigma,\quad
 \mathcal T_\varepsilon^+=\{(x,t):d_N(x,t)\le r_{\rm out}(t)\},
\]
and
$\chi(x,t)=\operatorname{clip}_{[0,1]}
((r_{\rm out}(t)-d_N(x,t))/(c_0\sigma))$.
Then $\chi=1$ on $\mathcal C_{[0,1]}^*$,
$\operatorname{Lip}_x\chi\le(c_0\sigma)^{-1}$, and throughout
$\mathcal T_\varepsilon^+$,
\begin{equation}
 Q^\circ(x,t)\ge N^{-d}\exp\{-CD(\Lambda_\varepsilon+1)^2\}
 \ge 2^{-K_\varepsilon},\qquad
 K_\varepsilon\lesssim_{d,\beta}\Gamma_\varepsilon^3.
 \label{eq:app-enlarged-tube-denominator}
\end{equation}
The detector $\chi$ has an exact ReLU realization of width and
sparsity $O(DN^d)$.
\end{lemma}
\begin{proof}
The inequalities $\sqrt{q_t}\le s_\sigma(t)\le\sqrt{2q_t}$
and $E_N\lesssim\sigma$ follow from the scale conditions. If
$x\in\mathcal K_t$, choose a witnessing $u$ and its cell. The Taylor
bound gives $d_N(x,t)\le R_t+E_N\le r_{\rm in}(t)$. The minimum of
the distances to fixed centers is one-Lipschitz in $x$, proving the
gate bound. The absolute value, maximum, minimum, and scalar clip
have exact ReLU formulas; balanced maximum and minimum trees yield the
stated detector size.

For $(x,t)\in\mathcal T_\varepsilon^+$ choose a minimizing cell. On
that entire cell,
$\|x-tg_{\mathbf j}^\circ(u)\|
\le\sqrt D\,r_{\rm out}(t)+E_N$. Moreover,
$r_{\rm out}(t)/\sqrt{q_t}\lesssim\Lambda_\varepsilon+1$.
Integrating the positive Gaussian weight only over this cell, of volume
$N^{-d}$, proves the first lower bound. The factor $D$ comes from
$\|z\|\le\sqrt D\|z\|_\infty$. Taking dyadic logarithms and using
$d\log N+D(\Lambda_\varepsilon+1)^2
\lesssim_{d,\beta}\Gamma_\varepsilon^3$
gives the second.
\end{proof}

\paragraph{Derivative-stable elementary modules.}
For $F=(F_1,\ldots,F_m):E\to\mathbb R^m$ on
$E\subset\mathbb R^D\times[0,1]$, write
\[
 \|F\|_{\mathsf W_x(E)}
 :=\max_{1\le r\le m}
 \mathop{\rm ess\,sup}_{(x,t)\in E}
 \bigl(|F_r(x,t)|+\|\nabla_xF_r(x,t)\|_2\bigr).
\]
For $m=1$ this is the scalar norm used below.  Only derivatives in $x$
occur in $\mathsf W_x$.  The following three modules make the
first-order construction explicit.

\begin{lemma}[Zero-preserving multiplication]
\label{lem:app-zero-preserving-multiplication}
For $M\ge1$ and $0<\delta<1$, there is a two-input ReLU network
$\Pi_{M,\delta}$ such that, on $[-M,M]^2$,
\begin{equation}
 \max\bigl\{|\Pi_{M,\delta}(a,b)-ab|,
 |\partial_a\Pi_{M,\delta}(a,b)-b|,
 |\partial_b\Pi_{M,\delta}(a,b)-a|\bigr\}\le\delta
 \label{eq:app-multiplier-first-order}
\end{equation}
almost everywhere.  Its depth and sparsity are
$O(1+\log(M/\delta))$, its width is bounded by a universal constant,
and $\log B\lesssim1+\log M$.  In addition,
\begin{equation}
 \Pi_{M,\delta}(0,b)=0\qquad\text{for every }b\in\mathbb R,
 \label{eq:app-multiplier-zero}
\end{equation}
and hence $\partial_b\Pi_{M,\delta}(0,b)=0$ almost everywhere.
If $A,B,\widetilde A,\widetilde B$ take values in $[-M,M]$, the ordinary
chain rule and \eqref{eq:app-multiplier-first-order} imply
\begin{align}
 \|\Pi_{M,\delta}(\widetilde A,\widetilde B)-AB\|_{\mathsf W_x(E)}
 &\le C_M\bigl(1+\|A\|_{\mathsf W_x(E)}
 +\|B\|_{\mathsf W_x(E)}\bigr) \\
 &\quad\times
 \bigl(\delta+\|\widetilde A-A\|_{\mathsf W_x(E)}
 +\|\widetilde B-B\|_{\mathsf W_x(E)}\bigr).
 \label{eq:app-multiplier-composition}
\end{align}
\end{lemma}
\begin{proof}
Let $S_m$ be the standard deep ReLU realization of the dyadic linear
interpolant of $z^2$ on $[0,1]$, extended by the same ReLU formula to
$\mathbb R$.  It satisfies
$\|S_m-z^2\|_{W^{1,\infty}(0,1)}\le C2^{-m}$ and $S_m(0)=0$.
Set
\[
 \Pi_{M,\delta}(a,b):=2M^2\left[
 S_m\!\left(\frac{|a+b|}{2M}\right)
 -S_m\!\left(\frac{|a|}{2M}\right)
 -S_m\!\left(\frac{|b|}{2M}\right)\right]
\]
and take $m\ge C(1+\log(M/\delta))$.  Polarization gives
\eqref{eq:app-multiplier-first-order}.  When $a=0$, the first and third
terms are identical for every real $b$, proving
\eqref{eq:app-multiplier-zero} without any restriction on the second
input.  The last assertion follows by differentiating the composition
almost everywhere and adding and subtracting the exact product.
\end{proof}

\begin{lemma}[Fixed-box holomorphic maps]
\label{lem:app-fixed-box-holomorphic}
Let $m$ be fixed and let $\mathcal B\subset\mathbb R^m$ be a fixed box.
Suppose that $F$ extends holomorphically to a fixed complex
neighborhood $\mathcal U$ of $\mathcal B$ and
$\sup_{z\in\mathcal U}|F(z)|\le M_0$.  For $0<\zeta<1$, there is a
ReLU network $F_\zeta$ satisfying
\begin{equation}
 \|F_\zeta-F\|_{W^{1,\infty}(\mathcal B)}\le\zeta,\qquad
 L,W,S,\log B\le
 C_{m,\mathcal U,M_0}\bigl(1+\log(1/\zeta)\bigr)^{C_m}.
 \label{eq:app-holomorphic-fixed-box}
\end{equation}
The same conclusion holds uniformly for a family sharing
$(m,\mathcal B,\mathcal U,M_0)$, and every such family has uniformly
bounded second derivatives on any smaller fixed complex neighborhood.
\end{lemma}
\begin{proof}
After an affine change of variables, truncate the tensor Chebyshev
expansion at degree $s=C\log(1/\zeta)$.  Holomorphy on $\mathcal U$
gives exponential decay of the coefficients and hence value and
first-derivative truncation error at most $\zeta/2$.  Realize the
resulting polynomial with Lemma~\ref{lem:app-zero-preserving-multiplication},
allocating accuracy
$c\zeta C^{-s}(s+1)^{-m-2}$ to every product.  The number of monomials
is polynomial in $s$, and the sum of the absolute monomial coefficients
is at most $C^s$.  Consequently all depth, width, and sparsity costs are
polynomial in $s$ and
$\log B\le C(s+\log(1/\zeta))$.  This proves
\eqref{eq:app-holomorphic-fixed-box}.  The derivative statement is
Cauchy's estimate on a smaller neighborhood.  This is the
fixed-dimensional constructive mechanism behind
\citet[Theorem~3.6]{opschoor2022holomorphic}; the displayed parameter
bound is extracted here from the polynomial construction rather than
quoted from that theorem.
\end{proof}

In particular, if $\psi$ is bounded and
\begin{equation}
 \Phi(v)=\int_{[0,1]^d}\psi(w)
 \exp\{-v^\top a(w)-b(w)\}\,dw,\qquad
 \|a\|_\infty\le A_0,\quad\|b\|_\infty\le B_0,
 \label{eq:app-integral-primitive}
\end{equation}
where $A_0,B_0$ are fixed, then $\Phi$ is entire and the preceding
lemma applies on every fixed real box.  Indeed, on a fixed complex
neighborhood the integrand is dominated by a constant depending only on
that neighborhood, $A_0,B_0$, and $\|\psi\|_\infty$.  This observation
also supplies the common complex supremum and the second-derivative
bound needed below.  Unlike the integral-specific corollary, the full
Lemma~\ref{lem:app-fixed-box-holomorphic} also applies to the fixed-box
quotient map used in the division module.

\begin{lemma}[Tapered exponential under a spatial gradient envelope]
\label{lem:app-tapered-exponential}
Let $V\ge0$ and $\widetilde V\ge-1$ on $E$, with
$\|\nabla_xV\|_2+\|\nabla_x\widetilde V\|_2\le A$, $A\ge1$, and
$\|\widetilde V-V\|_{\mathsf W_x(E)}\le\rho$.  For
$0<\kappa<1$ there is a one-input ReLU network
$\mathcal E_{A,\kappa}$ such that
\begin{equation}
 \|\mathcal E_{A,\kappa}(\widetilde V)-e^{-V}\|_{\mathsf W_x(E)}
 \le C(1+A)(\rho+\kappa).
 \label{eq:app-tapered-exponential-error}
\end{equation}
It is identically zero when $\widetilde V\ge T+1$, where
$T=C[1+\log(A/\kappa)]$, and its complexity is polynomial in
$T+\log(1/\kappa)$, with $\log B$ bounded by the same quantity.
\end{lemma}
\begin{proof}
Let $g_T(s)=\operatorname{clip}_{[0,1]}(T+1-s)$ and
$\vartheta_T(s)=g_T(s)^2(3-2g_T(s))$.  Then
$F_T(s)=e^{-s}\vartheta_T(s)$ equals $e^{-s}$ on $[0,T]$, vanishes on
$[T+1,\infty)$, and satisfies
\[
 \sup_{s\ge T}\bigl(|e^{-s}-F_T(s)|
 +|{-}e^{-s}-F_T'(s)|\bigr)\le Ce^{-T}.
\]
Approximate $e^{-s}$ in $W^{1,\infty}([-1,T+1])$ after affine
rescaling.  With a fixed input bound $M\ge4$, realize the cutoff in the
order
\[
 u_T=\Pi_{M,\kappa}(g_T,g_T),\qquad
 \widetilde\vartheta_T
 =\Pi_{M,\kappa}(u_T,3-2g_T),
\]
and use one final zero-preserving multiplier with
$\widetilde\vartheta_T$ as its first input and the exponential
approximant as its second input.
The output is therefore exactly zero for every second input whenever
$s\ge T+1$.  The one-dimensional Chebyshev argument used in
Lemma~\ref{lem:app-fixed-box-holomorphic} has cost polynomial in
$T+\log(1/\kappa)$.  Since $Ae^{-T}\le C\kappa$, the chain rule on
$\{V\ge T\}$ and the Lipschitz bounds of $F_T,F_T'$ on
$[-1,T+1]$ yield \eqref{eq:app-tapered-exponential-error}.
\end{proof}

\paragraph{Principal local realization.}
The preceding ingredients are now assembled cell by cell. The proof is
expanded into normalization, backward error allocation, stable division,
and vector/architecture bookkeeping so that the origin of every dimension
and accuracy factor remains visible.

\begin{lemma}[First-order cell and ratio realization]
\label{lem:app-first-order-local-posterior}
Let $\eta_*=c\min\{\varepsilon^\beta,D^{-1}\}$ with a sufficiently
small $c=c(d,\beta)$.  There is a $D$-output ReLU network
$f_\theta^{\rm loc}$ such that on $\mathcal T_\varepsilon^+$,
almost everywhere in $x$,
\begin{equation}
 \|f_\theta^{\rm loc}-m_t^\circ\|_2\le\sqrt D\,\eta_*,
 \qquad
 \|J_xf_\theta^{\rm loc}-J_xm_t^\circ\|_{\rm op}
 \le D\eta_*\le1.
 \label{eq:app-local-first-order-bounds}
\end{equation}
The network has the architecture envelope
\eqref{eq:app-regular-posterior-architecture}.
\end{lemma}
\begin{proof}
\smallskip\noindent\textbf{Step 1: normalization on the enlarged tube.}
Retain the cells, Taylor polynomials, scalar coordinates, and the exact
$N^{-d}$ factor in
\eqref{eq:fm-exponent-composition}--\eqref{eq:fm-cell-change-of-variables}.
For the present first-order realization only, replace every old
normalizing constant by
\[
 C_{\mathbf j,\mathbf k}^+
 :=1\vee\|V_{\mathbf j,\mathbf k}\|_{L^\infty(\mathcal T_\varepsilon^+)}.
\]
Define
\[
 R_{\mathbf j}^+
 :=\left(
 \left(\frac{V_{\mathbf j,\mathbf k}}{2C_{\mathbf j,\mathbf k}^+}
 +\frac12\right)_{1\le|\mathbf k|\le r},
 \frac{V}{\|V\|_{L^\infty([0,1])}}
 \right)^\top.
\]
Define $a_{\mathbf j}^+,b_{\mathbf j}^+$ by
\eqref{eq:fm-aj}--\eqref{eq:fm-bj}, with
$C_{\mathbf j,\mathbf k}$ replaced by
$C_{\mathbf j,\mathbf k}^+$.  Thus the cell identity itself is unchanged
and $R_{\mathbf j}^+\in[0,1]^{m_\beta^{\rm Tay}}$ on the enlarged tube.
The Taylor estimate, the scale condition, and
$r_{\rm out}/\sqrt{q_t}\lesssim\Lambda_\varepsilon+1$ give
\begin{equation}
 \sup_{w\in[0,1]^d}\bigl(
 \|a_{\mathbf j}^+(w)\|_\infty+|b_{\mathbf j}^+(w)|\bigr)
 \le C_{d,\beta},
 \qquad
 C_{\mathbf j,\mathbf k}^+
 \le[1+D+H+\sigma^{-1}+\Lambda_\varepsilon]^{C_{d,\beta}}.
 \label{eq:app-enlarged-normalization}
\end{equation}
The nonlinear input dimension remains at most
$m_\beta^{\rm Tay}+1\le m_\beta+1$, independently of $D$.  Fix
$\rho_0:=1/4$.  We approximate $R_{\mathbf j}^+$ on the fixed box
\begin{equation}
 \mathcal B_+:=[-\rho_0,1+\rho_0]^{m_\beta^{\rm Tay}}
 \label{eq:app-expanded-coordinate-box}
\end{equation}
and do not clip its approximation; this avoids the derivative
instability of hard clipping at an interior point of the exact range.

\smallskip\noindent\textbf{Step 2: backward first-order budget.}
Let $\lambda=2^{-K_\varepsilon}$ and put
\[
 M_1:=C(1+\sigma^{-1}),\qquad
 \delta_{PQ}:=c\eta_*M_1^{-1}\lambda^2.
\]
All exact numerator and denominator cell sums and their spatial
gradients are bounded by $M_1$.  Indeed, $0<Q^\circ\le1$ and
$|P_l^\circ|\le2Q^\circ$, while
\[
 \|\nabla_xQ^\circ\|_2
 \le\frac1{q_t}\int
 \|x-tg^\circ(u)\|e^{-\|x-tg^\circ(u)\|^2/(2q_t)}du
 \le Cq_t^{-1/2}\le C\sigma^{-1},
\]
and the same calculation with $|g_l^\circ|\le2$ controls
$\|\nabla_xP_l^\circ\|_2$.  Put
\[
 M_\varepsilon:=1+
 \sup_{(x,t)\in\mathcal T_\varepsilon^+}\|x\|_\infty.
\]
The moving-tube definition gives
$M_\varepsilon\le C(1+\Lambda_\varepsilon+H+\sigma^{-1})$.
Choose a deterministic envelope
\begin{equation}
 \mathfrak A_\varepsilon
 :=[1+D+H+\sigma^{-1}+\Lambda_\varepsilon+K_\varepsilon
 +M_\varepsilon+\log(1/\eta_*)]^{C_{d,\beta}}
 \label{eq:app-chain-envelope}
\end{equation}
large enough to dominate the value ranges, spatial gradients, and
second-derivative constants of all exact cell coordinates and all
intermediate approximations below.  Set
\begin{equation}
 \rho:=\min\{\rho_0/4,
 c\delta_{PQ}\mathfrak A_\varepsilon^{-4}\},\qquad
 \kappa:=c\delta_{PQ}\mathfrak A_\varepsilon^{-4}.
 \label{eq:app-backward-budget}
\end{equation}
Then
\begin{equation}
 K_\varepsilon+\log(1/\eta_*)+\log(1/\rho)+\log(1/\kappa)
 \le\Gamma_\varepsilon^{C_{d,\beta}}.
 \label{eq:app-log-budget}
\end{equation}
Indeed, \eqref{eq:app-enlarged-tube-denominator} bounds
$K_\varepsilon$, while $\log\mathfrak A_\varepsilon$ and
$\log M_1$ are logarithmic in the displayed polynomial parameters.

Affine maps in $x$ are exact.  Each remaining coordinate is a finite
composition of such affine maps, the scalar coefficients
$q_t^{-1},tq_t^{-1},t^2q_t^{-1}$, squares, and products.  Because $t$
is frozen in $\mathsf W_x$, the coefficient networks require value
accuracy only.  The quadratic coordinate and the inner products in
$V_{\mathbf j,0}$ and $V_{\mathbf j,\mathbf k}$ contain at most
$C_{d,\beta}D$ arithmetic nodes.  Allocate to every coefficient,
square, and product module the accuracy
\begin{equation}
 \rho_{\rm ar}:=
 \frac{c_{d,\beta}\rho}
 {D(1+H+\sigma^{-2})(1+M_\varepsilon)^2
 \mathfrak A_\varepsilon^2}.
 \label{eq:app-coordinate-internal-accuracy}
\end{equation}
The reciprocal construction in
Lemma~\ref{lem:fm-time-coefficients}, followed by
Lemma~\ref{lem:app-zero-preserving-multiplication} and the product rule,
then gives $\widetilde V_{\mathbf j,0}$ and
$\widetilde R_{\mathbf j}^+$ such that
\begin{equation}
 \|\widetilde V_{\mathbf j,0}-V_{\mathbf j,0}\|_{\mathsf W_x}
 +
 \|\widetilde R_{\mathbf j}^+-R_{\mathbf j}^+\|_{\mathsf W_x}
 \le\rho,\qquad
 \|\nabla_x\widetilde V_{\mathbf j,0}\|_2
 +\max_r\|\nabla_x\widetilde R_{\mathbf j,r}^+\|_2
 \le\mathfrak A_\varepsilon.
 \label{eq:app-coordinate-error}
\end{equation}
Because $R_{\mathbf j}^+\in[0,1]^{m_\beta^{\rm Tay}}$ and
$\rho\le\rho_0/4$, the approximate coordinate vector remains in
$\mathcal B_+$, so the holomorphic network below is evaluated only on
the box on which its first-order estimate is valid.
Indeed, each value error is amplified by at most
$C(1+H+\sigma^{-2})(1+M_\varepsilon)$, and each spatial derivative
error by at most the same quantity times
$\mathfrak A_\varepsilon^2$.  Summing over at most
$C_{d,\beta}D$ nodes and using
\eqref{eq:app-coordinate-internal-accuracy} leaves the total below
$\rho$.  Moreover,
$\log(1/\rho_{\rm ar})\le\Gamma_\varepsilon^{C_{d,\beta}}$, so every
cost here remains polynomial in the logarithmic budget
\eqref{eq:app-log-budget}.

For a cell primitive $\Phi_{\mathbf j,\psi}$, the bounds
\eqref{eq:app-enlarged-normalization} put it in the family
\eqref{eq:app-integral-primitive} with fixed $A_0,B_0$.
Lemma~\ref{lem:app-fixed-box-holomorphic} produces
$\widetilde\Phi_{\mathbf j,\psi}$ with first-order error $\rho$ on the
expanded box.  Its uniform second-derivative bound and
\eqref{eq:app-coordinate-error} give
\begin{equation}
 \|\widetilde\Phi_{\mathbf j,\psi}
   (\widetilde R_{\mathbf j}^+)
  -\Phi_{\mathbf j,\psi}(R_{\mathbf j}^+)\|_{\mathsf W_x}
 \le C\mathfrak A_\varepsilon^2\rho.
 \label{eq:app-primitive-composition-error}
\end{equation}
For the outer Gaussian exponent, apply
Lemma~\ref{lem:app-tapered-exponential} to
$(V_{\mathbf j,0},\widetilde V_{\mathbf j,0})$ with gradient envelope
$\mathfrak A_\varepsilon$.  Its truncation level is
\[
 T=C[1+\log(\mathfrak A_\varepsilon/\kappa)],
\]
and both its value and spatial derivative contribute at most
$C\mathfrak A_\varepsilon(\rho+\kappa)$.  Notice that this estimate
uses the general factor
$\mathfrak A_\varepsilon e^{-T}$; it does not rely on the special
identity for an exact quadratic exponent after the exponent has been
approximated.

Finally combine the tapered exponential, the composed primitive, and
the bounded polynomial prefactors by repeated use of
Lemma~\ref{lem:app-zero-preserving-multiplication}.  Allocating error
$\kappa$ to each of the finitely many products yields networks for the
unit-cube-normalized exact cells
\[
 \bar Q_{\mathbf j}^\circ:=N^dQ_{\mathbf j}^\circ,
 \qquad
 \bar P_{\mathbf j,l}^\circ:=N^dP_{\mathbf j,l}^\circ.
\]
Denote these networks by
$\widetilde{\bar Q}_{\mathbf j}$ and
$\widetilde{\bar P}_{\mathbf j,l}$.  For every cell and every numerator
coordinate,
\begin{equation}
 \|\widetilde{\bar Q}_{\mathbf j}
 -\bar Q_{\mathbf j}^\circ\|_{\mathsf W_x}
 +\max_{1\le l\le D}
 \|\widetilde{\bar P}_{\mathbf j,l}
 -\bar P_{\mathbf j,l}^\circ\|_{\mathsf W_x}
 \le C_{d,\beta}\mathfrak A_\varepsilon^3(\rho+\kappa)
 \le\delta_{PQ}.
 \label{eq:app-cell-first-order-error}
\end{equation}
Restoring the exact cell volume gives
\[
 \widetilde Q=N^{-d}\sum_{\mathbf j}
 \widetilde{\bar Q}_{\mathbf j},
 \qquad
 \widetilde P_l=N^{-d}\sum_{\mathbf j}
 \widetilde{\bar P}_{\mathbf j,l}.
\]
Since
$Q^\circ=N^{-d}\sum_{\mathbf j}\bar Q_{\mathbf j}^\circ$ and
$P_l^\circ=N^{-d}\sum_{\mathbf j}\bar P_{\mathbf j,l}^\circ$,
the triangle inequality and \eqref{eq:app-cell-first-order-error} imply
\begin{equation}
 \|\widetilde Q-Q^\circ\|_{\mathsf W_x}
 +\max_l\|\widetilde P_l-P_l^\circ\|_{\mathsf W_x}
 \le C_{d,\beta}\delta_{PQ};
 \label{eq:app-summed-first-order-error}
\end{equation}
there is no factor $N^d$ in the error.

\smallskip\noindent\textbf{Step 3: robust dyadic division.}
The exact inequalities $Q^\circ\in[\lambda,1]$ and
$|P_l^\circ|\le2Q^\circ$, together with
\eqref{eq:app-summed-first-order-error} and
$\delta_{PQ}\le c\lambda$, place $(\widetilde P_l,\widetilde Q)$ in
the robust band $\widetilde Q\in[\lambda/2,3/2]$ and
$|\widetilde P_l|\le3\widetilde Q$.  In the assembled division network
below, $p=\widetilde P_l$ and $q=\widetilde Q$; the symbols
$P_l^\circ,Q^\circ$ are reserved for the exact perturbation comparison.
There are exact piecewise-linear
hats $h_j$, $0\le j\le K_\varepsilon$, forming a partition of unity on
this band, with at most two active hats,
$\operatorname{Lip}(h_j)\le C2^j$, and, on the support of $h_j$,
$z=2^jq\in[1/2,2]$.  Put $w=2^jp/8$; then $w\in[-1,1]$ on a fixed
slightly enlarged shoulder.  Apply
Lemma~\ref{lem:app-fixed-box-holomorphic} to
\[
 F(w,z)=8w/z,\qquad (w,z)\in[-1,1]\times[1/2,2],
\]
and choose local first-order accuracy
\begin{equation}
 \delta_j=c\eta_*M_1^{-1}2^{-j}.
 \label{eq:app-local-division-budget}
\end{equation}
Allocate the same $\delta_j$ accuracy to the final multiplication by the
hat.  After composition with $p(x),q(x)$, the normalization error is
bounded by
$M_1 2^j\delta_j\le c\eta_*$, and the dangerous hat term is
\[
 M_1|h_j'(q)|\,|F_j-F|
 \le CM_1 2^j\delta_j\le C\eta_*.
\]
Inactive blocks have exactly zero value and zero derivative in their
uncontrolled second input by
\eqref{eq:app-multiplier-zero}; since at most two hats are active, the
assembled division error contains no factor $K_\varepsilon$.  This
constructs a genuine ReLU network $\operatorname{Div}_{K_\varepsilon,
\eta_*}$ with
\begin{equation}
 \left\|\operatorname{Div}_{K_\varepsilon,\eta_*}(p,q)
 -\frac pq\right\|_{W^{1,\infty}}
 \le C\eta_*,\qquad
 \log B_{\rm div}\lesssim
 K_\varepsilon+\log(1/\eta_*)+\log M_1.
 \label{eq:app-dyadic-division}
\end{equation}

Independently, direct differentiation of the exact quotient gives
\begin{equation}
 \left\|\frac{\widetilde P_l}{\widetilde Q}
 -\frac{P_l^\circ}{Q^\circ}\right\|_{\mathsf W_x}
 \le C\left(\frac{\delta_{PQ}}\lambda
 +\frac{M_1\delta_{PQ}}{\lambda^2}\right)
 \le C\eta_*,
 \label{eq:app-quotient-perturbation}
\end{equation}
which explains the necessary backward choice
$\delta_{PQ}\asymp\eta_*M_1^{-1}2^{-2K_\varepsilon}$, equivalently
$\delta_{PQ}\lesssim\eta_*\sigma2^{-2K_\varepsilon}$.

\smallskip\noindent\textbf{Step 4: vector and architecture bookkeeping.}
Applying the division network to every coordinate gives scalar value
and spatial-gradient error at most $C\eta_*$.  After decreasing the
constant in the definition of $\eta_*$,
\[
 \|f_\theta^{\rm loc}-m_t^\circ\|_2\le\sqrt D\,\eta_*,
 \qquad
 \|J_xf_\theta^{\rm loc}-J_xm_t^\circ\|_{\rm op}
 \le\|J_xf_\theta^{\rm loc}-J_xm_t^\circ\|_{F}
 \le D\eta_*\le1.
\]
All serial modules have depth, sparsity, and parameter logarithm bounded
by a fixed power of the budget in \eqref{eq:app-log-budget}.  The only
algebraic resolution cost is the $N^d$ parallel cell count.  Parallelize
the denominator and the $D$ numerator coordinates and add the
$O(DN^d)$ detector.  Since $N^d\lesssim_d\varepsilon^{-d}$, enlarge
once, if necessary, the fixed exponent
$c_{\rm arch}=c_{\rm arch}(d,\beta)$ from
Theorem~\ref{thm:velocity-approximation}.  Then
\begin{equation}
 L,\log B\lesssim_{d,\beta}
 \Gamma_\varepsilon^{c_{\rm arch}},\qquad
 W,S\lesssim_{d,\beta}D^2\varepsilon^{-d}
 \Gamma_\varepsilon^{c_{\rm arch}}.
 \label{eq:app-first-order-architecture-conclusion}
\end{equation}
Thus the refinement changes only powers of
$\Gamma_\varepsilon$---in particular, only polylogarithmic factors in
$1/\varepsilon$ when $(D,H,\sigma)$ are fixed---and introduces no new
algebraic factor $\varepsilon^{-c}$.
\end{proof}

\begin{proof}[Proof of Lemma~\ref{lem:app-surrogate-posterior-network}]
\medskip\noindent\textbf{Step 1: Value realization.}
The latent-cell construction in
\eqref{eq:fm-exponent-composition}--\eqref{eq:fm-step6-posterior-configuration}
identifies a network circuit approximating $m_t^\circ$ on
$\mathcal C_{[0,1]}^*$ with coordinatewise error of order
$\varepsilon^\beta$. The exact $N^{-d}$ cell weights prevent the sum over
the $N^d$ cells from increasing this error. The $N^d$ cells determine the
algebraic dependence on the approximation scale.
The next step realizes this circuit with first-order scalar modules.

\medskip\noindent\textbf{Step 2: Spatial refinement.}
Lemma~\ref{lem:app-first-order-local-posterior} realizes the same posterior
ratio on $\mathcal T_\varepsilon^+$ with Euclidean value error
$\sqrt D\eta_*$ and Jacobian error at most one. The covariance identity in
Lemma~\ref{lem:app-surrogate-spatial-covariance} then gives the absolute
Jacobian profile $4t/q_t+1$ on this tube. The architecture calculation
gives, for the finite exponent fixed in
Theorem~\ref{thm:velocity-approximation},
$L,\log B\lesssim_{d,\beta}\Gamma_\varepsilon^{c_{\rm arch}}$ and
$W,S\lesssim_{d,\beta}D^2\varepsilon^{-d}
\Gamma_\varepsilon^{c_{\rm arch}}$.

\medskip\noindent\textbf{Step 3: Global extension and assembly.}
Let $f_\theta^{\rm loc}$ be the network of
Lemma~\ref{lem:app-first-order-local-posterior} and clip its coordinates
to $[-2,2]$, obtaining $y$. Since $\|m_t^\circ\|\le2$, coordinatewise
clipping cannot increase the Euclidean value error. On the enlarged
tube, $\|y\|\le2+\sqrt D\eta_*\le3$ and
$\|J_xy\|_{\rm op}\le4t/q_t+1$. Fix $M_{\rm gate}:=4$ and define
\[
 f_{\theta,l}(x,t)
 :=\Pi_{M_{\rm gate},\delta}(\chi(x,t),y_l(x,t)),
 \qquad
 \delta=c\min\{\varepsilon^\beta,D^{-1/2}\}.
\]

On $\mathcal K_t$, $\chi=1$, so the local error and multiplier error
give \eqref{eq:app-main-posterior-error}. Outside
$\mathcal T_\varepsilon^+$, the output and its weak spatial derivative
vanish away from the boundary. On the transition region, the chain rule
gives
\[
 J_xf_\theta=A\nabla_x\chi^\top
 +\operatorname{diag}(B)J_xy,
\]
where
$A_l=\partial_1\Pi_{M_{\rm gate},\delta}(\chi,y_l)$ and
$B_l=\partial_2\Pi_{M_{\rm gate},\delta}(\chi,y_l)$. Hence
\[
 \|A\|_2\le\|y\|_2+\sqrt D\delta\le4,
 \quad\|\operatorname{diag}(B)\|_{\rm op}\le1+\delta.
\]
Lemma~\ref{lem:app-moving-tube} therefore yields
$\|J_xf_\theta\|_{\rm op}
\lesssim_{d,\beta}1+\sigma^{-1}+t/q_t$.
For each $t$, the ReLU realization is continuous and piecewise affine in
$x$, hence absolutely continuous on line segments. Integrating the
almost-everywhere Jacobian bound along such segments gives the global
Lipschitz estimate. The detector costs $O(DN^d)$ and the
zero-preserving multiplier contributes only logarithmic complexity, so
the architecture envelope \eqref{eq:app-regular-posterior-architecture}
is preserved. The $\mathsf W_x$ approximation to $m_t^\circ$ is used
only on $\mathcal T_\varepsilon^+$. Outside this tube, exact zero
preservation gives a global absolute Lipschitz bound; no global
approximation of $J_xm_t^\circ$ on $\mathbb R^D$ is claimed. This proves
all three conclusions.
\end{proof}

\subsection{Auxiliary neural-network lemmas}
\label{app:nn-approximation-tools}

\begin{lemma}[Composition of ReLU networks]
    \label{lem:nn-composition}
    Let
    \[
    \phi_1\in\mathrm{NN}
    (L_1,\mathbf W^1,S_1,B_1),
    \qquad
    \phi_2\in\mathrm{NN}
    (L_2,\mathbf W^2,S_2,B_2),
    \]
    with compatible input and output dimensions. By
    \citet[Lemma F.1]{oko2023diffusion}, $\phi_2\circ\phi_1$ can be
    realized by a ReLU network satisfying
    \[
    L_{\mathrm{comp}}
    \lesssim
    L_1+L_2,
    \]
    \[
    \|\mathbf W_{\mathrm{comp}}\|_\infty
    \lesssim
    \|\mathbf W^1\|_\infty
    +
    \|\mathbf W^2\|_\infty,
    \]
    \[
    S_{\mathrm{comp}}
    \lesssim
    S_1+S_2,
    \qquad
    B_{\mathrm{comp}}
    \lesssim
    B_1\vee B_2\vee1.
    \]
    \end{lemma}
    
    \begin{lemma}[Identity padding]
    \label{lem:nn-identity-padding}
    For every $m,L\in\mathbb N$, \citet[Lemma F.2]{oko2023diffusion}
    gives an exact ReLU realization of
    $\operatorname{id}_{\mathbb R^m}$ of depth $L$ with
    \[
    \|\mathbf W\|_\infty
    \lesssim m,
    \qquad
    S\lesssim mL,
    \qquad
    B=1.
    \]
    Consequently, networks of unequal depths may be padded to a common
    depth before parallelization, with only the above additional
    sparsity cost.
    \end{lemma}

This section collects the neural-network tools used in the posterior
construction. In particular, we require ReLU approximations of
\[
    q_t^{-1},\qquad tq_t^{-1},\qquad t^2q_t^{-1},
    \qquad t\in[0,1],
\]
Under the standing assumption $0<\sigma<1$, $q_t$ is uniformly bounded away from zero. We therefore combine standard multiplication and reciprocal approximation networks to construct the required time-dependent coefficients.

\begin{lemma}[Multiplication network]
    \label{lem:multiplication-network}
    The multiplication primitive of
    \citet[Lemma F.6]{oko2023diffusion} states that, for
    $p \ge 2$, $C \ge 1$, $\varepsilon' \in (0, 1]$, and
    $\varepsilon > 0$, there exists a ReLU network
    $\varphi(x_1,\ldots,x_p)$ with
    \[
        L \lesssim \log p(\log \varepsilon^{-1} + p\log C),\quad
        \|\mathbf W\|_\infty = 48p,\quad
        S \lesssim p\log \varepsilon^{-1} + p\log C,\quad
        B = C^p
    \]
    such that
    \[
        \left|\varphi(x_1',\ldots,x_p')-\prod_{k=1}^p x_k\right|
        \le \varepsilon + pC^{p-1}\varepsilon',
    \]
    for all $x \in [-C, C]^p$ and $x' \in \mathbb R^p$ with $\|x-x'\|_\infty \le \varepsilon'$. Moreover, $\varphi(x_1',\ldots,x_p') = 0$ if at least one $x_i' = 0$ and $|\varphi(x_1',\ldots,x_p')| \le C^p$. The same construction approximates $\prod_{i=1}^I x_i^{\alpha_i}$ when $\alpha_i\in\mathbb Z_+$ and $\sum_{i=1}^I\alpha_i=p$.
    \end{lemma}

\begin{lemma}[Reciprocal network]
    \label{lem:reciprocal-network}
    The reciprocal primitive of \citet[Lemma F.7]{oko2023diffusion}
    gives, for any $0<\varepsilon<1$, a ReLU network
    $\phi_{\mathrm{rec}}
    \in
    \mathrm{NN}(L,\mathbf W,S,B)$ with
    \[
    L\lesssim \log^2(\varepsilon^{-1}),\qquad
    \|\mathbf W\|_\infty\lesssim \log^3(\varepsilon^{-1}),\qquad
    S\lesssim \log^4(\varepsilon^{-1}),\qquad
    B\lesssim \varepsilon^{-2},
    \]
    such that, for all $x\in[\varepsilon,\varepsilon^{-1}]$ and $x'\in\mathbb R$,
    \[
    \left|\phi_{\mathrm{rec}}(x')-\frac{1}{x}\right|
    \le
    \varepsilon+\frac{|x'-x|}{\varepsilon^2}.
    \]
    \end{lemma}

\begin{lemma}[Robust division]
    \label{lem:robust-division}
    The following is \citet[Lemma A.4]{yakovlev2025generalization}
    in the notation used below. Let $K\ge4$ be an integer. For every
    $\eta\in(0,1]$, there is a
scalar-output ReLU network $\mathcal R_\eta$ such that
\begin{equation}
 \left|\mathcal R_\eta(x',y')-\frac{x}{y}\right|
 \le 2049\bigl(4K^2\log^2 2+\log^2(1/\eta)\bigr)\eta
 \label{eq:robust-division-error}
\end{equation}
whenever $y\in[2^{-K},1]$, $|x|\le y$, and
$|x-x'|\vee|y-y'|\le2^{-2K}\eta$. Moreover,
\begin{equation}
 \begin{aligned}
 L_R&\lesssim K^2+\log^2(1/\eta),\\
 \|\mathbf W_R\|_\infty&\lesssim K^3+K\log^2(1/\eta),\\
 S_R&\lesssim K^4+K\log^3(1/\eta),\\
 B_R&\lesssim16^K\bigl(K^2+\log^2(1/\eta)\bigr).
 \end{aligned}
 \label{eq:robust-division-complexity}
\end{equation}
The hidden constants are universal.
\end{lemma}

    \begin{lemma}[Parallelization and summation of neural networks]
        \label{lem:parallelization-network}
        Let $\phi^1,\ldots,\phi^K$ be ReLU networks of the same depth $L$
        with
        $\phi^i\in\mathrm{NN}(L,\mathbf W^i,S^i,B^i)$ for $i=1,\ldots,K$.
        By \citet[Lemma F.3]{oko2023diffusion}, their parallelization
        \[
        \phi_{\mathrm{par}}
        :=
        \bigl((\phi^1)^\top,\ldots,(\phi^K)^\top\bigr)^\top
        \]
        can be realized by a ReLU network satisfying
        \[
        L_{\mathrm{par}}=L,\qquad
        \|\mathbf W_{\mathrm{par}}\|_\infty
        \lesssim
        \sum_{i=1}^K\|\mathbf W^i\|_\infty,
        \]
        \[
        S_{\mathrm{par}}
        \lesssim
        \sum_{i=1}^K S^i,\qquad
        B_{\mathrm{par}}
        \lesssim
        \max_{1\le i\le K}(B^i\vee1).
        \]
        If all outputs are scalar, the sum $\sum_{i=1}^K\phi^i$ can be realized
        by appending one affine layer. In particular,
        \[
        L_{\mathrm{sum}}\le L+1,\qquad
        \|\mathbf W_{\mathrm{sum}}\|_\infty
        \lesssim
        \sum_{i=1}^K\|\mathbf W^i\|_\infty,
        \]
        \[
        S_{\mathrm{sum}}
        \lesssim
        \sum_{i=1}^K S^i+K,\qquad
        B_{\mathrm{sum}}
        \lesssim
        \max_{1\le i\le K}(B^i\vee1).
        \]
        \end{lemma}
        
\begin{lemma}[ReLU approximation of the exponential function]
        \label{lem:relu-exp}
        By \citet[Lemma E.2]{yakovlev2025generalization}, for any
        $\varepsilon_0\in(0,1)$, there exists a ReLU network
        $\phi_{\exp}\in\mathrm{NN}(L,\mathbf W,S,B)$ such that
        \[
        \sup_{x,x'\ge0}
        \left|e^{-x'}-\phi_{\exp}(x)\right|
        \le
        \varepsilon_0+|x-x'|.
        \]
        Moreover,
        \[
        L\lesssim\log^2\!\left(\frac1{\varepsilon_0}\right),\qquad
        \|\mathbf W\|_\infty
        \lesssim
        \log\!\left(\frac1{\varepsilon_0}\right),
        \]
        \[
        S\lesssim\log^2\!\left(\frac1{\varepsilon_0}\right),\qquad
        \log B
        \lesssim
        \log^2\!\left(\frac1{\varepsilon_0}\right).
        \]
        In addition, for every
        $x\ge\log(3/\varepsilon_0)$, it holds that
        \[
        |\phi_{\exp}(x)|\le\varepsilon_0.
        \]
        \end{lemma}
        
        \begin{corollary}[Exponential approximation under perturbed inputs]
        \label{cor:relu-exp-shift}
        By \citet[Corollary E.3]{yakovlev2025generalization}, for any
        $\varepsilon_0\in(0,1)$ and $a\ge0$, there exists a ReLU network
        $\phi\in\mathrm{NN}(L,\mathbf W,S,B)$ satisfying
        \[
        L\lesssim\log^2\!\left(\frac1{\varepsilon_0}\right),\qquad
        \|\mathbf W\|_\infty
        \lesssim
        \log\!\left(\frac1{\varepsilon_0}\right),
        \]
        \[
        S\lesssim\log^2\!\left(\frac1{\varepsilon_0}\right),\qquad
        \log B
        \lesssim
        \log^2\!\left(\frac1{\varepsilon_0}\right)+(a\vee1),
        \]
        such that
        \[
        \sup_{\substack{x\ge0\\x'\ge-a}}
        \left|\phi(x')-e^{-x}\right|
        \le
        e^a\bigl(\varepsilon_0+|x-x'|\bigr).
        \]
        \end{corollary}
        


The following three lemmas provide the ReLU approximations of the
time-dependent and space-time terms required in the posterior construction.
The first
approximates the coefficients $t^\gamma/q_t$, while the next two
combine these approximations with multiplication networks to treat
the quadratic and linear interaction terms.

\begin{lemma}[Approximation of the Flow Matching time coefficients]
\label{lem:fm-time-coefficients}
Under the standing assumption $0<\sigma<1$, for any
$\gamma\in\{0,1,2\}$ and $\delta\in(0,1]$, there exists a ReLU network
$\chi_{\gamma,\delta}\in\mathrm{NN}(L,\mathbf W,S,B)$ such that
\[
\sup_{t\in[0,1]}
\left|
\chi_{\gamma,\delta}(t)-\frac{t^\gamma}{q_t}
\right|
\le\delta.
\]
Moreover,
\[
\begin{aligned}
L
&\lesssim
\log^2(1/\delta)
+\log^2\!\left(\frac{1+\sigma^2}{\sigma^2}\right),\\
\|\mathbf W\|_\infty
&\lesssim
\log^3(1/\delta)
+\log^3\!\left(\frac{1+\sigma^2}{\sigma^2}\right),\\
S
&\lesssim
\log^4(1/\delta)
+\log^4\!\left(\frac{1+\sigma^2}{\sigma^2}\right),\\
\log B
&\lesssim
\log(1/\delta)
+\log\!\left(\frac{1+\sigma^2}{\sigma^2}\right).
\end{aligned}
\]
\end{lemma}

\begin{lemma}[Approximation of the quadratic term]
\label{lem:fm-quadratic-term}
Let $x\in\mathbb R^D$ and $M\ge1$. For any
$\varepsilon\in(0,1]$, there exists a ReLU network
$\rho_\varepsilon(x,t)\in\mathrm{NN}(L,\mathbf W,S,B)$ such that
\[
\sup_{\substack{\|x\|_\infty\le M\\ t\in[0,1]}}
\left|
\rho_\varepsilon(x,t)-\frac{\|x\|^2}{2q_t}
\right|
\le\varepsilon.
\]
Moreover,
\[
\begin{aligned}
L
&\lesssim
\log^2(1/\varepsilon)
+\log^2(MD)
+\log^2\!\left(\frac{1+\sigma^2}{\sigma^2}\right),\\
\|\mathbf W\|_\infty
&\lesssim
D\left[
\log^3(1/\varepsilon)
+\log^3(MD)
+\log^3\!\left(\frac{1+\sigma^2}{\sigma^2}\right)
\right],\\
S
&\lesssim
D\left[
\log^4(1/\varepsilon)
+\log^4(MD)
+\log^4\!\left(\frac{1+\sigma^2}{\sigma^2}\right)
\right],\\
\log B
&\lesssim
\log(1/\varepsilon)
+\log(MD)
+\log\!\left(\frac{1+\sigma^2}{\sigma^2}\right).
\end{aligned}
\]
\end{lemma}

\begin{lemma}[Approximation of the linear interaction term]
\label{lem:fm-linear-term}
Let $x,a\in\mathbb R^D$ and $M\ge1$. For any
$\varepsilon\in(0,1]$, there exists a ReLU network
$\omega_\varepsilon(x,t)\in\mathrm{NN}(L,\mathbf W,S,B)$ such that
\[
\sup_{\substack{\|x\|_\infty\le M\\ t\in[0,1]}}
\left|
\omega_\varepsilon(x,t)-\frac{t\,x^\top a}{q_t}
\right|
\le\varepsilon.
\]
Moreover,
\[
\begin{aligned}
L
&\lesssim
\log^2(1/\varepsilon)
+\log^2(DM\|a\|_\infty\vee1)
+\log^2\!\left(\frac{1+\sigma^2}{\sigma^2}\right),\\
\|\mathbf W\|_\infty
&\lesssim
D+\log^3(1/\varepsilon)
+\log^3(DM\|a\|_\infty\vee1)
+\log^3\!\left(\frac{1+\sigma^2}{\sigma^2}\right),\\
S
&\lesssim
D+\log^4(1/\varepsilon)
+\log^4(DM\|a\|_\infty\vee1)
+\log^4\!\left(\frac{1+\sigma^2}{\sigma^2}\right),\\
\log B
&\lesssim
\log(1/\varepsilon)
+\log(DM\|a\|_\infty\vee1)
+\log\!\left(\frac{1+\sigma^2}{\sigma^2}\right).
\end{aligned}
\]
\end{lemma}

\subsubsection{Constructions for the Flow Matching modules}
\label{app:proofs-fm-nn-lemmas}

\begin{proof}[Proof of Lemma~\ref{lem:fm-time-coefficients}]
We first approximate $q_t^{-1}$. The lower bound is stated in
Section~\ref{sec:preliminaries}; convexity of $q_t$ and
$q_0=1$, $q_1=\sigma^2<1$ give the upper bound. Thus
\begin{equation}
\label{eq:q-uniform-bound}
\frac{\sigma^2}{1+\sigma^2}
\le q_t\le1,
\qquad t\in[0,1].
\end{equation}

Fix
\[
\eta:=\frac{\delta\sigma^2}{32(1+\sigma^2)}.
\]
Since $\delta\le1$, we have
$0<\eta\le\sigma^2/(1+\sigma^2)$, and therefore
$q_t\in[\eta,\eta^{-1}]$ for all $t\in[0,1]$.

We first construct a neural approximation of $q_t$. Apply
Lemma~\ref{lem:multiplication-network} with
$d=2$, $C=1$, $\varepsilon=\eta^3/4$, and
$\varepsilon'=\eta^3/8$. Taking $x=x'=(t,t)$ gives a
multiplication network $\varphi_{\mathrm{sq}}$ satisfying
\[
\left|
\varphi_{\mathrm{sq}}(t,t)-t^2
\right|
\le
\frac{\eta^3}{4}
+2\frac{\eta^3}{8}
=
\frac{\eta^3}{2}.
\]
Let $\vartheta_\eta(t):=\varphi_{\mathrm{sq}}(t,t)$ and define
\[
\widetilde q_t
:=
1-2t+(1+\sigma^2)\vartheta_\eta(t).
\]
Since $1+\sigma^2\le2$,
\begin{equation}
\label{eq:q-approximation-error}
|\widetilde q_t-q_t|
\le
(1+\sigma^2)|\vartheta_\eta(t)-t^2|
\le\eta^3.
\end{equation}

Let $\phi_{\mathrm{rec},\eta}$ be the reciprocal network from
Lemma~\ref{lem:reciprocal-network}. Since
$q_t\in[\eta,\eta^{-1}]$, its sensitivity estimate together with
\eqref{eq:q-approximation-error} gives
\[
\left|
\phi_{\mathrm{rec},\eta}(\widetilde q_t)-\frac1{q_t}
\right|
\le
\eta+\frac{|\widetilde q_t-q_t|}{\eta^2}
\le2\eta
=
\frac{\delta\sigma^2}{16(1+\sigma^2)}.
\]
Hence, setting
\[
\chi_{0,\delta}(t)
:=
\phi_{\mathrm{rec},\eta}(\widetilde q_t),
\]
we have
\begin{equation}
\label{eq:fm-reciprocal-error}
\sup_{t\in[0,1]}
\left|
\chi_{0,\delta}(t)-\frac1{q_t}
\right|
\le2\eta\le\delta.
\end{equation}
Moreover,
\[
|\chi_{0,\delta}(t)|
\le
\frac1{q_t}+2\eta
\le
\frac{1+\sigma^2}{\sigma^2}+1
\le
\frac{2(1+\sigma^2)}{\sigma^2}.
\]

For $\gamma=1$, apply
Lemma~\ref{lem:multiplication-network} with
\[
d=2,\qquad
C=\frac{2(1+\sigma^2)}{\sigma^2},
\qquad
\varepsilon=\frac{\delta}{4},
\qquad
\varepsilon'=2\eta,
\]
and take
\[
x=\left(t,\frac1{q_t}\right),
\qquad
x'=\left(t,\chi_{0,\delta}(t)\right).
\]
By \eqref{eq:fm-reciprocal-error},
$\|x-x'\|_\infty\le2\eta$. Therefore there exists a multiplication
network $\varphi_{1,\delta}$ satisfying
\[
\left|
\varphi_{1,\delta}
\bigl(t,\chi_{0,\delta}(t)\bigr)
-\frac{t}{q_t}
\right|
\le
\frac{\delta}{4}
+
2\left(\frac{2(1+\sigma^2)}{\sigma^2}\right)(2\eta)
=
\frac{\delta}{2}.
\]
Thus
\[
\chi_{1,\delta}(t)
:=
\varphi_{1,\delta}
\bigl(t,\chi_{0,\delta}(t)\bigr)
\]
satisfies the required $\delta$-accuracy.

For $\gamma=2$, we directly approximate $t^2q_t^{-1}$. Using the
same multiplication-network parameters, take
\[
x=\left(t^2,\frac1{q_t}\right),
\qquad
x'=\left(\vartheta_\eta(t),\chi_{0,\delta}(t)\right).
\]
Since $\eta<1$,
\[
\|x-x'\|_\infty
\le
\max\left\{\frac{\eta^3}{2},2\eta\right\}
\le2\eta.
\]
Hence there exists a multiplication network $\varphi_{2,\delta}$
such that
\[
\left|
\varphi_{2,\delta}
\bigl(\vartheta_\eta(t),\chi_{0,\delta}(t)\bigr)
-\frac{t^2}{q_t}
\right|
\le
\frac{\delta}{4}
+
2\left(\frac{2(1+\sigma^2)}{\sigma^2}\right)(2\eta)
=
\frac{\delta}{2}.
\]
Defining
\[
\chi_{2,\delta}(t)
:=
\varphi_{2,\delta}
\bigl(\vartheta_\eta(t),\chi_{0,\delta}(t)\bigr)
\]
proves the approximation claim for all $\gamma\in\{0,1,2\}$.

Since
\[
\eta^{-1}
=
\frac{32(1+\sigma^2)}{\delta\sigma^2},
\]
we have
\[
\log(\eta^{-1})
\lesssim
\log(1/\delta)
+
\log\!\left(\frac{1+\sigma^2}{\sigma^2}\right).
\]
Hence Lemma~\ref{lem:reciprocal-network} yields
\[
L_{\mathrm{rec}}
\lesssim
\log^2(1/\delta)
+
\log^2\!\left(\frac{1+\sigma^2}{\sigma^2}\right),
\]
\[
\|\mathbf W_{\mathrm{rec}}\|_\infty
\lesssim
\log^3(1/\delta)
+
\log^3\!\left(\frac{1+\sigma^2}{\sigma^2}\right),
\]
\[
S_{\mathrm{rec}}
\lesssim
\log^4(1/\delta)
+
\log^4\!\left(\frac{1+\sigma^2}{\sigma^2}\right),
\qquad
\log B_{\mathrm{rec}}
\lesssim
\log(1/\delta)
+
\log\!\left(\frac{1+\sigma^2}{\sigma^2}\right).
\]
The square network and the additional multiplication network used for $\gamma=1$ or $\gamma=2$ have only first-order logarithmic depth and sparsity, constant hidden width, and logarithmic weight complexity. Hence, under composition, their contributions are absorbed by the reciprocal-network bounds.
\end{proof}

\begin{proof}[Proof of Lemma~\ref{lem:fm-quadratic-term}]
Set
\[
C
:=
M\vee\frac{DM^2}{2}
\vee\frac{1+\sigma^2}{\sigma^2}.
\]
We first approximate $\|x\|^2/2$. Apply
Lemma~\ref{lem:multiplication-network} with $p=2$, domain
$[-C,C]^2$, approximation accuracy $\varepsilon/(4DC)$, and
perturbation level $\varepsilon/(8DC^2)$. For every
$j\in\{1,\ldots,D\}$, using the exact inputs $(x_j,x_j)$ gives
\[
\left|
\phi_{\mathrm{sq}}(x_j,x_j)-x_j^2
\right|
\le
\frac{\varepsilon}{4DC}
+
2C\frac{\varepsilon}{8DC^2}
=
\frac{\varepsilon}{2DC}.
\]
Parallel stacking of the $D$ square networks, followed by the affine
map $(z_1,\ldots,z_D)\mapsto \frac12\sum_{j=1}^D z_j$, yields
\begin{equation}
\label{eq:norm-square-approximation}
\left|
\frac12\sum_{j=1}^D\phi_{\mathrm{sq}}(x_j,x_j)
-\frac{\|x\|^2}{2}
\right|
\le
\frac{\varepsilon}{4C}.
\end{equation}

Lemma~\ref{lem:fm-time-coefficients} with $\gamma=0$ and accuracy
$\varepsilon/(4C)$ gives
\begin{equation}
\label{eq:reciprocal-input-e7}
\sup_{t\in[0,1]}
\left|
\chi_{0,\varepsilon/(4C)}(t)-\frac1{q_t}
\right|
\le
\frac{\varepsilon}{4C}.
\end{equation}
By the definition of $C$ and \eqref{eq:q-uniform-bound},
\[
\frac{\|x\|^2}{2}
\le\frac{DM^2}{2}\le C,
\qquad
\frac1{q_t}
\le
\frac{1+\sigma^2}{\sigma^2}
\le C.
\]
Thus both exact factors lie in $[-C,C]$.

Applying Lemma~\ref{lem:multiplication-network} again with
$d=2$, domain $[-C,C]^2$, approximation accuracy
$\varepsilon/2$, and perturbation level $\varepsilon/(4C)$, define
\[
\rho_\varepsilon(x,t)
:=
\phi_{\mathrm{mult}}
\left(
\chi_{0,\varepsilon/(4C)}(t),
\frac12\sum_{j=1}^D
\phi_{\mathrm{sq}}(x_j,x_j)
\right).
\]
By \eqref{eq:norm-square-approximation} and
\eqref{eq:reciprocal-input-e7}, the perturbation of both inputs is at
most $\varepsilon/(4C)$. Hence
\[
\left|
\rho_\varepsilon(x,t)
-\frac{\|x\|^2}{2q_t}
\right|
\le
\frac{\varepsilon}{2}
+
2C\frac{\varepsilon}{4C}
=
\varepsilon.
\]

We now determine the network configuration. Each square network has
\[
L_{\mathrm{sq}}\vee S_{\mathrm{sq}}
\lesssim
\log(1/\varepsilon)+\log D+\log C,
\qquad
\|\mathbf W_{\mathrm{sq}}\|_\infty\lesssim1,
\qquad
\log B_{\mathrm{sq}}\lesssim\log C.
\]
By Lemma~\ref{lem:parallelization-network}, the $D$ square networks
can be evaluated in parallel with
\[
L_{\mathrm{par}}=L_{\mathrm{sq}},
\qquad
\|\mathbf W_{\mathrm{par}}\|_\infty
\lesssim
D\|\mathbf W_{\mathrm{sq}}\|_\infty
\lesssim D,
\]
and
\[
S_{\mathrm{par}}
\lesssim
D S_{\mathrm{sq}}
\lesssim
D\bigl[
\log(1/\varepsilon)+\log D+\log C
\bigr].
\]
The subsequent affine summation layer adds one layer and only
$O(D)$ non-zero weights, and therefore does not change the above
width and sparsity orders.

Lemma~\ref{lem:fm-time-coefficients}, used with accuracy
$\varepsilon/(4C)$, gives a time-coefficient network satisfying
\[
L_{\mathrm{time}}
\lesssim
\log^2(1/\varepsilon)
+\log^2 C
+\log^2\!\left(\frac{1+\sigma^2}{\sigma^2}\right),
\]
\[
\|\mathbf W_{\mathrm{time}}\|_\infty
\lesssim
\log^3(1/\varepsilon)
+\log^3 C
+\log^3\!\left(\frac{1+\sigma^2}{\sigma^2}\right),
\]
\[
S_{\mathrm{time}}
\lesssim
\log^4(1/\varepsilon)
+\log^4 C
+\log^4\!\left(\frac{1+\sigma^2}{\sigma^2}\right),
\]
and
\[
\log B_{\mathrm{time}}
\lesssim
\log(1/\varepsilon)+\log C
+\log\!\left(\frac{1+\sigma^2}{\sigma^2}\right).
\]
The outer multiplication network has only first-order logarithmic
depth and sparsity and constant hidden width.

If the spatial and time branches have different depths, the scalar
output of the shallower branch is propagated through identity layers
until both branches have the same depth. Since only a scalar output is
padded, this adds at most $O(L_{\mathrm{time}})$ non-zero parameters,
which is absorbed by $S_{\mathrm{time}}$.

Consequently, before simplification,
\[
L
\lesssim
\log^2(1/\varepsilon)
+\log^2 C
+\log^2\!\left(\frac{1+\sigma^2}{\sigma^2}\right),
\]
\[
\|\mathbf W\|_\infty
\lesssim
D
+\log^3(1/\varepsilon)
+\log^3 C
+\log^3\!\left(\frac{1+\sigma^2}{\sigma^2}\right),
\]
\[
\begin{aligned}
S
\lesssim{}&
D\bigl[
\log(1/\varepsilon)+\log D+\log C
\bigr]\\
&+
\log^4(1/\varepsilon)
+\log^4 C
+\log^4\!\left(\frac{1+\sigma^2}{\sigma^2}\right),
\end{aligned}
\]
and
\[
\log B
\lesssim
\log(1/\varepsilon)+\log C
+\log\!\left(\frac{1+\sigma^2}{\sigma^2}\right).
\]

Since $M\ge1$,
\[
\log C
\lesssim
\log(MD)
+
\log\!\left(\frac{1+\sigma^2}{\sigma^2}\right).
\]
Moreover, $D\ge1$ and $(1+\sigma^2)/\sigma^2>2$, so the lower-order
additive terms can be absorbed into the corresponding higher-order
ones. Hence
\[
L
\lesssim
\log^2(1/\varepsilon)
+\log^2(MD)
+\log^2\!\left(\frac{1+\sigma^2}{\sigma^2}\right),
\]
\[
\|\mathbf W\|_\infty
\lesssim
D\left[
\log^3(1/\varepsilon)
+\log^3(MD)
+\log^3\!\left(\frac{1+\sigma^2}{\sigma^2}\right)
\right],
\]
\[
S
\lesssim
D\left[
\log^4(1/\varepsilon)
+\log^4(MD)
+\log^4\!\left(\frac{1+\sigma^2}{\sigma^2}\right)
\right],
\]
and
\[
\log B
\lesssim
\log(1/\varepsilon)
+\log(MD)
+\log\!\left(\frac{1+\sigma^2}{\sigma^2}\right).
\]
This proves the stated configuration.
\end{proof}

\begin{proof}[Proof of Lemma~\ref{lem:fm-linear-term}]
Set
\[
C
:=
DM\|a\|_\infty
\vee
\frac{1+\sigma^2}{\sigma^2}.
\]
Since $M\ge1$ and $0<\sigma<1$, we have $C\ge1$. For
$\|x\|_\infty\le M$,
\[
|x^\top a|
\le
D\|x\|_\infty\|a\|_\infty
\le
DM\|a\|_\infty
\le C.
\]
Moreover, by \eqref{eq:q-uniform-bound} and $t\in[0,1]$,
\[
0
\le
\frac{t}{q_t}
\le
\frac1{q_t}
\le
\frac{1+\sigma^2}{\sigma^2}
\le C.
\]
Hence the two exact factors
\[
x^\top a,
\qquad
\frac{t}{q_t}
\]
belong to $[-C,C]$.

Lemma~\ref{lem:fm-time-coefficients}, with $\gamma=1$ and accuracy
$\varepsilon/(4C)$, provides a network
$\chi_{1,\varepsilon/(4C)}$ satisfying
\[
\sup_{t\in[0,1]}
\left|
\chi_{1,\varepsilon/(4C)}(t)
-
\frac{t}{q_t}
\right|
\le
\frac{\varepsilon}{4C}.
\]
The map
\[
\ell_a(x):=x^\top a
\]
is affine and is therefore represented exactly. Applying
Lemma~\ref{lem:multiplication-network} with $p=2$, domain
$[-C,C]^2$, approximation accuracy $\varepsilon/2$, and perturbation
level $\varepsilon/(4C)$, define
\[
\omega_\varepsilon(x,t)
:=
\psi_{\mathrm{mult}}
\left(
\ell_a(x),
\chi_{1,\varepsilon/(4C)}(t)
\right).
\]
The first input is exact, while the second satisfies
\[
\left|
\chi_{1,\varepsilon/(4C)}(t)
-\frac{t}{q_t}
\right|
\le
\frac{\varepsilon}{4C}.
\]
Therefore the perturbation of the two-dimensional input is bounded by
$\varepsilon/(4C)$ in the $\ell_\infty$ norm, and
Lemma~\ref{lem:multiplication-network} yields
\[
\begin{aligned}
\left|
\omega_\varepsilon(x,t)
-\frac{t\,x^\top a}{q_t}
\right|
&\le
\frac{\varepsilon}{2}
+
2C\frac{\varepsilon}{4C}
&=
\varepsilon.
\end{aligned}
\]
This proves the approximation claim.

The resulting network configuration is obtained as follows.
The exact affine map $\ell_a(x)=x^\top a$ uses $D$ coefficients and
hence contributes
\[
S_{\mathrm{lin}}\lesssim D,
\qquad
B_{\mathrm{lin}}
\lesssim
\|a\|_\infty\vee1.
\]
Moreover, since the full network takes $(x,t)\in\mathbb R^{D+1}$ as
input, its width parameter satisfies a contribution of order $D+1$.

For the time branch, Lemma~\ref{lem:fm-time-coefficients} with
accuracy $\varepsilon/(4C)$ gives
\[
L_{\mathrm{time}}
\lesssim
\log^2(1/\varepsilon)
+\log^2 C
+\log^2\!\left(\frac{1+\sigma^2}{\sigma^2}\right),
\]
\[
\|\mathbf W_{\mathrm{time}}\|_\infty
\lesssim
\log^3(1/\varepsilon)
+\log^3 C
+\log^3\!\left(\frac{1+\sigma^2}{\sigma^2}\right),
\]
\[
S_{\mathrm{time}}
\lesssim
\log^4(1/\varepsilon)
+\log^4 C
+\log^4\!\left(\frac{1+\sigma^2}{\sigma^2}\right),
\]
and
\[
\log B_{\mathrm{time}}
\lesssim
\log(1/\varepsilon)
+\log C
+\log\!\left(\frac{1+\sigma^2}{\sigma^2}\right).
\]

The linear branch has scalar output. If its depth is smaller than that
of the time branch, we propagate this scalar through identity layers
until the two branches have the same depth. This requires only
$O(L_{\mathrm{time}})$ additional non-zero parameters and constant
additional width, and is therefore absorbed by
$S_{\mathrm{time}}$ and $\|\mathbf W_{\mathrm{time}}\|_\infty$.

The outer multiplication network has constant hidden width and only
first-order logarithmic depth and sparsity. Hence its contribution is
also absorbed by the time-coefficient network. Combining all
components first gives
\[
L
\lesssim
\log^2(1/\varepsilon)
+\log^2 C
+\log^2\!\left(\frac{1+\sigma^2}{\sigma^2}\right),
\]
\[
\|\mathbf W\|_\infty
\lesssim
D
+\log^3(1/\varepsilon)
+\log^3 C
+\log^3\!\left(\frac{1+\sigma^2}{\sigma^2}\right),
\]
\[
S
\lesssim
D
+\log^4(1/\varepsilon)
+\log^4 C
+\log^4\!\left(\frac{1+\sigma^2}{\sigma^2}\right),
\]
and
\[
\log B
\lesssim
\log(1/\varepsilon)
+\log C
+\log\!\left(\frac{1+\sigma^2}{\sigma^2}\right).
\]

Finally, since
\[
C
=
DM\|a\|_\infty
\vee
\frac{1+\sigma^2}{\sigma^2},
\]
we have
\[
\log C
\lesssim
\log(DM\|a\|_\infty\vee1)
+
\log\!\left(\frac{1+\sigma^2}{\sigma^2}\right).
\]
Using the standing fixed-power logarithmic simplification therefore
gives
\[
L
\lesssim
\log^2(1/\varepsilon)
+\log^2(DM\|a\|_\infty\vee1)
+\log^2\!\left(\frac{1+\sigma^2}{\sigma^2}\right),
\]
\[
\|\mathbf W\|_\infty
\lesssim
D
+\log^3(1/\varepsilon)
+\log^3(DM\|a\|_\infty\vee1)
+\log^3\!\left(\frac{1+\sigma^2}{\sigma^2}\right),
\]
\[
S
\lesssim
D
+\log^4(1/\varepsilon)
+\log^4(DM\|a\|_\infty\vee1)
+\log^4\!\left(\frac{1+\sigma^2}{\sigma^2}\right),
\]
and
\[
\log B
\lesssim
\log(1/\varepsilon)
+\log(DM\|a\|_\infty\vee1)
+\log\!\left(\frac{1+\sigma^2}{\sigma^2}\right).
\]
This proves the stated configuration.
\end{proof}

\section{Proofs for Generalization}
\label{app:generalization}

We use the notation $P$, $P_n$, and $\Delta_v$ from
Section~\ref{sec:generalization}. The main theorem proof is given first; the
risk, truncation, entropy, and concentration lemmas used by that proof are
proved in the following subsections.

\subsection{Proof of the finite-sample theorem}
\label{app:generalization-main}

\noindent\textit{Proof of Theorem~\ref{thm:main-generalization}.}

\paragraph{Step 1: Risk representation and fixed-field localization.}
Lemma~\ref{lem:fm-risk-identity} identifies $P\Delta_v$ with the integrated
$L^2(\pi_t)$ velocity error. Lemma~\ref{lem:fm-excess-localization} supplies
both the localized second moment and the centered sub-exponential envelope,
and hence \eqref{eq:main-fm-fixed-bernstein} for every fixed deterministic
field. This does not yet control the data-dependent ERM.

\paragraph{Step 2: Finite reduction on the unbounded domain.}
Apply Lemma~\ref{lem:main-fm-finite-reduction} with confidence $\delta/2$.
The bounded-domain ReLU entropy and Gaussian truncation give a deterministic
family $v_1,\ldots,v_M$ inside
$\mathcal V(L,W,S,B,\gamma_x)$. On an event of probability at least
$1-\delta/2$, every field has one representative satisfying
\eqref{eq:main-fm-finite-reduction} simultaneously for $P$ and $P_n$.

\paragraph{Step 3: Uniform oracle inequality.}
Lemma~\ref{lem:fm-uniform-concentration} applies the fixed-field inequality to
the deterministic representatives, takes a union bound, and transfers the
result back to the full class. The ERM comparison then gives
\eqref{eq:main-fm-oracle}.

\paragraph{Step 4: Approximation--estimation balance.}
Set $\varepsilon:=n^{-1/(2\beta+d)}$ and suppose that this choice satisfies
the conditions of Theorem~\ref{thm:velocity-approximation}. Choose
$L,W,S,B$ and $\gamma_x$ as in that theorem and set
$\tau:=\varepsilon^{2\beta}$. The comparator supplied there satisfies
\begin{equation}
 \inf_{v\in\mathcal V(L,W,S,B,\gamma_x)}P\Delta_v
 \lesssim_{d,\beta}
 \frac{D(1+H^2)}{\sigma}\,\varepsilon^{2\beta}.
 \label{eq:fm-approximation-before-balance}
\end{equation}
Its architecture obeys
\[
 L,\log B\lesssim_{d,\beta}
 \Gamma_\varepsilon^{c_{\rm arch}},\qquad
 S\lesssim_{d,\beta}D^2\varepsilon^{-d}
 \Gamma_\varepsilon^{c_{\rm arch}}.
\]
Substituting these bounds into
\eqref{eq:fm-finite-reduction-cardinality} gives
\begin{equation}
 \log\frac{8M}{\delta}
 \lesssim_{d,\beta}D^2\varepsilon^{-d}
 \bigl\{\Gamma_\varepsilon+\log(en)+\log(e/\delta)\bigr\}^{c_{\rm arch}}.
 \label{eq:fm-entropy-before-balance}
\end{equation}
Indeed, $SL\lesssim_{d,\beta}D^2\varepsilon^{-d}
\Gamma_\varepsilon^{c_{\rm arch}}$: here the fixed exponent is chosen
after adding the individual, uncoarsened exponents from $S$, $L$, and
the remaining logarithm in the ReLU covering bound. Combining
\eqref{eq:fm-approximation-before-balance},
\eqref{eq:fm-entropy-before-balance}, and
\eqref{eq:main-fm-oracle} yields
\[
\begin{aligned}
 P\Delta_{\widehat v}
 &\lesssim_{d,\beta}
 \frac{D(1+H^2)}{\sigma}\,\varepsilon^{2\beta}\\
 &\quad+\frac{D^2\varepsilon^{-d}}n
 \left[\frac1{\sigma^2}
 +\left(\frac1\sigma+\sqrt D\right)\log(en)\right]
 \bigl\{\Gamma_\varepsilon+\log(en)+\log(e/\delta)\bigr\}^{c_{\rm arch}}.
\end{aligned}
\]
Since $\varepsilon^{2\beta}=\varepsilon^{-d}/n
=n^{-2\beta/(2\beta+d)}$ and
$\log(en)\lesssim_{d,\beta}\Gamma_\varepsilon$, this is
\eqref{eq:fm-main-generalization-rate}. Lemma~\ref{lem:fm-risk-identity}
converts $P\Delta_{\widehat v}$ to the asserted integrated velocity error.
\hfill\qedsymbol

\subsection{Risk identity and localized excess loss}
\label{app:generalization-localization}

\begin{proof}[Proof of Lemma~\ref{lem:fm-risk-identity}]
Fix $t\in[0,1]$ and write $Y=X_1-X_0$. Since
$v^*(X_t,t)=\E[Y\mid X_t]$,
\[
\begin{aligned}
 \|Y-v(X_t,t)\|^2
 &=\|Y-v^*(X_t,t)\|^2
 +\|v(X_t,t)-v^*(X_t,t)\|^2\\
 &\quad-2\langle Y-v^*(X_t,t),
 v(X_t,t)-v^*(X_t,t)\rangle.
\end{aligned}
\]
The last factor in the cross term is measurable with respect to $X_t$, and
\[
 \E[Y-v^*(X_t,t)\mid X_t]=0.
\]
Hence the cross term has expectation zero. Taking expectations and
integrating over $t\in[0,1]$ gives
\[
 \mathcal L(v)-\mathcal L(v^*)
 =\int_0^1\E\|v(X_t,t)-v^*(X_t,t)\|^2dt,
\]
which is \eqref{eq:fm-risk-identity}.
\end{proof}

For a fixed deterministic field, conditional label covariance and the
bounded structured discrepancy provide the localization needed below.

\begin{lemma}[Conditional covariance of the Flow Matching label]
    \label{lem:fm-conditional-covariance}
    For every fixed $t\in[0,1]$,
    \begin{equation}
        \operatorname{Cov}(X_1-X_0\mid X_t)
        =
        \frac{\sigma^2}{q_t}I_D
        +
        b_t^2\operatorname{Cov}(g^*(U)\mid X_t).
        \label{eq:fm-conditional-covariance}
    \end{equation}
    Consequently,
    \begin{equation}
        \lambda_{\max}
        \bigl(
            \operatorname{Cov}(X_1-X_0\mid X_t)
        \bigr)
        \lesssim
        \sigma^{-2},
        \qquad t\in[0,1].
        \label{eq:fm-conditional-covariance-bound}
    \end{equation}
\end{lemma}

\begin{proof}
    Fix $t\in[0,1]$. Conditional on $U$,
    \[
        X_1-X_0
        =
        g^*(U)+\sigma\xi-X_0,
        \qquad
        X_t
        =
        tg^*(U)+(1-t)X_0+t\sigma\xi.
    \]
    Since $X_0$ and $\xi$ are independent standard Gaussian vectors,
    \[
        \operatorname{Var}(X_1-X_0\mid U)
        =
        (1+\sigma^2)I_D,
        \qquad
        \operatorname{Var}(X_t\mid U)
        =
        q_tI_D,
    \]
    and
    \[
        \operatorname{Cov}(X_1-X_0,X_t\mid U)
        =
        \bigl(t\sigma^2-(1-t)\bigr)I_D.
    \]
    Therefore, Gaussian conditioning gives
    \begin{align*}
        \operatorname{Cov}(X_1-X_0\mid X_t,U)
        &=
        \left[
            1+\sigma^2
            -
            \frac{(t\sigma^2-(1-t))^2}{q_t}
        \right]I_D \nonumber\\
        &=
        \frac{\sigma^2}{q_t}I_D,
    \end{align*}
    where we used
    \[
        (1+\sigma^2)q_t
        -
        \bigl(t\sigma^2-(1-t)\bigr)^2
        =
        \sigma^2.
    \]

    The corresponding conditional mean is
    \begin{align*}
        \E[X_1-X_0\mid X_t,U]
        &=
        g^*(U)
        +
        \frac{t\sigma^2-(1-t)}{q_t}
        \bigl(X_t-tg^*(U)\bigr)\\
        &=
        a_tX_t
        +
        \left(
            1-ta_t
        \right)g^*(U)\\
        &=
        a_tX_t+b_tg^*(U),
    \end{align*}
    where $1-ta_t=b_t$.

Applying the conditional law of total covariance from $(X_t,U)$ to $X_t$ yields
    \[
    \begin{aligned}
        \operatorname{Cov}(X_1-X_0\mid X_t)
        &=
        \E\left[
            \operatorname{Cov}(X_1-X_0\mid X_t,U)
            \,\middle|\,
            X_t
        \right]\\
        &\quad+
        \operatorname{Cov}
        \left(
            \E[X_1-X_0\mid X_t,U]
            \,\middle|\,
            X_t
        \right)\\
        &=
        \frac{\sigma^2}{q_t}I_D
        +
        b_t^2\operatorname{Cov}(g^*(U)\mid X_t),
    \end{aligned}
    \]
    which proves \eqref{eq:fm-conditional-covariance}.

For the operator-norm bound, $\|g^*(U)\|\le1$ gives, for every unit vector
$u\in\mathbb R^D$,
    \[
    \begin{aligned}
        u^\top
        \operatorname{Cov}(g^*(U)\mid X_t)u
        &=
        \operatorname{Var}
        \bigl(
            u^\top g^*(U)\mid X_t
        \bigr)\\
        &\le
        \E\left[
            (u^\top g^*(U))^2
            \,\middle|\,
            X_t
        \right]
        \le1.
    \end{aligned}
    \]
    Hence
    \[
        \lambda_{\max}
        \bigl(
            \operatorname{Cov}(g^*(U)\mid X_t)
        \bigr)
        \le1.
    \]

    By the lower bound on $q_t$ stated in
    Section~\ref{sec:preliminaries}, and since $0<\sigma<1$,
    \[
        \frac{\sigma^2}{q_t}
        \le
        1+\sigma^2
        \le2.
    \]
    Also,
    \[
        b_t^2
        =
        \frac{(1-t)^2}{q_t^2}
        \le
        \frac1{q_t}
        \le
        \frac{1+\sigma^2}{\sigma^2}
        \le
        \frac2{\sigma^2}.
    \]
    Combining these estimates with
    \eqref{eq:fm-conditional-covariance} proves
    \eqref{eq:fm-conditional-covariance-bound}.
\end{proof}

The covariance bound controls the stochastic label. To convert it into a
variance-to-risk relation, we also need a uniform bound on the discrepancy
between a structured field and the population velocity.

\begin{lemma}[Uniform discrepancy of structured velocity fields]
        \label{lem:fm-velocity-discrepancy}
        For every $v\in\mathcal V(L,W,S,B,\gamma_x)$,
        \begin{equation}
            \|v(x,t)-v^*(x,t)\|^2
            \lesssim
            \sigma^{-2},
            \qquad
            (x,t)\in\mathbb R^D\times[0,1].
            \label{eq:fm-velocity-discrepancy}
        \end{equation}
        \end{lemma}

        \begin{proof}
        For $v=v_\theta\in\mathcal V(L,W,S,B,\gamma_x)$,
        \[
            v(x,t)
            =
            a_tx
            +
            b_t\operatorname{clip}_2(f_\theta(x,t)),
        \]
        whereas
        \[
            v^*(x,t)
            =
            a_tx+b_tm_t(x).
        \]
        Hence the common affine term cancels:
        \[
            v(x,t)-v^*(x,t)
            =
            b_t
            \left[
                \operatorname{clip}_2(f_\theta(x,t))-m_t(x)
            \right].
        \]
        Since
        \[
            \|\operatorname{clip}_2(f_\theta(x,t))\|\le2
        \]
        and
        \[
            \|m_t(x)\|
            =
            \|\E[g^*(U)\mid X_t=x]\|
            \le
            \E[\|g^*(U)\|\mid X_t=x]
            \le1,
        \]
        we obtain
        \[
            \|v(x,t)-v^*(x,t)\|
            \le
            3|b_t|.
        \]
        Using
        \[
            b_t^2
            \le
            \frac2{\sigma^2}
        \]
        from the proof of Lemma~\ref{lem:fm-conditional-covariance} gives
        \[
            \|v(x,t)-v^*(x,t)\|^2
            \le
            \frac{18}{\sigma^2},
        \]
        which proves \eqref{eq:fm-velocity-discrepancy}.
\end{proof}

The preceding two bounds yield the localized second-moment estimate for
the target-indexed excess loss.

\begin{lemma}[Bernstein condition for the Flow Matching excess loss]
\label{lem:fm-bernstein}
Under the standing assumptions, for every
$v\in\mathcal V(L,W,S,B,\gamma_x)$,
\begin{equation}
    \E_{X_1}\left[\Delta_v(X_1)^2\right]
    \lesssim
    \frac1{\sigma^2}
    \E_{X_1}\Delta_v(X_1).
    \label{eq:fm-bernstein}
\end{equation}
The hidden constant is absolute and, in particular, independent of
$D,L,W,S,B$.
\end{lemma}

\begin{proof}[Proof of Lemma~\ref{lem:fm-bernstein}]
    By the definition of $\Delta_v$ and Jensen's inequality with respect to
    the probability measure $dt\otimes P_{X_0}$,
    \begin{equation}
    \begin{aligned}
        \E_{X_1}\Delta_v(X_1)^2
        \le
        \int_0^1
        \E
        \Big[
            &\|X_1-X_0-v(X_t,t)\|^2 \\
            &-
            \|X_1-X_0-v^*(X_t,t)\|^2
        \Big]^2dt.
    \end{aligned}
    \label{eq:fm-bernstein-jensen}
    \end{equation}

    Fix $t\in[0,1]$. Expanding the squared-loss difference gives
    \[
    \begin{aligned}
    &
        \|X_1-X_0-v(X_t,t)\|^2
        -
        \|X_1-X_0-v^*(X_t,t)\|^2 \\
    &\qquad=
        \|v(X_t,t)-v^*(X_t,t)\|^2
        -
        2\left\langle
            X_1-X_0-v^*(X_t,t),
            v(X_t,t)-v^*(X_t,t)
        \right\rangle .
    \end{aligned}
    \]
    Since
    \[
        v^*(X_t,t)=\E[X_1-X_0\mid X_t],
    \]
    the tower property yields
    \begin{equation*}
    \begin{aligned}
    &
    \E
    \Big[
        \|X_1-X_0-v(X_t,t)\|^2
        -
        \|X_1-X_0-v^*(X_t,t)\|^2
    \Big]^2 \\
    &=
    \E
    \Bigg[
        \E
        \left[
        \left.
        \Big(
            \|X_1-X_0-v(X_t,t)\|^2
            -
            \|X_1-X_0-v^*(X_t,t)\|^2
        \Big)^2
        \right|X_t
        \right]
    \Bigg] \\
    &=
    \E\|v(X_t,t)-v^*(X_t,t)\|^4 \\
    &\quad+
    4\E
    \Big[
        \bigl(v(X_t,t)-v^*(X_t,t)\bigr)^\top
        \operatorname{Cov}(X_1-X_0\mid X_t)
        \bigl(v(X_t,t)-v^*(X_t,t)\bigr)
    \Big].
    \end{aligned}
    \end{equation*}

    Lemma~\ref{lem:fm-velocity-discrepancy} gives
    \[
        \|v(X_t,t)-v^*(X_t,t)\|^4
        \lesssim
        \frac1{\sigma^2}
        \|v(X_t,t)-v^*(X_t,t)\|^2,
    \]
    while Lemma~\ref{lem:fm-conditional-covariance} gives
    \[
    \begin{aligned}
    &
    \bigl(v(X_t,t)-v^*(X_t,t)\bigr)^\top
    \operatorname{Cov}(X_1-X_0\mid X_t)
    \bigl(v(X_t,t)-v^*(X_t,t)\bigr) \\
    &\qquad\lesssim
    \frac1{\sigma^2}
    \|v(X_t,t)-v^*(X_t,t)\|^2.
    \end{aligned}
    \]
    Hence, for every fixed $t\in[0,1]$,
    \begin{equation}
    \begin{aligned}
    &
    \E
    \Big[
        \|X_1-X_0-v(X_t,t)\|^2
        -
        \|X_1-X_0-v^*(X_t,t)\|^2
    \Big]^2 \\
    &\qquad\lesssim
    \frac1{\sigma^2}
    \E\|v(X_t,t)-v^*(X_t,t)\|^2.
    \end{aligned}
    \label{eq:fm-bernstein-fixed-time}
    \end{equation}
    Combining \eqref{eq:fm-bernstein-jensen} and
    \eqref{eq:fm-bernstein-fixed-time}, and then using the population-risk
    identity, gives
    \[
    \begin{aligned}
        \E_{X_1}\Delta_v(X_1)^2
        &\lesssim
        \frac1{\sigma^2}
        \int_0^1
        \E\|v(X_t,t)-v^*(X_t,t)\|^2dt \\
        &=
        \frac1{\sigma^2}
        \bigl[\mathcal L(v)-\mathcal L(v^*)\bigr] \\
        &=
        \frac1{\sigma^2}\E_{X_1}\Delta_v(X_1).
    \end{aligned}
    \]
    This proves \eqref{eq:fm-bernstein}.
    \end{proof}

The second moment controls local variance, while fixed-field concentration
also requires an envelope for the unbounded target observation.

\begin{lemma}[Uniform tail bound for the Flow Matching excess loss]
\label{lem:fm-excess-tail}
For every $v=v_\theta\in\mathcal V(L,W,S,B,\gamma_x)$ and every
$x_1\in\mathbb R^D$,
\begin{equation}
    |\Delta_v(x_1)|
    \le
    \frac{3\pi}{2\sigma}\|x_1\|
    +
    \frac{9\pi}{4\sigma}
    +
    3\sqrt D .
    \label{eq:fm-excess-pointwise-envelope}
\end{equation}
Consequently,
\begin{equation}
    \sup_{v\in\mathcal V(L,W,S,B,\gamma_x)}
    \|\Delta_v(X_1)\|_{\psi_1}
    \lesssim
    \frac1{\sigma}+\sqrt D,
    \label{eq:fm-excess-psi1}
\end{equation}
and hence
\begin{equation}
    \sup_{v\in\mathcal V(L,W,S,B,\gamma_x)}
    \left\|
        \Delta_v(X_1)-\E_{X_1}\Delta_v(X_1)
    \right\|_{\psi_1}
    \lesssim
    \frac1{\sigma}+\sqrt D.
    \label{eq:fm-centered-excess-psi1}
\end{equation}
The hidden constants are absolute.
\end{lemma}

The raw FM square loss is unbounded, but the common affine velocity component cancels in the excess loss and clipping controls the nonlinear discrepancy.

\begin{proof}[Proof of Lemma~\ref{lem:fm-excess-tail}]
Fix $v=v_\theta\in\mathcal V(L,W,S,B,\gamma_x)$ and $x_1\in\mathbb R^D$. Within
this proof, write
\[
    \bar f_\theta(x,t):=\operatorname{clip}_2(f_\theta(x,t)).
\]
Then $\|\bar f_\theta(x,t)\|\le2$, while the representation of the population
velocity implies $\|m_t(x)\|\le1$. Hence
\begin{equation}
    \|\bar f_\theta(x,t)-m_t(x)\|\le3.
    \label{eq:fm-tail-discrepancy}
\end{equation}

With $X_1$ fixed at $x_1$, write
$X_t=(1-t)X_0+t x_1$. Since
\[
    1-ta_t=b_t,
    \qquad
    1+(1-t)a_t=\frac{\sigma^2t}{q_t},
\]
the residuals corresponding to $v_\theta$ and $v^*$ satisfy
\[
\begin{aligned}
    x_1-X_0-v_\theta(X_t,t)
    &=
    b_t\bigl(x_1-\bar f_\theta(X_t,t)\bigr)
    -
    \frac{\sigma^2t}{q_t}X_0,\\
    x_1-X_0-v^*(X_t,t)
    &=
    b_t\bigl(x_1-m_t(X_t)\bigr)
    -
    \frac{\sigma^2t}{q_t}X_0.
\end{aligned}
\]
Subtracting the corresponding squared norms gives
\begin{align}
&
    \|x_1-X_0-v_\theta(X_t,t)\|^2
    -
    \|x_1-X_0-v^*(X_t,t)\|^2
    \nonumber\\
&=
    b_t^2
    \left(
        \|x_1-\bar f_\theta(X_t,t)\|^2
        -
        \|x_1-m_t(X_t)\|^2
    \right)
    \nonumber\\
&\quad+
    2b_t\frac{\sigma^2t}{q_t}
    \left\langle
        X_0,
        \bar f_\theta(X_t,t)-m_t(X_t)
    \right\rangle .
\label{eq:fm-tail-loss-difference}
\end{align}

Using
\[
    |\|a\|^2-\|b\|^2|
    \le
    \|a-b\|(\|a\|+\|b\|)
\]
together with \eqref{eq:fm-tail-discrepancy}, we obtain
\[
\begin{aligned}
&
\left|
    \|x_1-\bar f_\theta(X_t,t)\|^2
    -
    \|x_1-m_t(X_t)\|^2
\right|
\\
&\qquad\le
    \|\bar f_\theta(X_t,t)-m_t(X_t)\|
    \left(
        2\|x_1\|
        +
        \|\bar f_\theta(X_t,t)\|
        +
        \|m_t(X_t)\|
    \right)
\\
&\qquad\le
    6\|x_1\|+9.
\end{aligned}
\]
The second term in \eqref{eq:fm-tail-loss-difference} is bounded by
\[
    2b_t\frac{\sigma^2t}{q_t}
    \|X_0\|
    \|\bar f_\theta(X_t,t)-m_t(X_t)\|
    \le
    6b_t\frac{\sigma^2t}{q_t}\|X_0\|.
\]
Therefore, by the definition of $\Delta_v$,
\begin{align}
    |\Delta_v(x_1)|
    &\le
    (6\|x_1\|+9)\int_0^1 b_t^2dt
    +
    6\E\|X_0\|
    \int_0^1
        b_t\frac{\sigma^2t}{q_t}dt.
    \label{eq:fm-tail-before-time-integration}
\end{align}

Using $q_t=(1-t)^2+\sigma^2t^2$ and the change of variables
$r=t/(1-t)$,
\begin{equation*}
\begin{aligned}
    \int_0^1 b_t^2dt
    &=
    \int_0^\infty
    \frac{dr}{(1+\sigma^2r^2)^2}
    =
    \frac{\pi}{4\sigma},\\
    \int_0^1
        b_t\frac{\sigma^2t}{q_t}dt
    &=
    \int_0^\infty
    \frac{\sigma^2r}{(1+\sigma^2r^2)^2}dr
    =
    \frac12.
\end{aligned}
\end{equation*}
Since $X_0\sim N(0,I_D)$,
\[
    \E\|X_0\|
    \le
    \sqrt{\E\|X_0\|^2}
    =
    \sqrt D.
\]
Substituting these estimates into
\eqref{eq:fm-tail-before-time-integration} yields
\[
    |\Delta_v(x_1)|
    \le
    \frac{3\pi}{2\sigma}\|x_1\|
    +
    \frac{9\pi}{4\sigma}
    +
    3\sqrt D,
\]
which proves \eqref{eq:fm-excess-pointwise-envelope}.

For the Orlicz norm under the target distribution, recall that
\[
    X_1=g^*(U)+\sigma\xi,
    \qquad
    \xi\sim N(0,I_D),
\]
with $\|g^*(U)\|\le1$. Hence
\[
    \|X_1\|
    \le
    1+\sigma\|\xi\|.
\]
Since the Euclidean norm of a standard Gaussian vector satisfies
\[
    \bigl\|\|\xi\|\bigr\|_{\psi_1}
    \lesssim
    \sqrt D,
\]
\eqref{eq:fm-excess-pointwise-envelope} and the triangle inequality for
the $\psi_1$ norm imply
\[
\begin{aligned}
    \|\Delta_v(X_1)\|_{\psi_1}
    &\lesssim
    \frac1{\sigma}
    \left(
        1+\bigl\|\|X_1\|\bigr\|_{\psi_1}
    \right)
    +
    \sqrt D\\
    &\lesssim
    \frac1{\sigma}
    \left(
        1+\sigma\sqrt D
    \right)
    +
    \frac1{\sigma}
    +
    \sqrt D\\
    &\lesssim
    \frac1{\sigma}+\sqrt D.
\end{aligned}
\]
The bound is independent of $v$, which proves
\eqref{eq:fm-excess-psi1}. Finally, the standard centering inequality
for the sub-exponential Orlicz norm gives
\[
    \|\Delta_v(X_1)-\E_{X_1}\Delta_v(X_1)\|_{\psi_1}
    \lesssim
    \|\Delta_v(X_1)\|_{\psi_1},
\]
and \eqref{eq:fm-centered-excess-psi1} follows.
\end{proof}

Combining the localized second moment with this tail envelope gives the
fixed-field Bernstein inequality stated in the main text.

\begin{proof}[Proof of Lemma~\ref{lem:fm-excess-localization}]
Lemma~\ref{lem:fm-bernstein} and Lemma~\ref{lem:fm-excess-tail} give the two
bounds in \eqref{eq:fm-excess-localization}. For a fixed deterministic $v$,
Lemma~\ref{lem:fm-excess-tail} and the maximal inequality for
sub-exponential variables give
\[
 \left\|\max_{1\le i\le n}
 \left|\Delta_v(X_{1,i})-P\Delta_v\right|\right\|_{\psi_1}
 \lesssim(\sigma^{-1}+\sqrt D)\log(en).
\]
Applying the Bernstein inequality for unbounded variables
\cite[Proposition~5.2]{lecue2012crossvalidation} to both signs, and using
\eqref{eq:fm-bernstein}, proves \eqref{eq:main-fm-fixed-bernstein}.
\end{proof}

\subsection{Truncation, local loss stability, and finite reduction}
\label{app:generalization-finite-reduction}

\begin{lemma}[Local stability of the Flow Matching loss]
    \label{lem:fm-local-loss-stability}
    Let $v_1,v_2\in\mathcal V(L,W,S,B,\gamma_x)$ be induced by
    $f_{\theta_1},f_{\theta_2}\in\mathrm{NN}(L,W,S,B)$ according to
    the definition in \eqref{eq:structured-velocity-class}. Let $R_0,R_1\ge1$ and
    set $R:=R_0\vee R_1$. Suppose that
    \[
        \|f_{\theta_1}-f_{\theta_2}\|_
        {L^\infty(
            \mathcal B(0,R)\times[0,1];\ell_2)}
        \le\eta.
    \]
    Then, for every $x_1\in\mathcal B(0,R_1)$,
    \begin{equation}
    \begin{aligned}
    |\ell(x_1,v_1)-\ell(x_1,v_2)|
    &\le
    \eta\left[
        \frac{\pi}{2\sigma}(R_1+2)+\sqrt D
    \right]
    \\
    &\quad+
    \frac{2\pi}{\sigma}(R_1+2)
    \Pr(\|X_0\|>R_0)
    +
    4\sqrt{
        D\Pr(\|X_0\|>R_0)
    }.
    \end{aligned}
    \label{eq:fm-local-loss-stability}
    \end{equation}
    Moreover,
    \begin{equation}
    \begin{aligned}
    \left|
        P(\Delta_{v_1}-\Delta_{v_2})
    \right|
    \lesssim
    \left(
        \frac1{\sigma}+\sqrt D
    \right)
    \Big[
    &\eta
    +\sqrt{\Pr(\|X_0\|>R_0)}
    \\
    &+\sqrt{\Pr(\|X_1\|>R_1)}
    \Big].
    \end{aligned}
    \label{eq:fm-population-loss-stability}
    \end{equation}
    The hidden constant is absolute.
\end{lemma}

\begin{proof}
    By \eqref{eq:structured-velocity-class},
    \[
        v_j(x,t)
        =
        a_t x+
        b_t\operatorname{clip}_2(f_{\theta_j}(x,t)),
        \qquad j\in\{1,2\}.
    \]
    Fix $x_1\in\mathbb R^D$ and write
    \[
        X_t=(1-t)X_0+tx_1.
    \]
    Using
    \[
        1-ta_t=b_t,
        \qquad
        1+(1-t)a_t=\frac{\sigma^2t}{q_t},
    \]
    we obtain
    \[
        x_1-X_0-v_j(X_t,t)
        =
        b_t
        \left[
            x_1-\operatorname{clip}_2(f_{\theta_j}(X_t,t))
        \right]
        -
        \frac{\sigma^2t}{q_t}X_0.
    \]
    Therefore,
    \begin{equation}
    \begin{aligned}
    &
    \left|
    \|x_1-X_0-v_1(X_t,t)\|^2
    -
    \|x_1-X_0-v_2(X_t,t)\|^2
    \right|
    \\
    &\le
    \left[
        2b_t^2(\|x_1\|+2)
        +
        2b_t\frac{\sigma^2t}{q_t}\|X_0\|
    \right]
    \\
    &\qquad\qquad\times
    \left\|
        \operatorname{clip}_2(f_{\theta_1}(X_t,t))
        -
        \operatorname{clip}_2(f_{\theta_2}(X_t,t))
    \right\|.
    \end{aligned}
    \label{eq:fm-pairwise-loss-integrand}
    \end{equation}

    The clipping map is nonexpansive and takes values in
    $\mathcal B(0,2)$. If $\|x_1\|\le R_1$ and $\|X_0\|\le R_0$, then
    $\|X_t\|\le R$, so the last factor in
    \eqref{eq:fm-pairwise-loss-integrand} is at most $\eta$; globally it is
    bounded by $4$. Using
    \[
        \int_0^1 b_t^2\,dt=\frac{\pi}{4\sigma},
        \qquad
        \int_0^1
        b_t\frac{\sigma^2t}{q_t}\,dt=\frac12,
    \]
    together with
    \[
        \E\|X_0\|\le\sqrt D,
        \qquad
        \E\!\left[
            \|X_0\|
            \mathbf 1_{\{\|X_0\|>R_0\}}
        \right]
        \le
        \sqrt{
            D\Pr(\|X_0\|>R_0)
        },
    \]
    gives \eqref{eq:fm-local-loss-stability}.

    For the population bound, let
    \[
        A_0:=\{\|X_0\|>R_0\},
        \qquad
        A_1:=\{\|X_1\|>R_1\}.
    \]
    On $A_0^c\cap A_1^c$, $X_t\in\mathcal B(0,R)$ for every
    $t\in[0,1]$, while globally the clipped difference is at most $4$.
    Hence
    \[
    \begin{aligned}
    &
    \left\|
        \operatorname{clip}_2(f_{\theta_1}(X_t,t))
        -
        \operatorname{clip}_2(f_{\theta_2}(X_t,t))
    \right\|
    \\
    &\qquad\le
    \eta+4\mathbf 1_{A_0}+4\mathbf 1_{A_1}.
    \end{aligned}
    \]
    Substituting this bound into
    \eqref{eq:fm-pairwise-loss-integrand}, integrating over $t$, and taking
    expectation gives
    \begin{align*}
    \left|
        P(\Delta_{v_1}-\Delta_{v_2})
    \right|
    &\le
    \eta
    \left[
        \frac{\pi}{2\sigma}
        \bigl(\E\|X_1\|+2\bigr)
        +
        \E\|X_0\|
    \right]
    \\
    &\quad+
    \frac{2\pi}{\sigma}
    \bigl(\E\|X_1\|+2\bigr)\Pr(A_0)
    +
    4\E\!\left[
        \|X_0\|\mathbf 1_{A_0}
    \right]
    \\
    &\quad+
    \frac{2\pi}{\sigma}
    \E\!\left[
        (\|X_1\|+2)\mathbf 1_{A_1}
    \right]
    +
    4\E\|X_0\|\Pr(A_1),
    \end{align*}
    where independence of $X_0$ and $X_1$ is used in the cross terms.
    Since
    \[
        \E\|X_0\|\le\sqrt D,
        \qquad
        \E\|X_1\|^2\le1+\sigma^2D,
    \]
    Cauchy--Schwarz yields
    \[
        \E\!\left[
            (\|X_1\|+2)\mathbf 1_{A_1}
        \right]
        \lesssim
        (1+\sigma\sqrt D)\sqrt{\Pr(A_1)}.
    \]
    Together with
    \[
        \E\!\left[
            \|X_0\|\mathbf 1_{A_0}
        \right]
        \le
        \sqrt{D\Pr(A_0)},
    \]
    this proves \eqref{eq:fm-population-loss-stability}.
\end{proof}

The local stability estimate transfers posterior-network approximation to
conditional-loss approximation. We next quantify the bounded-domain network
cover used in that transfer.

\begin{lemma}[Metric entropy on a bounded state-time domain]
    \label{lem:fm-relu-covering}
    Let $R\ge1$. For every $\eta\in(0,1)$, there
    exists a finite internal net
    \[
        \mathcal F_\eta
        \subset
        \mathrm{NN}(L,W,S,B)
    \]
    such that
    \[
        \sup_{f_\theta\in\mathrm{NN}(L,W,S,B)}
        \inf_{f\in\mathcal F_\eta}
        \|f_\theta-f\|_
        {L^\infty(
            \mathcal B(0,R)\times[0,1];\ell_2)}
        \le\eta
    \]
    and
    \begin{equation}
        \log|\mathcal F_\eta|
        \lesssim
        SL
        \log\left(
            \frac{
                \sqrt D\,L(W+1)(B\vee1)(R+1)
            }{
                \eta
            }
        \right).
        \label{eq:fm-relu-internal-covering}
    \end{equation}
    Consequently,
    \begin{equation}
    \begin{aligned}
    &
    \log
    \mathcal N\left(
        \eta,
        \mathrm{NN}(L,W,S,B),
        \|\cdot\|_
        {L^\infty(
            \mathcal B(0,R)\times[0,1];\ell_2)}
    \right)
    \\
    &\qquad\lesssim
    SL
    \log\left(
        \frac{
            \sqrt D\,L(W+1)(B\vee1)(R+1)
        }{
            \eta
        }
    \right).
    \end{aligned}
    \label{eq:fm-relu-covering}
    \end{equation}
\end{lemma}

\begin{proof}[Proof of Lemma~\ref{lem:fm-relu-covering}]
We adapt the parameter-discretization argument of
\cite[Lemma~3]{suzuki2019adaptivity} to $D$-dimensional outputs equipped
with the Euclidean norm. Zero-pad
to the maximal width-$W$ architecture, fix a support $J$ with $|J|\le S$,
and let $\theta,\theta'$ be supported on $J$ with
$\|\theta-\theta'\|_\infty\le h$. For
$M:=(W+1)(B\vee1)$, the affine recursion and the $1$-Lipschitz property
of ReLU give
\[
 H_\ell\le(R+1)M^\ell,\qquad
 E_\ell\le ME_{\ell-1}+h(R+1)M^\ell
 \le\ell(R+1)M^\ell h.
\]
Consequently,
\[
 \|f_\theta-f_{\theta'}\|_
 {L^\infty(\mathcal B(0,R)\times[0,1];\ell_2)}
 \lesssim\sqrt D\,L(R+1)M^Lh.
\]
Choose $h\asymp\eta/[\sqrt D\,L(R+1)M^L]$. Quantizing each active
coordinate in $[-B,B]$ at mesh $h$, while leaving inactive coordinates
zero, produces an internal net. The maximal architecture has at most
$CL(W+1)^2$ parameter locations, so
\[
 |\mathcal F_\eta|
 \le\left[CL(W+1)^2\left(1+\frac Bh\right)\right]^S.
\]
Taking logarithms and substituting $h$ yields
\[
 \log|\mathcal F_\eta|
 \lesssim SL\log\!\left(
 \frac{\sqrt D\,L(W+1)(B\vee1)(R+1)}{\eta}\right),
\]
which proves \eqref{eq:fm-relu-internal-covering} and
\eqref{eq:fm-relu-covering}.
\end{proof}

The ambient cover is converted into constrained representatives only after the
truncation scales have been fixed.

\begin{proof}[Proof of Lemma~\ref{lem:main-fm-finite-reduction}]
Fix $\tau,\delta\in(0,1)$ and introduce the proof-local scales
\[
 \Lambda_{n,\delta,\tau}:=\frac1\sigma+\sqrt D
 +\sqrt{\log\frac{en}{\delta}}+\sqrt{\log\frac e\tau},
 \qquad
 \eta:=\frac{c_0\tau}{\Lambda_{n,\delta,\tau}},
\]
\[
 R_0:=\sqrt D+\sqrt{2\log(1/\eta^2)},\qquad
 R_1:=1+\sigma\sqrt D+
 \sigma\sqrt{2\log\frac{n}{\delta\eta^2}},\qquad
 R:=R_0\vee R_1,
\]
where $c_0>0$ is a sufficiently small absolute constant. Gaussian
concentration and $\|g^*\|_{L^\infty}\le1$ imply
\[
 \Pr(\|X_0\|>R_0)\le\eta^2,\qquad
 \Pr(\|X_1\|>R_1)\le\frac{\delta\eta^2}{n},
\]
and therefore
\begin{equation}
 \Pr\!\left(\max_{i\le n}\|X_{1,i}\|\le R_1\right)\ge1-\delta.
 \label{eq:fm-target-truncation-event}
\end{equation}

Let $\mathcal F_{\mathcal V}$ be the posterior-network realizations whose
induced velocities belong to $\mathcal V(L,W,S,B,\gamma_x)$.
Lemma~\ref{lem:fm-relu-covering} supplies an ambient $\eta$-net on
$\mathcal B(0,R)\times[0,1]$. Retain the balls that intersect
$\mathcal F_{\mathcal V}$ and choose an actual constrained realization
$\widetilde f_j$ from each retained ball. The corresponding fields
\[
 v_j(x,t):=a_tx+b_t\operatorname{clip}_2(\widetilde f_j(x,t))
\]
are deterministic members of $\mathcal V(L,W,S,B,\gamma_x)$, and their
number $M_{\tau,\delta}$ is no larger than the ambient covering number.
Every constrained realization has a selected representative with internal
posterior-network distance at most $2\eta$ on the truncated state--time
domain.

Since $(R_1+1)/\sigma+\sqrt D\lesssim\Lambda_{n,\delta,\tau}$,
Lemma~\ref{lem:fm-local-loss-stability}, the displayed tail bounds, and a
sufficiently small $c_0$ give
\[
 \|\Delta_v-\Delta_{v_j}\|_{L^\infty(\mathcal B(0,R_1))}
 \le\frac\tau4,
 \qquad
 |P(\Delta_v-\Delta_{v_j})|\le\frac{3\tau}4.
\]
On the event in \eqref{eq:fm-target-truncation-event}, the same selected
field $v_j$ therefore satisfies
\[
 |P(\Delta_v-\Delta_{v_j})|
 +|P_n(\Delta_v-\Delta_{v_j})|\le\tau,
\]
which proves \eqref{eq:main-fm-finite-reduction}.

Finally, set
\[
 \mathfrak R_{n,\delta,\tau}:=2+\sqrt D+
 \sqrt{\log\frac{en\Lambda_{n,\delta,\tau}^2}
 {\delta\tau^2}}.
\]
Then $R+1\lesssim\mathfrak R_{n,\delta,\tau}$, and
\eqref{eq:fm-relu-internal-covering} gives
\begin{equation}
 \log M_{\tau,\delta}\lesssim SL\log\!\left(
 \frac{\sqrt D\,L(W+1)(B\vee1)\Lambda_{n,\delta,\tau}
 \mathfrak R_{n,\delta,\tau}}{\tau}\right).
 \label{eq:fm-finite-reduction-cardinality}
\end{equation}
Moreover,
\[
 \Lambda_{n,\delta,\tau}\lesssim
 1+\frac1\sigma+\sqrt D
 +\sqrt{\log\frac{en}{\delta}}
 +\sqrt{\log\frac e\tau},
 \qquad
 \mathfrak R_{n,\delta,\tau}\lesssim
 1+\frac1\sigma+\sqrt D
 +\sqrt{\log\frac{en}{\delta}}
 +\sqrt{\log\frac e\tau}.
\]
Substitution into \eqref{eq:fm-finite-reduction-cardinality} proves the explicit bound \eqref{eq:main-fm-finite-reduction-cardinality} stated in the main text.
\end{proof}

\subsection{Uniform concentration}
\label{app:generalization-uniform}

\begin{lemma}[Uniform localized concentration of the excess loss]
\label{lem:fm-uniform-concentration}
For every $\tau,\delta\in(0,1)$, let $v_1,\ldots,v_M$ be the
deterministic family supplied by
Lemma~\ref{lem:main-fm-finite-reduction} with confidence parameter
$\delta/2$. With probability at least $1-\delta$,
\begin{equation}
 |(P-P_n)\Delta_v|
 \le\frac14P\Delta_v+C\left\{\left[\frac1{\sigma^2}
 +\left(\frac1\sigma+\sqrt D\right)\log(en)\right]
 \frac{\log(8M/\delta)}n+\tau\right\}
 \label{eq:fm-uniform-localized-deviation}
\end{equation}
simultaneously for all $v\in\mathcal V(L,W,S,B,\gamma_x)$, where
$C>0$ is an absolute constant.
\end{lemma}

\begin{proof}
Use the deterministic family supplied by
Lemma~\ref{lem:main-fm-finite-reduction} with confidence parameter
$\delta/2$. On an event of probability at least $1-\delta/2$,
every $v$ has a representative $v_k$ satisfying
\eqref{eq:main-fm-finite-reduction}. Put $h:=\log(8M/\delta)$.
The fixed-field estimate \eqref{eq:main-fm-fixed-bernstein} and a union bound
give, on another event of probability at least $1-\delta/2$,
\[
 |(P-P_n)\Delta_{v_k}|
 \le C_{\mathrm B}\sqrt{\frac{P\Delta_{v_k}h}{\sigma^2n}}
 +C_{\mathrm B}\frac{(\sigma^{-1}+\sqrt D)\log(en)h}{n}
\]
for every representative. On the intersection of these events,
$P\Delta_{v_k}\le P\Delta_v+\tau$ and transfer through the selected constrained representative gives
\[
 |(P-P_n)\Delta_v|
 \le C_{\mathrm B}\sqrt{\frac{(P\Delta_v+\tau)h}{\sigma^2n}}
 +C_{\mathrm B}\frac{(\sigma^{-1}+\sqrt D)\log(en)h}{n}
 +\tau.
\]
The inequalities $\sqrt{a+b}\le\sqrt a+\sqrt b$ and Young's inequality
prove \eqref{eq:fm-uniform-localized-deviation}.
\end{proof}

\begin{proof}[Proof of the oracle inequality \eqref{eq:main-fm-oracle}]
Use Lemma~\ref{lem:main-fm-finite-reduction} with confidence $\delta/2$
and write $M=M_{\tau,\delta/2}$. For every deterministic comparator
$\bar v\in\mathcal V(L,W,S,B,\gamma_x)$, the ERM property gives
\[
 P\Delta_{\widehat v}
 \le P\Delta_{\bar v}
 +(P-P_n)\Delta_{\widehat v}+(P_n-P)\Delta_{\bar v}.
\]
On the simultaneous event from
Lemma~\ref{lem:fm-uniform-concentration}, apply
\eqref{eq:fm-uniform-localized-deviation} to both fields and absorb the
$\frac14P\Delta_{\widehat v}$ term. The remaining
$\frac14P\Delta_{\bar v}$ term is absorbed into the universal multiplicative
constant. Taking the infimum over $\bar v$ proves
\eqref{eq:main-fm-oracle}.
\end{proof}

\section{Auxiliary Results for Distributional Convergence}
\label{app:sampling}

This appendix collects the auxiliary growth estimate used in the sampling
analysis and proves the Euler discretization bound in
Proposition~\ref{prop:euler-sampling-error}.

\subsection{Linear-growth bound}
\label{app:sampling-growth}

\begin{lemma}[Linear growth of the velocity fields]
\label{lem:sampling-linear-growth}
Let
\[
    K_\sigma^{\mathrm{sam}}
    :=
    \max\left\{
        \sup_{t\in[0,1]}|a_t|,
        2\sup_{t\in[0,1]}|b_t|
    \right\}.
\]
Then, for every $(x,t)\in\mathbb R^D\times[0,1]$,
\[
    \|v^*(x,t)\|
    \le
    K_\sigma^{\mathrm{sam}}(1+\|x\|),
\]
and, for every
$v\in\mathcal V(L,W,S,B,\gamma_x)$,
\[
    \|v(x,t)\|
    \le
    K_\sigma^{\mathrm{sam}}(1+\|x\|).
\]
\end{lemma}

\begin{proof}
By Lemma~\ref{lem:posterior-velocity},
\[
    v^*(x,t)=a_tx+b_tm_t(x),
    \qquad
    \|m_t(x)\|\le1,
\]
and therefore
\[
    \|v^*(x,t)\|
    \le
    |a_t|\|x\|+|b_t|.
\]
For
$v\in\mathcal V(L,W,S,B,\gamma_x)$,
the radius-two clipping in the definition of the velocity class gives
\[
    \|v(x,t)\|
    \le
    |a_t|\|x\|+2|b_t|.
\]
Both bounds are dominated by
$K_\sigma^{\mathrm{sam}}(1+\|x\|)$.
\end{proof}

\subsection{Euler discretization}
\label{app:sampling-euler}

\begin{proof}[Proof of Proposition~\ref{prop:euler-sampling-error}]
Condition on the target observations. Couple the continuous learned
trajectory and the Euler scheme through the same initial condition,
\[
    \widehat Z_0
    =
    \widetilde Z_{t_0}
    \sim\pi_0.
\]
By Lemma~\ref{lem:sampling-linear-growth} and Gr\"onwall's inequality,
\[
    1+\|\widehat Z_t\|
    \le
    e^{K_\sigma^{\mathrm{sam}}t}
    (1+\|\widehat Z_0\|)
    \le
    e^{K_\sigma^{\mathrm{sam}}}
    (1+\|\widehat Z_0\|),
    \qquad
    t\in[0,1].
\]
Consequently, for $t\in[t_k,t_{k+1}]$,
\[
\begin{aligned}
    \|\widehat Z_t-\widehat Z_{t_k}\|
    &\le
    \int_{t_k}^t
    \|\widehat v(\widehat Z_s,s)\|\,ds
    \\
    &\le
    K_\sigma^{\mathrm{sam}}
    e^{K_\sigma^{\mathrm{sam}}}
    (1+\|\widehat Z_0\|)
    (t-t_k).
\end{aligned}
\]

Write
$h_k:=t_{k+1}-t_k$.
The continuous learned trajectory satisfies
\[
    \widehat Z_{t_{k+1}}
    =
    \widehat Z_{t_k}
    +
    h_k\widehat v(\widehat Z_{t_k},t_k)
    +
    r_k,
\]
where
\[
    r_k
    :=
    \int_{t_k}^{t_{k+1}}
    \Bigl[
        \widehat v(\widehat Z_t,t)
        -
        \widehat v(\widehat Z_{t_k},t_k)
    \Bigr]dt .
\]
Using the spatial class constraint and
Assumption~\ref{ass:learned-temporal-regularity},
\[
\begin{aligned}
    \|
        \widehat v(\widehat Z_t,t)
        -
        \widehat v(\widehat Z_{t_k},t_k)
    \|
    \le{}&
    \bar\gamma_x
    \|\widehat Z_t-\widehat Z_{t_k}\|
    \\
    &+
    \gamma_t
    (1+\|\widehat Z_{t_k}\|)
    (t-t_k).
\end{aligned}
\]
The preceding growth estimates therefore yield
\[
    \|r_k\|
    \le
    \frac12
    e^{K_\sigma^{\mathrm{sam}}}
    \bigl(
        \bar\gamma_xK_\sigma^{\mathrm{sam}}+\gamma_t
    \bigr)
    (1+\|\widehat Z_0\|)
    h_k^2.
\]

Subtracting the Euler recursion
\eqref{eq:euler-sampling-scheme} gives
\[
\begin{aligned}
    \|
        \widehat Z_{t_{k+1}}
        -
        \widetilde Z_{t_{k+1}}
    \|
    \le{}&
    (1+\bar\gamma_xh_k)
    \|
        \widehat Z_{t_k}
        -
        \widetilde Z_{t_k}
    \|
    +
    \|r_k\|.
\end{aligned}
\]
Iterating this inequality and using
$1+u\le e^u$ gives
\[
\begin{aligned}
    \|
        \widehat Z_1-\widetilde Z_{t_M}
    \|
    \le{}&
    \frac12
    e^{K_\sigma^{\mathrm{sam}}+\bar\gamma_x}
    \bigl(
        \bar\gamma_xK_\sigma^{\mathrm{sam}}+\gamma_t
    \bigr)
    (1+\|\widehat Z_0\|)
    \sum_{k=0}^{M-1}h_k^2.
\end{aligned}
\]
Since
$\sum_{k=0}^{M-1}h_k=1$ and $h_k\le h$,
\[
    \sum_{k=0}^{M-1}h_k^2\le h.
\]
Moreover,
$\widehat Z_0\sim\mathcal N(0,I_D)$, so
\[
    \|1+\|\widehat Z_0\|\|_{L^2}
    \le
    1+\sqrt{\mathbb E\|\widehat Z_0\|^2}
    =
    1+\sqrt D.
\]
The coupling of the continuous learned flow and the Euler scheme
therefore gives
\[
\begin{aligned}
    W_2(\widetilde\pi_1,\widehat\pi_1)
    \le{}&
    \frac12(1+\sqrt D)
    e^{K_\sigma^{\mathrm{sam}}+\bar\gamma_x}
    \bigl(
        \bar\gamma_xK_\sigma^{\mathrm{sam}}+\gamma_t
    \bigr)h.
\end{aligned}
\]

Finally, the triangle inequality gives
\[
    W_2(\widetilde\pi_1,\pi_1)
    \le
    W_2(\widetilde\pi_1,\widehat\pi_1)
    +
    W_2(\widehat\pi_1,\pi_1).
\]
Applying the continuous-flow bound
\eqref{eq:continuous-w2-stability} to the second term yields
\[
\begin{aligned}
    W_2(\widehat\pi_1,\pi_1)
    \le
    \exp\!\left\{
        \frac12+\|\gamma_x\|_{L^1(0,1)}
    \right\}
    \left(
        \int_0^1
        \|\widehat v(\cdot,t)-v^*(\cdot,t)\|_{L^2(\pi_t)}^2
        \,dt
    \right)^{1/2}.
\end{aligned}
\]
Combining the last two displays proves
\eqref{eq:euler-sampling-bound}.
\end{proof}


\end{document}